\documentclass[11pt]{article}

\usepackage[margin=1in]{geometry}
\usepackage{amsthm,amsmath,amsfonts,amssymb}
\usepackage[numbers]{natbib}
\usepackage[normalem]{ulem}
\usepackage[subscriptcorrection]{newtxmath}
\usepackage{multicol}
\usepackage{paracol}

\usepackage{setspace}
\newtheorem{lemma}{Lemma}
\newtheorem{assumption}{Assumption}
\newtheorem{theorem}{Theorem}
\newtheorem{proposition}{Proposition}
\newtheorem{corollary}{Corollary}
\newtheorem{definition}{Definition}
\newtheorem{remark}{Remark}
\def\T{\top}

\def\idop{{\mathrm 1}}
\newcommand{\E}{\mathbb{E}} 
\def\cov{{\mathrm{cov}}}
\def\pr{\mathbb{P}} 

\def\tr{\mathrm{tr}}

\def\diag{\mathrm{diag}}
\DeclareMathOperator*{\argmin}{arg\,min}
\DeclareMathOperator*{\argmax}{arg\,max}
\def\rank{\mathrm{rank}}
\def\span{\mathrm{span}}
\def\row{\mathrm{row}}

\newcommand{\bA}{\mathbf{A}}

\newcommand{\bB}{\mathbf{B}}
\newcommand{\bC}{\mathbf{C}}
\newcommand{\bD}{\mathbf{D}}

\newcommand{\bG}{\mathbf{G}}
\newcommand{\bH}{\mathbf{H}}
\newcommand{\be}{\mathbf{e}}
\newcommand{\bX}{\mathbf{X}}
\newcommand{\bx}{\mathbf{x}}

\newcommand{\bY}{\mathbf{Y}}
\newcommand{\by}{\mathbf{y}}
\newcommand{\bL}{\mathbf{L}}
\newcommand{\bMM}{\mathbf{M}}

\newcommand{\bW}{\mathbf{W}}

\newcommand{\bZ}{\mathbf{Z}}

\newcommand{\bI}{\mathbf{I}}
\newcommand{\bE}{\mathbf{E}}
\newcommand{\bU}{\mathbf{U}}
\newcommand{\bu}{\mathbf{u}}
\newcommand{\bV}{\mathbf{V}}

\newcommand{\bR}{\mathbf{R}}
\newcommand{\bS}{\mathbf{S}}

\newcommand{\bT}{\mathbf{T}}
\newcommand{\bO}{\mathbf{O}}

\newcommand{\bQ}{\mathbf{Q}}

\newcommand{\bSigma}{\boldsymbol{\Sigma}}

\newcommand{\bvarepsilon}{\boldsymbol{\varepsilon}}
\newcommand{\bxi}{\boldsymbol{\xi}}
\newcommand{\bXi}{\boldsymbol{\Xi}}
\newcommand{\bDelta}{\boldsymbol{\Delta}}

\newcommand{\bzero}{\mathbf{0}}

\newcommand{\op}{\mathrm{op}}
\def\F{{\mathrm{F}}}
\def\blk{{\mathrm{blk}}}
\def\Sign{{\mathrm{Sign}}}

\def\range{{\mathrm{range}}}
\def\vec{{\mathrm{vec}}}

\def\PT{P_{\mathcal{T}}}
\def\PTC{P_{\mathcal{T}^\perp}}
\def\PI{P_{\mathcal{I}}}
\def\PIC{P_{\mathcal{I}^c}}

\def\PU{P_{\bU}}

\def\PV{P_{\bA}} 
\def\PVC{P_{\bA^\perp}} 
\def\PA{P_{\bA}}
\def\PAC{P_{\bA^\perp}}

\def\Pg{P_{g}}

\def\Pomega{P_{\omega}}
\def\veg{\varepsilon_{g}}

\def\pL{\mathbf{L}^\prime}
\def\pS{\mathbf{S}^\prime}
\def\pU{\mathbf{U}^\prime}
\def\pV{\mathbf{V}^\prime}
\def\pVT{\mathbf{V}^{\prime\top}}
\def\pSigma{\mathbf{\Sigma}^\prime}
\def\pPT{P_{\mathcal{T}(\pL)}}
\def\pPTC{P_{\mathcal{T}(\pL)^{\perp}}}
\def\pPU{P_{\mathbf{U}^\prime}}
\def\pPUC{P_{\mathbf{U}^{\prime\perp}}}
\def\pPV{P_{\mathbf{V}^\prime}}
\def\pPVC{P_{\mathbf{V}^{\prime\perp}}}

\def\hL{\widehat{\mathbf{L}}}
\def\hS{\widehat{\mathbf{S}}}
\def\hD{\widehat{\mathbf{D}}}
\def\hU{\widehat{\mathbf{U}}}
\def\hV{\widehat{\mathbf{A}}} 
\def\hSigma{\widehat{\mathbf{\Sigma}}}

\def\hPV{P_{\hV}}
\def\hPVC{P_{\hV^\perp}}
\def\hR{\widehat{\mathbf{R}}}
\def\hW{\widehat{\bW}}
\def\tW{\widetilde{\bW}}

\def\EV{\mathcal{E}_{\bA}}
\def\EVC{\mathcal{E}_{\bA^\perp}}

\usepackage[ruled]{algorithm2e}
\newlength{\alglabelwidth}
\usepackage{xcolor}
\usepackage{graphicx}

\usepackage{comment}

\graphicspath{{./pic/}}

\usepackage[colorlinks,citecolor=blue,urlcolor=blue,linkcolor=blue]{hyperref}

\title{Federated LoRA Fine-Tuning for LLMs via\\ Collaborative Alignment}
\author{Shuaida He \quad Liwen Chen \quad Long Feng\thanks{Correspondence to: lfeng@hku.hk}\\
School of Computing \& Data Science, The University of Hong Kong}
\date{}

\begin{document}

\maketitle

\begin{abstract}
Low-rank adaptation (LoRA) has emerged as a powerful tool for parameter-efficient fine-tuning of large language models (LLMs). This paper studies LoRA under a federated learning setting, enabling collaborative fine-tuning across clients while preserving parameter efficiency. We focus on a highly heterogeneous regime in which clients share only partial structure and a substantial subset may be contaminated.
We propose Collaborative Low-rank Alignment and Identifiable Recovery (CLAIR), a contamination-aware framework that relies only on preliminary local estimators. Its formulation applies broadly, from linear regression to neural network and LLM modules, whenever local adaptation can be represented by matrix-valued updates. CLAIR recovers the shared LoRA subspace and detects contaminated clients via a structured low-rank plus block-sparse decomposition.
We prove exact recovery of the shared LoRA subspace in the noiseless case, stable recovery under preliminary estimation error, and consistent collaborative-set recovery under mild separation conditions.
We further quantify the gain from CLAIR refinement: it reduces off-subspace estimation error through cross-client averaging while preserving client-specific variation within the shared LoRA subspace, thus improves over local fine-tuning whenever this oracle gain outweighs the costs of subspace estimation and benign-client heterogeneity.
Empirically, we demonstrate the benefits of CLAIR by fine-tuning a Transformer architecture on a text-copying task. The results show accurate contamination detection and improved benign-client performance compared with local fine-tuning and non-robust federated averaging.
\end{abstract}


\clearpage

\section{Introduction}\label{sec:introduction}
Large language models (LLMs), such as ChatGPT \citep{wu2023brief}, Gemini \citep{comanici2025gemini}, and Llama \citep{grattafiori2024llama}, have become central to modern artificial intelligence, achieving strong performance in text generation, reasoning, and decision-making \citep{ji2026overview}. Built on Transformer architectures with attention mechanisms~\citep{vaswani2017attention,sheen2024implicit},
these models often contain millions to billions of parameters and are pretrained on massive data to learn general-purpose representations.
In contrast, local downstream adaptation typically targets narrower tasks or domains and relies on substantially smaller samples.
Full parameter fine-tuning is then computationally expensive and statistically inefficient for individual clients or organizations.
This motivates parameter-efficient fine-tuning (PEFT), which freezes most pretrained weights and updates only a small set of task-specific parameters \citep{houlsby2019parameter,xu2026parameter}.

Low-Rank Adaptation \citep[LoRA;][]{hu2022lora} is a prominent example of this principle. For a pretrained weight matrix $\bW_0 \in \mathbb{R}^{q\times p}$, LoRA parameterizes the adapted weight as
\begin{equation}\label{eq:original-lora}
	\bW=\bW_0 +\bDelta,\quad \bDelta=\bB\bA,
\end{equation}
where $\bDelta$ is a low-rank update, or adapter,  with $\bB\in \mathbb{R}^{q\times r}$, $\bA \in \mathbb{R}^{r \times p}$, and $r \ll \min(p,q)$.
By representing adaptation through low-dimensional factors, LoRA greatly reduces the number of trainable parameters while retaining substantial flexibility. The resulting adapters are compact, easy to store and transfer, and can be switched across tasks while sharing the same pretrained backbone $\bW_0$.
This makes LoRA well-suited to efficient fine-tuning on small local datasets, and particularly useful in multi-task and multi-client settings.
Consequently, LoRA has become a standard approach for adapting large pretrained models and has been widely used across a broad range of downstream applications~\citep{houlsby2019parameter,liu2025lora}.

Although LoRA was introduced as an engineering device for PEFT, it reflects a broader statistical principle:  exploiting intrinsic structures, such as low-rankness, to improve statistical efficiency.
This principle underlies classical methods such as matrix completion~\citep{candes2012exact,klopp2019structured},  principal component analysis (PCA) and robust PCA \citep{candes2011robust},  reduced-rank regression \citep{yuan2007dimension}, mixture models \citep{doss2023optimal}, and tensor regression \citep{zhou2013tensor}.
At the same time, the factorized formulation of LoRA provides a flexible model for shared structure across related learning problems, connecting naturally to representation based multi-task learning, in which multiple tasks share a latent representation while retaining task-specific coefficients \citep{duan2023adaptive,niu2024collaborative,tian2025learning}; fine-tuning-based meta-learning \citep{chua2021fine}; and linear model transfer learning, where a target problem borrows structure from related source problems \citep{gu2025robust}.

Low-rank adaptation can benefit substantially from collaboration across groups or clients through federated learning \citep[FL, ][]{mcmahan2017communication}, which mitigates local data scarcity by pooling information at the model level. Federated learning enables collaborative model improvement while preserving data governance, as each client trains locally on private data and shares only model updates or parameters.
Combining LoRA with federated aggregation methods such as FedAvg has shown broad applicability across downstream tasks~\citep{zhang2024towards,bian2025lora}.
This paradigm is particularly attractive in sensitive or regulated domains, such as healthcare, where LLMs could be customized to analyze institution-specific medical text while complying with strict privacy requirements.

Federated deployment of LoRA raises several issues beyond the centralized PEFT.
The first is to determine which object should be shared across clients.
Specifically, for client $k$, write the locally adapted weight as $\bW^{(k)}=\bW_0+\bB^{(k)}\bA^{(k)}$.
Then, in centralized LoRA, only the refined estimator $\hW^{(k)}$ is of primary interest.
In a federated system, however, aggregation must identify the component of $\hW^{(k)}$ that is transferable across clients.
For example, product-level methods construct a stacked adapter whose product represents a weighted average of the client updates $\bDelta^{(k)}$, thereby treating the full adapters as directly comparable matrices \citep{wang2024flora}.
Factor-level methods impose stronger structural alignment, typically by sharing, fixing, or selectively aggregating one LoRA factor, either $\bA^{(k)}$ or $\bB^{(k)}$~\citep{sun2024improving,tian2024hydralora,guo2024selective}.
Expert-level methods allocate domain-specific LoRA experts and allow clients select subsets through an adaptive mixture-of-experts (MoE) mechanism, so that the transferable units are selected expert adapters \citep{wang2025adaptive}.
Despite their empirical success, there remains limited theoretical understanding of which shared structure is identifiable and why its aggregation, even in basic FedAvg-type procedures, improves estimation.
This motivates a formal treatment of the shared LoRA component as the collaborative structure underlying federated fine-tuning.

The second challenge is to handle statistical heterogeneity across clients.
Heterogeneity, arising from distributional shifts, task differences, or other sources, determines whether collaboration is beneficial: when clients are sufficiently aligned, methods such as FedAvg can improve efficiency, whereas under severe heterogeneity, purely local training may be preferable~\citep{chen2023minimax}.
This difficulty is compounded by the decentralized nature of federated learning, where some clients may be outlying, weakly related, or even adversarial~\citep{kumar2023impact}.
We refer to such clients as contaminated, in contrast to benign clients whose LoRA adapters share a common structure and can benefit from collaborative refinement.
In practice, even a small contaminated fraction can distort the estimated shared structure and induce negative transfer for benign clients \citep{guo2025robust,wang2025adaptive}.
Thus, the central inferential task is to separate benign variation around the shared low-rank adapter structure from contamination that violates it. We formalize this task as collaborative-set recovery: adaptively identifying the benign clients for collaboration while excluding contaminated ones.

Moreover, existing federated LoRA methods typically require client updates to be represented relative to the same backbone and parameterization. This assumption is natural when the server controls a shared open backbone $\bW_0$, or when clients report adapters $\{\bDelta^{(k)}\}_{k=1}^K$ relative to a common checkpoint. However, direct adapter-level aggregation is not readily applicable when the server observes only locally adapted weights or preliminary local estimators $\{\hW^{(k)}\}_{k=1}^K$, without access to the underlying base model or the corresponding adapter decompositions.
This issue is particularly relevant when clients are initialized from similar but non-identical base models within a common parameter space, since LoRA adapters are tied to their underlying backbones \citep{wang2024textit,smith2017federated,fallah2020personalized}.
Therefore,~\eqref{eq:original-lora} should be viewed as an ideal population representation rather than an observed decomposition, raising a basic identifiability question: when neither the base model nor the adapters is separately observed, what shared low-rank structure is identifiable from preliminary local estimators alone?

To address these issues, we propose a contamination-aware framework for collaborative federated fine-tuning of LLMs via low-rank adaptation. The framework accommodates settings where (i) the underlying foundation model $\bW_0$ is unknown or inaccessible, (ii) model structure is only partially shared across clients, and (iii) client heterogeneity and contamination are present, with the contaminated subset allowed to grow with the number of clients.
Its formulation applies broadly, from linear regression to neural network and LLM modules, whenever local adaptation can be represented by matrix-valued updates.

Conceptually, our framework converts federated LoRA fine-tuning into a structured decomposition problem, whose identifiable components encode transferable local knowledge and the client contamination pattern.  This is achieved through a pairwise contrast construction that cancels the unobserved backbone while preserving client-pair information.
The resulting decomposition takes a low-rank plus block-sparse form, which resembles robust PCA \citep{candes2011robust,hsu2011robust}. On the other hand, its geometry differs fundamentally due to the constraints induced by the federated LoRA formulation.
Consequently, classical exact-recovery arguments do not apply directly. We address this challenge with a tractable convex program followed by subspace projection, which simultaneously recovers the shared structure and filters out contaminated clients, providing the basis for client-specific collaborative refinement.

Theoretically, we establish exact and stable recovery guarantees for the shared structure and prove consistency of contaminated client detection under suitable identifiability conditions.
A central contribution is an explicit characterization of when collaborative fine-tuning improves upon purely local estimation, together with a quantification of the resulting mean-squared error gain.
In contrast to existing empirical studies on federated LoRA aggregation \citep{wang2024flora,guo2024selective,bian2025lora}, our analysis identifies the source of client-level improvement: the refined estimator preserves the local-specific error in the shared adapter row space, while reducing the orthogonal-complement error through averaging over collaborating benign clients. This yields a transparent gain-to-cost condition under which collaborative fine-tuning outperforms local fine-tuning, and also clarifies when negative transfer may occur.

Empirically, we evaluate CLAIR in a controlled sequence-copying experiment with Transformer models locally fine-tuned via LoRA. The experiment illustrates the two statistical roles characterized by the theory: detecting a contaminated client through contrasts among locally adapted models, and improving prediction accuracy for benign clients relative to local fine-tuning and non-robust FedAvg baselines.

\subsection{Notation}
Denote $[m] =\{1, 2, \dots, m\}$ and $|\mathcal{S}|$ as the cardinality of a finite set $\mathcal{S}$. For a sub-Gaussian random variable $X\in \mathbb{R}$, define $\|X\|_{\psi_2} = \sup_{p\geq 1}(\E|X|^p)^{\frac{1}{p}}/\sqrt{p}$.
Let $\bI_p$ denote the $p\times p$ identity matrix.
For $p \geq r$, define $\mathbb{O}_{p,r} = \{\bV\in \mathbb{R}^{p \times r}: \bV^\T \bV = \bI_r\}$.
For a matrix $\bA\in \mathbb{R}^{m\times n}$, denote its Frobenius, spectral, and nuclear norms by $\|\bA\|_{\F}$, $\|\bA\|_{\op}$, and $\|\bA\|_*$, respectively. Denote its column and row spaces by $\mathrm{col}(\bA)$ and $\row(\bA)$, and its vectorization by $\vec(\bA)$.
We use $\langle\bA,\bB\rangle=\tr(\bA^\T\bB)$ for the Frobenius inner product, and $\cov(\bZ)$ for the covariance matrix of a random vector $\bZ$.
We use $\otimes$ to denote the Kronecker product.
For a row-block matrix $\bX$ with blocks $\{\bX^{(g)}\}$, let $\Pg(\bX) := (\be_{g}\otimes \bI_q) (\be_{g}\otimes \bI_q)^\T (\bX)$ denote the zero-padded projection that keeps block $g$ and zeros out all other blocks, and define $\|\bX\|_{\blk,1}:=\sum_g\|\bX^{(g)}\|_{\F}$.
For a symmetric matrix $\bB \in \mathbb{R}^{m \times m}$, let its eigenvalues be ordered as $\lambda_1(\bB) \geq \cdots \geq \lambda_m (\bB)$. For symmetric matrices $\bA$ and $\bB$ of the same dimension, $\bA\preceq\bB$ means $\bB-\bA$ is positive semidefinite.
Unless otherwise stated, the singular value decomposition of a rank-$r$ matrix $\bL$ refers to its compact SVD $\bL=\bU\bSigma\bV^\T$, where $\bU\in\mathbb{O}_{m,r}$, $\bV\in\mathbb{O}_{n,r}$, and $\bSigma$ contains the positive singular values.
For deterministic sequences $\{a_n\}$ and $\{b_n\}$, write $a_n=O(b_n)$ if $|a_n|\le C|b_n|$ for some constant $C>0$ and all sufficiently large $n$, and write $a_n\asymp b_n$ if both $a_n=O(b_n)$ and $b_n=O(a_n)$.

\section{Problem setup}\label{sec:problem-setup}
Consider regression models distributed across $K$ data clients.
For client $k\in[K]$, the model is specified as
\begin{equation}\label{eq:model-k}
		\by^{(k)} = f^{(k)}\bigl(\bx^{(k)},\bvarepsilon^{(k)}; \bW^{(k)}\bigr),
\end{equation}
where $\by^{(k)}$ is the response vector, $\bx^{(k)}$ the covariate vector, $\bvarepsilon^{(k)}$ the random noise, $f^{(k)}$ a general (linear or nonlinear) function, and $\bW^{(k)} \in \mathbb{R}^{q\times p}$ the target parameter matrix. We assume
\begin{equation}\label{eq:lora}
		\bW^{(k)}=\bW_0 +\bDelta^{(k)},\quad \bDelta^{(k)}=\bB^{(k)}\bA^{(k)},
\end{equation}
where $\bW_0$ is an unknown global base model shared across clients, and $\bDelta^{(k)}$ captures client-specific heterogeneity.
Here, $\bB^{(k)}\in \mathbb{R}^{q\times r}$, $\bA^{(k)} \in \mathbb{R}^{r \times p}$, and rank $r < \min(p,q)$.  By constraining $\bDelta^{(k)}$ to be low-rank, model \eqref{eq:model-k} together with \eqref{eq:lora} naturally leads to a federated LoRA formulation. With only the decomposition in (\ref{eq:lora}), the pair $(\bW_0,\bDelta^{(k)})$ is clearly non-identifiable. However, our aim is not to identify $(\bW_0,\bDelta^{(k)})$, but rather to refine a local estimator of $\bW^{(k)}$ and obtain an improved one, as discussed in detail later.

With a suitable choice of $f^{(k)}$, model~\eqref{eq:model-k} is highly expressive and covers a broad class of settings,  including linear regression,  multiple-index models with nonlinear links, and neural networks with additional trainable parameters beyond $\bW^{(k)}$.
In particular, in modern LLM fine-tuning, $\bW^{(k)}$ denotes a locally adapted module for client $k$, obtained from the pretrained backbone $\bW_0$ via a low-rank adapter $\bDelta^{(k)}$.
Such modules typically include transformer weights, such as the query, key, and value projections in self-attention \citep{vaswani2017attention}.
The heterogeneity in $\bDelta^{(k)}$ reflects differences in datasets, downstream tasks, or reasoning objectives, etc.

In federated learning, each client reports a local estimator of $\bW^{(k)}$ based on its private dataset $\mathcal{D}^{(k)}=\{(\bx_i^{(k)},\by_i^{(k)})\}_{i=1}^{n_k}$.
We call the resulting estimator $\hW^{(k)}$ the preliminary estimator for client $k$.
It may be obtained either by local full training, such as stochastic gradient descent with a client-specific loss, or by fine-tuning from $\bW_0$ through estimation of $\bDelta^{(k)}$.
Since $\bW_0$ is often inaccessible or unidentifiable in practice, and the shared row space $\operatorname{row}(\bA)$ is also unknown, our subsequent procedure relies only on the collection of preliminary estimators $\{\hW^{(k)}\}_{k=1}^K$.

Our goal is to develop a unified framework for collaborative fine-tuning under the federated LoRA paradigm, without imposing structural assumptions on $f^{(k)}$ or requiring access to $\bW_0$, and to quantify the statistical gain of the resulting estimators of  $\{\bW^{(k)}\}_{k=1}^{K}$. A natural question is therefore: which clients can collaborate, and what shared information yields mutual benefit?

We answer this question by introducing the collaborative set of benign clients.
Formally, assume there exist a subset of clients $\mathcal{C}\subseteq[K]$ and a matrix $\bA^\T\in\mathbb{O}_{p,r}$ such that, for some small $\delta>0$,
\[
	\max_{k\in\mathcal{C}}\ \min_{\bO\in\mathbb{O}_{r,r}} \|\bO^\T\bA^{(k)}-\bA\|_{\F}\le \delta.
\]
This formulation accounts for the rotational non-identifiability of LoRA factorization. Indeed, for any $\bO^{(k)}\in\mathbb{O}_{r,r}$, $\bB^{(k)}\bA^{(k)} = (\bB^{(k)}\bO^{(k)})(\bO^{(k),\T}\bA^{(k)})$, so $\bA^{(k)}$ is identifiable only through its row space.
We may therefore assume, without loss of generality, that $\bA^{(k),\T}\in\mathbb{O}_{p,r}$.
For each $k\in\mathcal{C}$, after rotating $\bA^{(k)}$ by an optimizer and reusing the same notation, the preceding condition becomes
\begin{equation}\label{eq:subset-A}
	\max_{k\in\mathcal{C}}\|\bA^{(k)}-\bA\|_{\F}\le \delta.
\end{equation}
Thus, clients in $\mathcal{C}$ have adapter row spaces uniformly close to the common row space $\operatorname{row}(\bA)$, whereas clients in $\mathcal{C}^c:=[K]\setminus \mathcal{C}$ are treated as contaminated.
The set $\mathcal{C}$ therefore identifies the clients that share a compatible adapter structure and can benefit from collaborative estimation.
We write $\eta:=|\mathcal{C}|/K$ as the collaborative proportion, with larger $\eta$ indicating stronger cross-client homogeneity in the federated network.

\begin{remark}[Interpretation of the shared LoRA factor]\label{rem:shared-A}
	The shared row space condition~\eqref{eq:subset-A} places the collaborative structure in the right LoRA factor $\bA^{(k)}$. This choice is consistent with prior federated LoRA formulations~\citep{guo2024selective} and supported by empirical evidence on adapter sharing~\citep{wang2025adaptive}, and has a natural representation-learning interpretation.
	Specifically, consider a single-layer network with activation $\sigma$,
	\[
		\bY^{(k)} = \sigma\{\bW_0\bX^{(k)}+\bB^{(k)}(\bA^{(k)}\bX^{(k)})\},
	\]
	where $\bX^{(k)}\in\mathbb{R}^{p\times n}$ denotes the input representations.
	Then for a fixed LoRA factorization, $\bA^{(k)}$ maps the input to an $r$-dimensional adapter bottleneck,  whereas $\bB^{(k)}$ maps this bottleneck to the output coordinates.
	Thus, sharing the row space of $\bA^{(k)}$ amounts to sharing an input-side representation, while allowing client-specific output coefficients.
	When $q=1$ and $\bW_0=\bzero$, this includes representation-based multi-task learning as a special case~\citep{tian2025learning}; when $q>1$, it gives a reduced-rank or multi-response regression structure~\citep{yuan2007dimension}.
	Related shared-representation and task-specific structures also appear in adaptive multi-task and transfer learning~\citep{chua2021fine,duan2023adaptive}.
	A related input-output asymmetry appears in domain adaptation, where domain-specific feature maps may be combined with a shared or nearly shared classifier or output head \citep{bousmalis2016domain,tzeng2017adversarial}.
	This provides an analogue of a formulation based on shared left factors $\bB^{(k)}$, although our main analysis focuses on the shared right-factor row space in~\eqref{eq:subset-A}; the dual formulation can be handled analogously.
\end{remark}

\section{Methodology}\label{sec:method}

This section develops Collaborative Low-rank Alignment and Identifiable Recovery (CLAIR), a data-driven procedure for federated LoRA fine-tuning. CLAIR constructs a canonical matrix decomposition from pairwise differences of preliminary local estimators, thereby eliminating the common backbone and separating the shared row-space signal of benign clients from orthogonal contamination. This decomposition defines the estimands that underlie both active-set recovery and collaborative refinement.

\subsection{The CLAIR procedure}
We start by circumventing the need for accessing the shared global parameter  $\bW_0$  while still leveraging the similar structure encoded in $\{\bW^{(k)}\}_{k=1}^K$.
For distinct clients $j\neq k$, regard client $j$ as a reference for client $k$.
For benign pairs  $(j,k)\in\mathcal{C}$, we have the approximate relation
\[
	\bW^{(k)} \approx \bW^{(j)} - (\bB^{(j)} - \bB^{(k)})\bA,
\]
where $\bA$ acts as the ``center'' of $\{\bA_j:  j\in\mathcal{C}\}$.
Consequently, instead of estimating the unidentified global parameter $\bW_0$ in \eqref{eq:model-k}, we refine $\bW^{(k)}$ by adapting the reference model $\bW^{(j)}$ with $(\bB^{(j)} - \bB^{(k)})\bA$.
Aggregating such pairwise adaptations over multiple reference clients integrates local information and yields a refined estimator of $\bW^{(k)}$.

To formalize this construction, let
$\mathcal{G}:=\{(j,k)\in [K]\times [K]: 1\leq j<k\leq K\}$ denote the set of unordered client pairs,  indexed by the convention $j<k$, with cardinality $G=|\mathcal{G}|=\binom{K}{2}$.
For $g=(j,k)\in\mathcal{G}$, we use $g$ and $(j,k)$ interchangeably to label the corresponding block, and define the pairwise difference of preliminary estimates by
\begin{equation}\label{Dhat}
\hD^{(j,k)}:=\hW^{(j)}-\hW^{(k)} \in \mathbb{R}^{q\times p}.
\end{equation}
Stacking all these blocks vertically in a fixed order yields the aggregated difference matrix
\begin{equation}\label{Dhat2}
\hD:=\bigl[ \{\hD^{(1,2)}\}^\T,\ldots, \{\hD^{(1,K)}\}^\T,\{\hD^{(2,3)}\}^\T,\ldots, \{\hD^{(K-1,K)}\}^\T
\bigr]^\T \in \mathbb{R}^{Gq\times p}.
\end{equation}
This construction collects each pairwise contrast and preserves the client-pair indexing of each row block. The particular stacking order does not affect the subsequent analysis, and we fix one ordering of $\mathcal{G}$ throughout.

Given $\hD$, we solve the penalized convex program
\begin{equation}\label{eq:problem-penalized}
	(\hL,\hS)\ \in\ \arg\min_{\bL,\bS}  \frac{1}{2}\sum_{g\in\mathcal{G}}\omega_g\|\Pg(\hD-\bL-\bS)\|_{\F}^2 + \lambda_L \|\bL\|_* + \lambda_S \|\bS\|_{\blk,1},
\end{equation}
where $\lambda_L>0$, $\lambda_S>0$ are regularization parameters and $\omega_g>0$ are block-specific weights.
While a detailed justification of~\eqref{eq:problem-penalized} is deferred to the next subsection, we briefly note that $\hL$ estimates the low-rank shared signal, whereas $\hS$ captures blockwise sparse contamination.
We use $\hS$ for active-set identification and $\hL$ for collaborative fine-tuning.
Specifically, denote
the singular value decomposition (SVD) of $\hL$ and the corresponding projection onto the row space as 
\[\hL=\hU\hSigma\hV^\T, \quad P_{\hV}:=\hV\hV^\T.\]
Here, we deliberately denote the right singular vectors of $\hL$ by $\hV$ rather than the conventional $\widehat{\bV}$, to emphasize that $P_{\hV}$ is an estimate of $\PA$.
With $\hPVC:=\bI-\hPV$, we further define the plug-in estimators
\[
\hL_{\bA}:=\hD\hPV,\qquad \hS_{\bA^\perp}:=\hD\hPVC.
\]

For distinct $j,k\in[K]$, write
$P_{j,k}:=P_{(\min\{j,k\},\,\max\{j,k\})}$ and $\hS_{\bA^\perp}^{(j,k)}:=P_{j,k}(\hS_{\bA^\perp})$,
so that $\hS_{\bA^\perp}^{(j,k)}$ denotes the block associated with the unordered pair $\{j,k\}$.
The collaborative set is then estimated by
\begin{equation}\label{eq:A-hat}
	\widehat{\mathcal{C}}_\alpha:=\Bigl\{k \in [K]: \frac{1}{K-1}\sum_{j\neq k} \idop\{\|\hS_{\bA^\perp}^{(j,k)}\|_{\F} \le \tau_n\}\ge \alpha \Bigr\},
\end{equation}
where tuning parameters $\tau_n \ge 0$ and $\alpha \in [0.5,1)$.

For each $k \in \widehat{\mathcal{C}}_\alpha$, define the refined estimator
\begin{equation}\label{eq:Wk-lora}
	\tW^{(k)}=\frac{1}{|\widehat{\mathcal{C}}_\alpha|}\sum_{j\in\widehat{\mathcal{C}}_\alpha}\Bigl(\hW^{(j)} -\hL_{\bA}^{(j,k)}\Bigr),
\end{equation}
where, for all $j,k\in[K]$, the pairwise correction term $\hL_{\bA}^{(j,k)}$ is defined antisymmetrically: $\hL_{\bA}^{(j,k)}=(\hD\hPV)^{(j,k)}$ if $j<k$, $\hL_{\bA}^{(j,k)}=-(\hD\hPV)^{(k,j)}$ if $j>k$, and $\hL_{\bA}^{(j,k)}=\bzero$ if $j=k$.
Thus, $\hW^{(j)}-\hL_{\bA}^{(j,k)}$ is obtained by correcting client $j$'s preliminary estimator toward client $k$.

Equation~\eqref{eq:A-hat} estimates the collaborative set  $\widehat{\mathcal{C}}_\alpha$ by a majority-voting rule. For each retained client, \eqref{eq:Wk-lora} transports the reference estimators toward the target client through the estimated shared component and averages these pairwise adaptations over $\widehat{\mathcal{C}}_\alpha$.
The next subsection explains why this procedure targets the identifiable component of the federated LoRA signal.

\subsection{Canonical estimands and oracle interpretation}\label{sec:construct-D}
We now introduce the latent population objects that motivate the preceding procedure.
For each $g=(j,k)\in \mathcal G$, define
\[
\bD^{(j,k)}:=\bW^{(j)}-\bW^{(k)} \in \mathbb{R}^{q\times p}
\]
and the stacked population contrast matrix
\begin{equation}\label{eq:D}
	\bD:=
	\left[ \{\bD^{(1,2)}\}^\T,\ldots, \{\bD^{(1,K)}\}^\T,\{\bD^{(2,3)}\}^\T,\ldots \{\bD^{(K-1,K)}\}^\T
	\right]^\T \in \mathbb{R}^{Gq\times p}.
\end{equation}
By direct algebra,
\[
\bD^{(j,k)} = (\bB^{(j)} -\bB^{(k)})\bA + [\bB^{(j)} (\bA^{(j)} - \bA) - \bB^{(k)} (\bA^{(k)} - \bA)].
\]
For each $g=(j,k)$, define
\begin{equation}\label{eq:L0}
\bL_0^{(j,k)}:=\begin{cases}
	(\bB^{(j)} - \bB^{(k)})\bA, &\text{if } j, k \in \mathcal{C}, \\
	\bzero_{q\times p},& \text{otherwise},
\end{cases}
\quad
\bS_0^{(j,k)} := \begin{cases}
	\bzero_{q\times p}, & \text{if } j, k \in \mathcal{C}, \\
	\bB^{(j)}\bA^{(j)}-\bB^{(k)}\bA^{(k)}, & \text{otherwise},
\end{cases}
\end{equation}
and
\[
\bE_0^{(j,k)} := \begin{cases}
	\bB^{(j)} (\bA^{(j)} - \bA) - \bB^{(k)} (\bA^{(k)} - \bA), & \text{if } j, k \in \mathcal{C}, \\
	\bzero_{q\times p}, & \text{otherwise}.
\end{cases}
\]
Stacking blockwise yields the decomposition
\[
\bD= \bL_0 + \bS_0 + \bE_0,
\]
where $\bL_0,\bS_0,\bE_0 \in \mathbb{R}^{Gq\times p}$ are formed by vertically concatenating $\{\bL_0^{(j,k)}\}, \{\bS_0^{(j,k)}\},  \{\bE_0^{(j,k)}\}$ across $(j,k)\in \mathcal{G}$.

We analyze the structural roles of $\bL_0$, $\bS_0$, and $\bE_0$ separately.
First, $\bL_0$ is a low-rank matrix satisfying
$\operatorname{rank} (\bL_0)\le r \ll \min\{Gq,p\}$ as $\bA$ is of rank $r$.
Without loss of generality, we may further assume that
\begin{equation}\label{eq:rank-r}
	\rank(\bL_0)=r.
\end{equation}
This condition holds whenever the active clients exhibit sufficient heterogeneity in their left LoRA factors;  for example, it is implied by the existence of a pair $(j_0,k_0)\in\mathcal{C}$ such that $\rank(\bB^{(j_0)}-\bB^{(k_0)})=r$.
Under~\eqref{eq:rank-r}, the row space of $\bL_0$ coincides with the shared LoRA row space spanned by $\bA$, thus we have
\begin{equation}\label{L0}
 \bL_0=\bL_0\PA.
\end{equation}

Next, $\bS_0$ is blockwise sparse, with support determined by the contaminated clients.
Define
\[
\mathcal{I}:=\{(j,k)\in\mathcal{G}: j\in\mathcal{C}^c\ \text{or}\ k \in \mathcal{C}^c\}
\]
as the index set of client pairs involving at least one contaminated client, and let $s:=|\mathcal{I}|$.
The complement $\mathcal{I}^c$ then consists of benign pairs, for which the shared low-rank structure is retained.
Let
$\PI:=\sum_{g\in \mathcal{I}}\Pg$, $\PIC:=\sum_{g\in \mathcal{I}^c}\Pg$,
and define the corresponding block-support subspace
$\mathcal{S}_{\mathcal{I}}:=\{\bX\in\mathbb{R}^{Gq\times p}:\PI(\bX)=\bX\}$.
By construction, we have $P_g(\bS_0)=\bzero$ for all $g\in \mathcal{I}^c$ and thus
\begin{equation}\label{S0}
  \PI(\bS_0)=\bS_0.
\end{equation}

For $\bE_0$, a direct calculation shows that for all distinct $j,k\in [K]$, the corresponding block satisfies
$\|\bE_0^{(j,k)}\|_{\F}\le \delta(\|\bB^{(j)}\|_{\F}+\|\bB^{(k)}\|_{\F})$,
where the maximal approximation error $\delta$ is defined in \eqref{eq:subset-A}. Thus, $\bS_0$ isolates structured heterogeneity associated with contaminated client pairs, whereas $\bE_0$ collects the residual approximation error arising from imperfect alignment of the active clients' row-space factors.

With the constraints (\ref{L0}) and (\ref{S0}), we rewrite the decomposition as
\begin{equation}\label{eq:decomp-original}
	\bD_0= \bL_0 + \bS_0,\quad \bL_0=\bL_0\PA, \quad \PI(\bS_0)=\bS_0.
\end{equation}
We emphasize the two constraints because, as shown later, the identifiable objects are only $(P_A,\mathcal{I})$, rather than the individual components $(\bL_0,\bS_0)$. In practice, the approximation error is typically nonzero, the observed matrix $\hD$ introduced in (\ref{Dhat}) admits the following noisy decomposition:
\begin{equation}\label{eq:decomp-original-noisy}
	\hD= \bL_0 + \bS_0 + \bE,\quad \bL_0=\bL_0\PA, \quad \PI(\bS_0)=\bS_0, \quad  \bE=\bE_0+\bE_1.
\end{equation}
Here, $\bE_0$ is the deterministic approximation error defined above, while $\bE_1$ is induced by the preliminary estimation errors.
Specifically, with
\[
\bXi_k:=\hW^{(k)}-\bW^{(k)}\in\mathbb{R}^{q\times p},
\qquad k\in[K],
\]
the $(j,k)$th block of $\bE_1$ is
\[
\bE_1^{(j,k)}=\bXi_j-\bXi_k .
\]

\begin{remark}[Role of $\bE_1$]\label{remark:E1}
	For each client $k$, the local model follows \eqref{eq:model-k}.
	Given a preliminary estimator $\hW^{(k)}$, obtained from either a linear or nonlinear local fit, our procedure first constructs the stacked matrix of pairwise differences in (\ref{Dhat2}). The term $\bE_1$ collects the resulting stochastic discrepancies in these contrasts, including finite-sample estimation error in $\hW^{(k)}$, optimization error from local training, model misspecification in $f^{(k)}$, and observation noise.

	Under standard regularity conditions, these discrepancies may have sub-Gaussian behavior at the stacked-matrix level. For example, this holds in linear models with sub-Gaussian covariates and noise, and in certain nonlinear models with bounded gradients and sub-Gaussian residual fluctuations.
	The subsequent analysis does \emph{not} require a fully specified form of $f^{(k)}$, and it only uses tail control for the residuals, such as a sub-Gaussian or, more generally, finite Orlicz-norm condition.
	This formulation covers linear regression, index models, and one-layer neural adaptations commonly used in local fine-tuning.
\end{remark}

With the decomposition in (\ref{eq:decomp-original-noisy}), it is natural to consider the noisy recovery problem
\begin{equation}\label{eq:problem-noisy}
\begin{aligned}
    \min_{\bL,\bS}\  &\lambda_L\|\bL\|_{*} + \lambda_S \|\bS\|_{\blk,1} \cr \text{s.t.}\  & \bL=\bL \PA,\  \PI(\bS)=\bS,\ \omega_g \|\Pg(\widehat{\bD} -\bL -\bS)\|_{\F}^2 \leq \veg^2, \ \forall g=(j,k)\in \mathcal{G},
\end{aligned}
\end{equation}
 where $\lambda_L, \lambda_S>0$ are regularization parameters, $\omega_g>0$ are block-specific weights, and $\varepsilon_g:=\|\Pg(\bE)\|_{\F}$ is the blockwise noise level, which combines the deterministic approximation error $\bE_0$ and the preliminary estimation error $\bE_1$ for client pair $g$.

Problem~\eqref{eq:problem-noisy}  is an oracle formulation: it assumes knowledge of the row-space projector $\PA$, the contamination support projector $\PI$, and the blockwise noise levels $\{\varepsilon_g:g\in\mathcal{G}\}$.
These quantities are unavailable in practice.
Accordingly, the implemented estimator of our approach is based on the penalized program~\eqref{eq:problem-penalized}, which serves as a fully data-driven surrogate for~\eqref{eq:problem-noisy}. In the theoretical analysis, we prove that the surrogate in \eqref{eq:problem-penalized}
 automatically yields a solution that matches the oracle problem in \eqref{eq:problem-noisy}.
Before analyzing this surrogate, however, we first clarify what is identifiable from the underlying decomposition, even in the noiseless setting.

Although the decomposition $\bD_0 = \bL_0 + \bS_0$ arises naturally in the federated LoRA setting and has a clear practical interpretation, the pair $(\bL_0,\bS_0)$ is generally not identifiable. The obstruction is geometric: a row-block sparse component may also lie in the shared row space, allowing part of the signal to be shifted between the low-rank and sparse components without changing their sum.
We formalize this ambiguity as follows.
\begin{definition}[Shiftable subspace]\label{def:shiftable-subspace}
	Let $\mathcal{V}_{\bA}:=\{\bX\in\mathbb{R}^{Gq\times p}:\bX=\bX\PA\}$ and $\mathcal{S}_{\mathcal{I}}:=\{\bX\in\mathbb{R}^{Gq\times p}:\PI(\bX)=\bX\}$.
	The shiftable subspace is defined as
	\begin{equation*}
		\mathcal{M}:=\mathcal{S}_{\mathcal{I}}\cap \mathcal{V}_{\bA}
		=\{\bMM\in\mathbb{R}^{Gq\times p}:\PI(\bMM)=\bMM,\ \bMM=\bMM\PA\}.
	\end{equation*}
\end{definition}
The presence of $\mathcal{M}$ creates an intrinsic non-identifiability.
Any $\bMM\in\mathcal{M}$ can be shifted between the low-rank and sparse components without altering feasibility:
$\bD_0=(\bL_0+\bMM)+(\bS_0-\bMM)$.
Thus, the raw pair $(\bL_0,\bS_0)$ is not the appropriate estimand.
To remove this ambiguity, we introduce a canonical representation that is invariant to shifts in $\mathcal{M}$. Specifically, write
\begin{equation}\label{eq:decomp-canonical}
	\hD=\bL_{\bA} + \bS_{\bA^\perp} + \bE, \qquad \bL_{\bA}:=\bD_0\PA,\qquad \bS_{\bA^\perp}:=\bD_0\PAC,
\end{equation}
where $\PAC:=\bI-\PA$.
Then, the projected components $ \bL_{\bA}$ and $\bS_{\bA^\perp}$ are unchanged under every admissible shift in $\mathcal{M}$, thus serving as primary targets for subsequent collaborative refinement and contamination identification.

To clarify the role of $\bS_{\bA^\perp}$ in collaborative-set recovery, observe first that $\bS_{\bA^\perp}^{(j,k)}=\bzero$ for every benign pairs $j,k\in\mathcal C$.
For pairs of contaminated clients, $j,k\in\mathcal C^c$, the corresponding block may be zero or nonzero: it can vanish when the two contaminated effects cancel, or when their difference lies entirely within the shared row space.
We therefore impose a mixed-pair separation condition $\|\bS_{\bA^\perp}^{(j,k)}\|_{\F}>0$ for all  $j\in\mathcal C$ and $k\in\mathcal C^c$.
Under this condition, every benign client has at least $|\mathcal C|-1$ pairwise blocks satisfying
$\bS_{\bA^\perp}^{(j,k)}=\bzero$, while every contaminated client has at most $|\mathcal C^c|-1$ such zero blocks.
Consequently, when $|\mathcal{C}|>|\mathcal{C}^c|$, equivalently $\eta>1/2$, the collaborative-set is identified by the population majority zero-block rule
\[
    \mathcal{C}=\biggl\{k \in [K]: \frac{1}{K-1}\sum_{j\neq k} \idop\Bigl\{\|\bS_{\bA^\perp}^{(j,k)}\|_{\F}=0 \Bigr\}\ge \frac{K\eta-1}{K-1} \biggr\}.
\]

The estimated quantities $\hL_{\bA}=\hD\hPV$ and $\hS_{\bA^\perp}=\hD\hPVC$ introduced in the preceding subsection are precisely plug-in estimators of these canonical targets, which leads directly to the active-set rule \eqref{eq:A-hat} and the refined estimator \eqref{eq:Wk-lora}.

\section{Computation}\label{sec:computation}
This section describes how to solve the penalized problem~\eqref{eq:problem-penalized}. The objective admits an efficient proximal gradient implementation \citep{beck2017first,feng2022projected}.
Denote the loss and penalty function as
\[
\ell(\bL,\bS):=\frac{1}{2}\sum_{g\in \mathcal{G}}\omega_g\|\Pg(\hD-\bL-\bS)\|_{\F}^2, \qquad h(\bL,\bS):=\lambda_L \|\bL\|_* + \lambda_S \|\bS\|_{\blk,1},
\]
respectively.
Define the linear operator $P_\omega:=\sum_{g\in \mathcal{G}}\omega_g\Pg$.
Since each $\Pg$ is an orthogonal projection, $\Pomega$ is self-adjoint and positive semi-definite.
Writing $\bR:=\hD-\bL-\bS$, we have
$\ell(\bL,\bS)=\frac{1}{2}\langle\bR,\ P_\omega\bR \rangle$. Direct differentiation yields
\begin{equation}\label{eq:grad-f}
	\nabla_{\bL} \ell(\bL,\bS)= -\Pomega\bR,\qquad
	\nabla_{\bS} \ell(\bL,\bS)= -\Pomega\bR.
\end{equation}

We apply the proximal gradient method in the product space equipped with the norm $\|(\bL,\bS)\|^2:=\|\bL\|_{\F}^2+\|\bS\|_{\F}^2$.
Given the $m$th iterate $(\hL^m,\hS^m)$ and a step size $t_m>0$ for $m\in [M-1]$,  define the proximal function
\[
\begin{aligned}
	\psi_m(\bL,\bS \mid \hL^m,\hS^m)
	&:= \ell(\hL^m,\hS^m)
	+ \langle \nabla_{\bL} \ell(\hL^m,\hS^m),\ \bL-\hL^m\rangle
	+ \langle \nabla_{\bS} \ell(\hL^m,\hS^m),\ \bS-\hS^m\rangle \\
	&\quad + \frac{1}{2t_m}\Big(\|\bL-\hL^m\|_{\F}^2+\|\bS-\hS^m\|_{\F}^2\Big)
	+ h(\bL,\bS).\nonumber
\end{aligned}
\]
The next iterate is obtained from:
\begin{equation}\label{eq:pg-step}
	(\hL^{m+1},\hS^{m+1}) = \argmin_{\bL,\bS}\ \psi_m(\bL,\bS\mid \hL^m,\hS^m).
\end{equation}
Denote the residual
\[
\bR^m := \hD - \hL^m - \hS^m
\]
and use~\eqref{eq:grad-f}, the problem~\eqref{eq:pg-step} is then equivalent to
\[
\min_{\bL,\bS} \ \lambda_L\|\bL\|_* + \frac{1}{2t_m}\|\bL-(\hL^m + t_m \Pomega\bR^m)\|_{\F}^2 + \lambda_S\|\bS\|_{\blk,1} + \frac{1}{2t_m}\|\bS-(\hS^m + t_m \Pomega\bR^m)\|_{\F}^2,
\]
which separates in $\bL$ and $\bS$.
Consequently,
\begin{equation}\label{eq:pg-closed-form}
	\hL^{m+1}=\mathcal{SVT}\big(\hL^m +t_m \Pomega\bR^m;\ t_m\lambda_L\big), \qquad     \hS^{m+1}=\mathcal{BST}\big(\hS^m + t_m \Pomega\bR^m;\ t_m\lambda_S\big).
\end{equation}
Here, the Singular Value Thresholding operator is defined by
\[
\mathcal{SVT}(\bMM;\tau)=\bU\diag((\sigma_i-\tau)_+)\bV^\T
\]
for $\bMM=\bU\diag(\sigma_i)\bV^\T$, and the Block Soft Thresholding operator acts on each row-block $g$ as
\[
\Pg\bigl(\mathcal{BST}(\bMM;\tau)\bigr) =\Big(1-\frac{\tau}{\|\Pg(\bMM)\|_{\F}}\Big)_+ \Pg(\bMM),\quad g\in\mathcal{G}.
\]

Since the projections $\{\Pg\}_{g\in\mathcal G}$ are mutually orthogonal and $\Pomega=\sum_{g\in\mathcal G}\omega_g\Pg$ is block-diagonal under  the row-block decomposition,
$\|\Pomega\|_{\op} =\max_{g\in\mathcal G}\omega_g$.
The gradient of $\ell$ on the product space is therefore Lipschitz with constant
$L_\ell = 2\|\Pomega\|_{\op} = 2\max_g\omega_g$.
Thus, one may use any constant step size $t \le 1/L_\ell$, or select $t_m$ adaptively by backtracking.

\begin{algorithm}[t]
	\caption{Proximal gradient for CLAIR}
	\label{alg}

	\noindent \textbf{Input:}\hspace{0.5em} $\hD\in\mathbb{R}^{Gq\times p}$, $\lambda_L$, $\lambda_S$, steps $M$, step size $t$, and weights $\{\omega_g\}_{g\in\mathcal{G}}$, thresholds $\alpha$ and $\tau_n$.

	\noindent \textbf{Initialization}:\hspace{0.5em}  $\hL^0 = \bzero_{Gq\times p}$, $\hS^0 = \bzero_{Gq\times p}$.

	\For{$m = 0$ \KwTo $M-1$}{
		$\bR^m=\hD-\hL^m-\hS^m$\;
		$\hL^{m+1} =  \mathcal{SVT}\big(\hL^m +t \Pomega\bR^m;\ t\lambda_L\big)$\;
		$\hS^{m+1}=\mathcal{BST}\big(\hS^m + t \Pomega\bR^m;\ t\lambda_S\big)$\;
	}

	\BlankLine
	Compute  $\hL^M = \hU \hSigma \hV^\top$, $P_{\hV}:=\hV\hV^\T$, and $\hPVC:=\bI-\hPV$\;

	Compute $\hL_{\bA}=\hD\hPV$ and $\hS_{\bA^\perp}=\hD\hPVC$\;

	Compute $\widehat{\mathcal{C}}_\alpha$ using~\eqref{eq:A-hat} and $\tW^{(k)}$ using \eqref{eq:Wk-lora}\;
	\BlankLine
	\noindent \textbf{Output}:\hspace{0.5em}   $\widehat{\mathcal{C}}_\alpha$ and $\{\tW^{(k)}: k\in\widehat{\mathcal{C}}_\alpha\}$.
\end{algorithm}

Algorithm~\ref{alg} summarizes the proximal-gradient solver for~\eqref{eq:problem-penalized} and the subsequent plug-in steps used by CLAIR.
The choice of the majority threshold $\alpha$ in~\eqref{eq:A-hat} is discussed in Remark~\ref{remark:A-hat}, which is justified by the support-recovery theory developed below.
\begin{remark}[Choice of the majority threshold $\alpha$]\label{remark:A-hat}
For collaborative-set recovery, an $\alpha$ in the range of
    \[
	\frac{|\mathcal{C}^c|-1}{K-1}<\alpha\le \frac{|\mathcal{C}|-1}{K-1}.
	\]
    is sufficient to guarantee the consistency of benign-set recovery, as shown in Theorem~\ref{th:set-recovery}. When $K$ is large, this requirement becomes approximately $\alpha\in(1-\eta,\eta]$, where $\eta=|\mathcal{C}|/K$ denotes the proportion of benign clients. Therefore, a non-empty admissible range for $\alpha$ requires $\eta>1/2$, meaning that benign clients must form a majority.
However, this requirement is rather conservative,as it accounts for the unconstrained worst-case scenario in which all contaminated/outlier clients exhibit identical behavior and thereby form another cluster. When contaminated clients are sufficiently heterogeneous, the condition $\eta>1/2$ may be further relaxed; See also the discussion after Theorem \ref{th:set-recovery}.

\end{remark}

\section{Recovery analysis}\label{sec:recovery}
This section studies the recovery properties of CLAIR.
We first clarify the identifiability issue and analyze exact recovery for~\eqref{eq:decomp-original} in the idealized noiseless setting, where the approximation error $\bE_0$ and preliminary estimation error $\bE_1$ are zero.
We then consider the practical noisy setting, where $\hD$ contains both $\bE_0$ and $\bE_1$, and establish stable row-space recovery and consistent collaborative set detection.

\subsection{Identifiability}
Although the low-rank plus block-sparse form in~\eqref{eq:decomp-original} suggests a structured robust PCA-type decomposition, the inherent geometry is different.
The row-block sparse component in \eqref{eq:decomp-original}  may partially align with the shared LoRA row space, allowing the signal to be shifted between the low-rank and sparse components without changing their sum, as illustrated by Definition~\ref{def:shiftable-subspace}.
Consequently, the raw pair $(\bL_0,\bS_0)$ defined in~\eqref{eq:L0} is generally not identifiable, even in the noiseless setting.

To formalize this obstruction, we introduce the relevant notation.
Let $\bL_0=\bU\bSigma\bV^\T$ be a rank-$r$ singular value decomposition of $\bL_0$, where
$\bU\in\mathbb{R}^{Gq\times r}$ and $\bV\in\mathbb{R}^{p\times r}$.
We write
$\mathcal{U}:= \{\bU\bX^\T:\bX\in\mathbb{R}^{p\times r}\}$ and $\mathcal{V}:= \{\bX\bV^\T:\bX\in\mathbb{R}^{Gq\times r}\}$ for the associated column and row space subspaces, respectively.
The tangent space at $\bL_0$ is
\[
\mathcal{T}(\bL_0) := \{\bU\bX_1^\T+\bX_2\bV^\T: \bX_1\in\mathbb{R}^{p\times r},\ \bX_2\in\mathbb{R}^{Gq\times r}\},
\]
with projection
\begin{equation}\label{eq:PT}
	\PT(\bMM) := \bU\bU^\T\bMM +\bMM\bV\bV^\T -\bU\bU^\T\bMM\bV\bV^\T
\end{equation}
for any $\bMM\in \mathbb{R}^{Gq\times p}$. Let $\PTC$ denote the projector onto the orthogonal complement $\mathcal{T}^\perp$.

Classical robust PCA arguments typically rely on transversality between the tangent space $\mathcal{T}(\bL_0)$ and the sparse support subspace $\mathcal{S}_{\mathcal I}$.
This condition fails in the present setting.
Indeed, if $\mathcal{I}\neq\emptyset$, take any nonzero $\bW$ satisfying $\PI(\bW)=\bW$ and set $\bH:=\bW\bV^\T$.
Then $\bH\in\mathcal{T}(\bL_0)\cap\mathcal{S}_{\mathcal I}$ and $\bH\neq\bzero$.
This nontrivial intersection is the geometric obstruction underlying the shift ambiguity.
Consequently, classical exact-decomposition arguments \citep{candes2011robust,chandrasekaran2011rank} do not apply directly.

We address this by handling the column and row spaces of $\bL_0$ separately. We impose rank-sparsity conditions through the column space $\mathcal{U}$, and handle row-space distinguishability by projecting each sparse block onto $\mathcal{V}$ and $\mathcal{V}_{\perp}$.
We first show that $\mathcal{U}$ and the block-sparse subspace $\mathcal{S}_{\mathcal I}$ are transverse under the following condition.

\begin{lemma}\label{lem:rho}
Define $\rho:=\|\PI\PU\PI\|_{\op}$,
	where $\PU$ denotes the orthogonal projector onto $\mathcal{U}$.
	If $\rho<1$, then $\mathcal{U}\cap\mathcal{S}_{\mathcal{I}}=\{0\}$ and $\PA$ is identifiable from the decomposition~\eqref{eq:decomp-original}.
\end{lemma}

The condition $\rho<1$ rules out overlap between the column space and the contaminated block support, but it does not by itself identify the raw pair $(\bL_0,\bS_0)$.  The remaining ambiguity lies in the shiftable subspace $\mathcal{M}$.
We therefore work with the projected canonical components
\[
(\bL_{\bA},\bS_{\bA^\perp}) := (\bD_0\PA,\bD_0\PAC).
\]
These components are invariant to all shifts in $\mathcal{M}$.
Indeed, for any feasible decomposition satisfying~\eqref{eq:decomp-original}, we have $\bL\PAC=\bzero$.
Hence,
\[
\bS\PAC=\bD_0\PAC, \quad \bL+\bS\PA=\bD_0\PA.
\]
Thus, the projected pair $(\bL_{\bA},\bS_{\bA^\perp})$ is uniquely determined by $\bD_0$, even though $(\bL_0,\bS_0)$ itself may not be.

The central object is therefore the projector $\PA$, which encodes the shared row space of the benign low-rank adapters.
Once $\PA$ is identified, the canonical decomposition~\eqref{eq:decomp-canonical} is determined.
Consequently, accurate recovery of $\PA$ from $\bD_0$, or from its noisy counterpart $\hD$, is essential for both contaminated-client identification and collaborative fine-tuning.
We turn to the recovery of $\PA$ in the next subsection.

\subsection{Exact recovery}
This section studies the ideal noiseless decomposition, which provides a benchmark for recovery analysis and isolates the role of structural conditions such as rank-sparsity incoherence.

When $\delta=0$, all benign clients share the same matrix $\bA$, and hence $\bE_0=\bzero$.
Suppose further that $\bE_1=\bzero$. The observed contrast matrix then admits the noiseless decomposition~\eqref{eq:decomp-original}.
To recover the low-rank and block-sparse components in~\eqref{eq:decomp-original}, it is natural to consider the constrained program
\begin{equation}\label{eq:problem-oracle}
	\min_{\bL,\bS}\  \lambda_L\|\bL\|_{*} + \lambda_S \|\bS\|_{\blk,1} \quad\text{s.t.}\quad \bL+\bS=\bD_0,\quad \bL=\bL \PA,\quad \PI(\bS)=\bS.
\end{equation}
The constraint $\bL=\bL\PA$ restricts the low-rank component to the shared row space, whereas $\PI(\bS)=\bS$ imposes the prescribed block support.
Since both $\PA$ and $\mathcal I$ are unknown in practice, we refer to~\eqref{eq:problem-oracle} as the noiseless oracle problem.

The following assumption formally underlies the common row-space constraint.
\begin{assumption}[Active clients share a common space $\mathcal{V}_{\bA}$ ] 
	\label{ass:common-A}
	There exists an active client set $\mathcal{C} \subseteq [K]$ with $|\mathcal{C}|=\eta K$ and  $\eta \in (0,1]$ such that Condition~\eqref{eq:subset-A} and the rank condition~\eqref{eq:rank-r} hold.
\end{assumption}

Recall that the separation condition $\rho<1$  ensures that contamination in the column space can be distinguished from the sparse block components.
To make this conceptual condition verifiable, we introduce concrete assumptions motivated by the federated LoRA setting, together with their practical interpretations.

\begin{assumption}[No client-pair dominates $\mathcal{U}$] 
	\label{ass:separation}
	There exists a constant $\mu\ge 1$ such that
		\begin{equation*}\label{eq:block-IC}
			\max_{g}\|\Pg\PU\|_{\F}^2\le \frac{\mu r}{G}.
			\tag{Block-IC}
		\end{equation*}
\end{assumption}

Assumption~\ref{ass:separation} imposes a blockwise incoherence (Block-IC) condition on the column space $\mathcal U$.
Since $\sum_{g\in\mathcal{G}}\|\Pg\PU\|_{\F}^2=r$, the average blockwise overlap is $r/G$.
Condition~\eqref{eq:block-IC} requires that no single block exceed $\mu$ times this average, thereby preventing $\mathcal{U}$ from being overly concentrated on a small number of client-pair blocks.
Consequently, for any contaminated set $\mathcal I$ with $|\mathcal I|=s$, one has $\rho\le s\mu r/G$, hence separation holds whenever $s\mu r/G<1$.
This is a worst-case condition, since it ignores cancellations or misalignment across corrupted blocks. Combined with an upper bound on $s$, Assumption~\ref{ass:separation} therefore yields a sufficient condition for $\rho<1$, as formalized in Lemma~\ref{lem:rho-worst-case}.

\begin{lemma}\label{lem:rho-worst-case}
	Under Assumption~\ref{ass:common-A}-\ref{ass:separation}, if $\eta > \sqrt{1-\frac{1}{\mu r}+\frac{1}{4(K-1)}}$, then $\rho<1$.
\end{lemma}

In practice, both $\PA$ and $\mathcal{I}$ are unknown.
We therefore consider the unconstrained noiseless program
\begin{equation}\label{eq:problem-original}
	\min_{\bL,\bS}\ \lambda_L\|\bL\|_{*} + \lambda_S \|\bS\|_{\blk,1} \quad\text{s.t.}\quad \bL+\bS=\bD_0.
\end{equation}
The following exact recovery result shows that, although the optimizer of~\eqref{eq:problem-original} may not be unique, every optimal low-rank component recovers the same shared row space $\mathcal{V}_{\bA}$.
\begin{theorem}\label{th:PV-noiseless}
	Assume $\delta=0$, and suppose the noiseless decomposition $\bD_0 = \bL_0 + \bS_0$ satisfies Assumption~\ref{ass:common-A}-\ref{ass:separation}.
	Let constants $a,b\in(0,1)$ satisfy $\sqrt{\mu rs/G}<ab/(1+a)^2$.
	If
	\begin{equation}\label{eq:lambda-ratio}
		\frac{\lambda_S}{\lambda_L}\in\biggl(
		\frac{\sqrt{\mu r/G}}{b-(1+a)\sqrt{\mu rs/G}},
		\ \frac{a}{(1+a)\sqrt{s}}
		\biggr),
	\end{equation}
	then, for \emph{any} optimal solution $(\hL,\hS)$ of~\eqref{eq:problem-original} with compact SVD $\hL=\hU\hSigma\hV^\T$,
	the row-space projector $P_{\hV}:=\hV\hV^\T$ is uniquely determined and satisfies
    \begin{equation}\label{thm1}
    P_{\hV}=P_{\bA}.
    \end{equation}
\end{theorem}

Theorem~\ref{th:PV-noiseless} establishes exact recovery of the shared row-space projector in a nonidentifiable low-rank plus block-sparse decomposition induced by federated LoRA contrasts. The main challenge is to show that this row-space conclusion remains valid after the oracle constraints on $\PA$ and $\mathcal I$ are removed in program~\eqref{eq:problem-original}.
We address this by constructing a dual certificate at an arbitrary oracle optimizer and showing that it certifies optimality for~\eqref{eq:problem-original}.
Thus, the oracle problem~\eqref{eq:problem-oracle} is used only as a proof tool to reveal the identifiable target.

The regularizer requirement~\eqref{eq:lambda-ratio} ensures that the dual certificate satisfies the required subgradient conditions.
Its upper bound keeps the correction on contaminated support within the nuclear-norm subgradient,  while its lower bound makes the block penalty large enough to control leakage on uncontaminated blocks.
With such a certificate, any optimizer of~\eqref{eq:problem-original} must agree with the oracle optimizer except possibly along the nonidentifiable shift directions in $\mathcal{M}$.
The separation condition rules out any shift that changes the component orthogonal to the oracle row space, and the identity $\PIC(\hL)=\bL_0$ ensures that the recovered row space contains the true one.
Together, these facts imply $P_{\hV}=P_{\bA}$ for every optimizer, which is precisely the quantity needed for active-set detection and collaborative refinement. 

\subsection{Stable recovery}\label{sec:stable-recovery}
Now we consider the noisy decomposition problem with nonzero approximation error $\bE_0$ and estimation error $\bE_1$. We begin by characterizing the behavior of these noises.

\begin{assumption}\label{ass:bounded-B}
	There exists a constant $0<B<\infty$ such that $\max_k\|\bB^{(k)}\|_{\F}\le B$.
\end{assumption}

\begin{assumption}\label{ass:hete-B-moment}
	Let  $\bar{\bB}:=|\mathcal{C}|^{-1}\sum_{k\in\mathcal{C}}\bB^{(k)}$ and
	\[
	\bSigma_B:= \frac{1}{|\mathcal{C}|}\sum_{k\in\mathcal{C}}(\bB^{(k)}-\bar{\bB})^{\T}(\bB^{(k)}-\bar{\bB}) \in\mathbb R^{r\times r}.
	\]
	There exists constants $0<\kappa_0\le \kappa_1<\infty$ such that $\kappa_0 \le \lambda_{r}(\bSigma_B) \le  \lambda_{1}(\bSigma_B)\le \kappa_1.$
\end{assumption}

Assumption~\ref{ass:bounded-B} complements the common structure condition in Assumption~\ref{ass:common-A}, and further yields a uniform bound on the approximation error $\bE_0$.
Assumption~\ref{ass:hete-B-moment} imposes sufficient heterogeneity in the client-specific factors $\{\bB^{(k)}\}$ to identify the shared row space $\mathcal{V}_{\bA}$, and may be regarded as a moment-type strengthening of the rank condition~\eqref{eq:rank-r}.

As discussed in Remark~\ref{remark:E1},  the term $\bE_1$ captures preliminary estimation error in $\hW^{(k)}$, optimization error from local training, model misspecification of $f^{(k)}$, and other stochastic perturbations. For concreteness, Assumption~\ref{ass:E1} imposes a sub-Gaussian condition on such preliminary estimation errors. More general tail assumptions can be handled by analogous arguments.
\begin{assumption}\label{ass:E1}
	Assume $\{\bXi_k\}_{k=1}^K$ are independent zero mean random matrices, especially each $\bXi_k$ is sub-Gaussian in the sense that
	\[
	\sup_{\|\bB\|_{\F}=1}\ \bigl\|\langle \bB,\bXi_k\rangle\bigr\|_{\psi_2}\ \le \tau_k,\quad \tau_k:= \frac{\sigma_k}{\sqrt{n_k}}
	\]
	for some constant $\tau_k>0$, where $\sigma_k>0$ and $n_k$ is the sample size of client $k$. Moreover, denote \[\tau:=\max_{k\in[K]} \tau_k.\]
\end{assumption}

\begin{theorem}[Asymptotic noisy recovery]\label{th:PV-noisy}
	Suppose the noisy decomposition $\hD=\bL_0+\bS_0+\bE$, with $\bE=\bE_0+\bE_1$, satisfies Assumption~\ref{ass:common-A}--\ref{ass:hete-B-moment}. Assume further that $\bE_1$ with the blockwise representation $\bE_1^{(j,k)} = \bXi_j -\bXi_k$ for $(j,k)\in\mathcal{G}$ satisfies Assumption~\ref{ass:E1}.
	Let $(\hL,\hS)$ be any optimal solution of the penalized problem~\eqref{eq:problem-penalized}, and let $\hPV$ be an orthogonal projector onto a dominant $r$-dimensional right singular subspace of $\hL$.
	Let $a,b\in(0,1)$ satisfy $\sqrt{\mu rs/G}<ab/(1+a)^2$ and
	\begin{equation}
		\frac{\lambda_S}{\lambda_L}\in\biggl( \frac{\sqrt{\mu r/G}}{b-(1+a)\sqrt{\mu rs/G}},\ \frac{a}{(1+a)\sqrt{s}}\biggr).
	\end{equation}
	Suppose also that, $\lambda_L\asymp K^{-1/2}$, $\delta=O(K^{-1/2})$, $\omega_{g}\asymp K^{-1}$ for $g\in\mathcal{G}$, and $K\tau^2=O(1)$.
	Then, when $K\rightarrow\infty$, we have
\begin{equation}\label{th2-1}
    \|\hPV-\PV\|_{\op}=O_{\pr}(K^{-1/2}).
\end{equation}

\end{theorem}

The preceding noisy analysis shows that the exact-recovery phenomenon is stable under the perturbations arising in federated LoRA, namely imperfect approximation of the common row space and preliminary local estimation error.
The rate $\|\hPV-\PV\|_{\op}=O_{\pr}(K^{-1/2})$, showing that increasing the number of clients sharpens estimation of the shared adaptation subspace, even when the proportion of contaminated clients remains fixed.

\subsection{Collaborative-set recovery}\label{sec:S-estimate}
In this subsection, we show that the estimator
\[
\widehat{\mathcal{C}}_\alpha:=\Bigl\{k \in [K]: \frac{1}{K-1}\sum_{j\neq k} \idop\{\|\hS_{\bA^\perp}^{(j,k)}\|_{\F} \le \tau_n\}\ge \alpha \Bigr\}
\]
given by~\eqref{eq:A-hat} recovers the benign client set $\mathcal{C}$ under a mild signal-gap condition.
Recall that for any benign block $g=(j,k)\in\mathcal{G}$ with $j,k\in\mathcal{C}$, the block $\bD_0^{(g)}$ lies in the row space of $\bA$, and hence $\bS_{\bA^\perp}^{(j,k)}=\bzero$ for every benign pair $j\neq k$.
However, for $(j,k)$ involving at least one contaminated client, the behavior of $\bS_{\bA^\perp}^{(j,k)}$ is not automatically determined, as discussed at the end of Section~\ref{sec:construct-D}.
For this, we define the minimum mixed-pair signal level
\[
	\beta_{\min}:=\min_{\substack{j\in\mathcal{C},\,k\in\mathcal{C}^c}}\|\bS_{\bA^\perp}^{(j,k)}\|_{\F},
\]
and require it to dominate the uniform projected estimation error
\[
	\varepsilon_{\sup}:=\max_{g\in\mathcal{G}}\|\Pg(\hS_{\bA^\perp}-\bS_{\bA^\perp})\|_{\F}.
\]
The collaborative set recovery argument then follows.
We first establish a high probability bound for $\varepsilon_{\sup}$.
\begin{lemma}\label{lem:epsilon-max}
	Suppose the assumptions of Theorem~\ref{th:PV-noisy} hold. Then
	\[
	\varepsilon_{\sup} =O_{\pr}\Bigl(\sqrt{\frac{\log K}{K}}\Bigr).
	\]
\end{lemma}

The following theorem combines this uniform error control with a mixed-pair signal gap to establish collaborative-set recovery in the noisy setting.

\begin{theorem}[Collaborative-set recovery]\label{th:set-recovery}
    Assume $|\mathcal C|=\eta K$ with $\eta>1/2$.
    Suppose
    \[
        \frac{|\mathcal C^c|-1}{K-1}<\alpha \le	\frac{|\mathcal C|-1}{K-1}
    \]
    and $\beta_{\min}>0$.

    (i) Under the assumptions of Theorem~\ref{th:PV-noiseless}, suppose further $0\le \tau_n <\beta_{\min}$,  then  $\widehat{\mathcal C}_{\alpha}=\mathcal C$.

    (ii) Under the assumptions of Theorem~\ref{th:PV-noisy}, suppose further $c\sqrt{\log K/ K} < \tau_n < \beta_{\min} - c\sqrt{\log K/ K} $ for sufficiently large $K$ and some constant $c>0$.  If $\sqrt{\log K/K}/ \beta_{\min} \to 0$, then $			\pr(\widehat{\mathcal C}_{\alpha}=\mathcal C)\to1.$
\end{theorem}

The mixed-pair separation condition $\beta_{\min}>0$ is mild: it requires only that benign--contaminated pairs leave a nonzero orthogonal signal, while imposing no restriction on pairs of contaminated clients.
The majority condition $\eta>1/2$ ensures that the admissible interval for $\alpha$ is nonempty and prevents contaminated clients from passing the zero-block counting threshold.
As discussed earlier, this condition can be further relaxed by imposing additional assumptions on pairs of contaminated clients, such as requiring them to be sufficiently heterogeneous.

\section{Error analysis}\label{sec:error-analysis}
We now quantify the statistical gain of collaborative fine-tuning by analyzing the CLAIR refined estimator $\tW^{(k)}$.  Recall from~\eqref{eq:Wk-lora} that, for each $k \in \widehat{\mathcal{C}}_\alpha$,
\[
\tW^{(k)}=\frac{1}{|\widehat{\mathcal{C}}_\alpha|}\sum_{j\in\widehat{\mathcal{C}}_\alpha}\Bigl(\hW^{(j)}+\hL_{\bA}^{(k,j)}\Bigr).
\]
Define the local mean squared error
\[
	\mathcal{E}^{(k)}:=\E\|\hW^{(k)}-\bW^{(k)}\|_{\F}^2 =\E\|\bXi_k\|_{\F}^2,
\]
and its projected counterparts in the space of $\bA^\perp$ and $\bA$ as
\begin{equation}\label{EVC} \EVC^{(k)}:=\E\|\bXi_k\PVC\|_{\F}^2, \quad \EV^{(k)}:=\E\|\bXi_k\PV\|_{\F}^2.
\end{equation}
Further let $\theta_k$ and $1-\theta_k$ denote the respective ratios of $\EVC$ and $\EV$ to $\mathcal{E}^{(k)}$, i.e.,
\begin{equation}\label{eq:error-ratio}
		\theta_k:=	\frac{\EVC^{(k)}}{\mathcal{E}^{(k)}}, \quad 1-\theta_k=\frac{\EV^{(k)}}{\mathcal{E}^{(k)}}
	\end{equation}
When $\bXi_k$ is evenly distributed over the space spanned by $\PV\cup\PVC$, we have
\[
\theta_k=O\left( \frac{p-r}{p} \right), \quad 1-\theta_k=O\left( \frac{r}{p} \right),
\]
In later analysis, we will show that $\EVC^{(k)}$ corresponds to the noise reducible by CLAIR, whereas $\EV^{(k)}$ is not reducible. Consequently, in the LoRA setting where $r \ll p$, CLAIR can yield substantial estimation gains.

We now introduce a quantity that measures the oracle CLAIR gain for client $k$, along with its ratio to $\EVC^{(k)}$:
\begin{equation}\label{HFL}
  \mathcal{H}^{(k)}:=\EVC^{(k)}-\frac{1}{|\mathcal{C}|^2}\sum_{j\in\mathcal{C}}\EVC^{(j)},\quad \phi_k:=\frac{ \mathcal{H}^{(k)}}{ \EVC^{(k)}}.
\end{equation}
Here, the quantity $\mathcal{H}^{(k)}$ measures the oracle collaborative gain
in the orthogonal complement of the shared row space, as will be illustrated in Theorem~\ref{th:oracle-mse-decomp}. Furthermore, $\phi_k$ measures the fraction of the total reducible noise captured by this oracle gain; clearly, $\phi_k < 1$.
Although $\mathcal{H}^{(k)}$ and $\phi_k$ are not necessarily positive, such as in the case of a benign client with an unusually small orthogonal complement error, they are typically on the order of $\EVC^{(k)}$. 
In particular, if $\{\bXi_j\}_{j\in\mathcal{C}}$ are identically distributed, then
\[
    \mathcal{H}^{(k)}= \frac{|\mathcal{C}|-1}{|\mathcal{C}|} \EVC^{(k)}, \quad \phi_k=\frac{|\mathcal{C}|-1}{|\mathcal{C}|}.
\]

We now state the estimation gain of CLAIR under an oracle setting where $\hPV=\PV$ and $\delta=0$.
In practice, estimating $\PV$ incurs an additional cost and introduces extra misalignment error, which will be discussed in Theorem~ \ref{prop:noisy-gain} and \ref{th:refined-W-order}.
\begin{theorem}[Oracle CLAIR gain]\label{th:oracle-mse-decomp}
    Assume the conditions of Theorem ~\ref{th:PV-noisy} hold with $\delta=0$ and assume $\hPV=\PV$.
	Fix $k\in\mathcal{C}$, then the exact oracle MSE gain of the refined estimator relative to the local estimator is
	\begin{equation}\label{eq:oracle-gain-general}
		\E\|\tW^{(k)}-\bW^{(k)}\|_{\F}^2 =(1-\phi_k\theta_k) \	\E\|\hW^{(k)}-\bW^{(k)}\|_{\F}^2.
	\end{equation}
    In particular, if $\{\bXi_j\}_{j\in\mathcal{C}}$ are identically distributed and evenly distributed over the space $\PV \cup \PVC$, then, for any $k \in \mathcal{C}$, we have
    \[
    \E\|\tW^{(k)}-\bW^{(k)}\|_{\F}^2 =\left[1-\frac{(|\mathcal{C}|-1)(p-r)}{|\mathcal{C}|p}\right] \	\E\|\hW^{(k)}-\bW^{(k)}\|_{\F}^2.
    \]
\end{theorem}

Theorem \ref{th:oracle-mse-decomp} provides three insights. First, the gain comes entirely from reducible noise in the orthogonal complement of the shared row space. 
Second, averaging over benign clients yields a $1/|\mathcal{C}|$ variance-reduction effect,
making the benefit of a larger collaborative set explicit. Third, a lower LoRA rank $r$ leads to a larger CLAIR gain.

Now we consider the practical setting with an additional noisy decomposition cost. Define the cost for client $k$ and its ratio to the total local risk $\mathcal{E}^{(k)}$ as follows:
\begin{equation} \mathcal{E}_{\mathrm{cost}}^{(k)}:=\E\|- \widebar{\bE}_{0}\PVC +\bigl(\hW^{(k)}-\widebar{\bW}\bigr)(\hPV-\PV)\|_{\F}^2, \quad \varphi_k:=	\frac{\mathcal{E}_{\mathrm{cost}}^{(k)}}{\mathcal{E}^{(k)}}.
	\end{equation}
where $\widebar{\bW}:=\sum_{j\in\mathcal{C}}\hW^{(j)}/|\mathcal{C}|$ and $\widebar{\bE}_{0}:=\sum_{j\in\mathcal{C}}\bE_{0}^{(k,j)}/|\mathcal{C}|$ with $\bE_{0}^{(j,k)} = -\bE_{0}^{(k,j)} $ and $\bE_{0}^{(k,k)}=\bzero$.
Here, $\mathcal{E}_{\mathrm{cost}}^{(k)}$ represents the additional cost of estimating $\PA$, consisting of a benign misalignment bias $- \widebar{\bE}_{0}\PVC$ and a subspace estimation error controlled by $\hPV-\PV$.
We next show that improvement of the CLAIR estimator depends on a tradeoff between the collaborative gain $\mathcal{H}^{(k)}$ and the noisy-decomposition cost $\mathcal{E}_{\mathrm{cost}}^{(k)}$.
\begin{theorem}[Noisy CLAIR gain]\label{prop:noisy-gain}
	Suppose assumptions of Theorem~\ref{th:set-recovery}-(ii) hold, $|\mathcal{C}|\ge 2$, and Assumptions~\ref{ass:bounded-B} and~\ref{ass:E1} are satisfied.
	Fix $k\in\mathcal{C}$, suppose $\phi_k>0$.
    Further assume that there exists a certain constant $\zeta_k\in(0,\phi_k)$ such that \begin{equation}\label{eq:refine-condition2}
		 \varphi_k<\left(\sqrt{1-\zeta_k\theta_k} - \sqrt{1-\phi_k\theta_k}\right)^2.
	\end{equation}
	Then the MSE of the CLAIR can be controlled by
\begin{equation}\label{eq:noisy-gain}
	\E\|\tW^{(k)}-\bW^{(k)}\|_{\F}^2
		\le (1-\zeta_k\theta_k) \ \E\|\hW^{(k)}-\bW^{(k)}\|_{\F}^2.
	\end{equation}
\end{theorem}

\begin{remark}\label{remark:SNR}
Condition~\eqref{eq:refine-condition2} can be interpreted as a signal-to-cost requirement.
   Here, a larger $\theta_k$ corresponds to a greater fraction of reducible local risk, while a larger $\phi_k$ indicates that the oracle gain accounts for a larger share of the total reducible risk. In contrast, a smaller $\varphi_k$ lowers the non-oracle cost. Thus, CLAIR benefits client $k$ when the cost $\varphi_k$ is dominated by the squared gain
$\left(\sqrt{1-\zeta_k\theta_k} - \sqrt{1-\phi_k\theta_k}\right)^2$.
In the following theorem, we show that the cost requirement \eqref{eq:refine-condition2} is trivial under a balanced benign error covariance scenario.

\end{remark}

Under a balanced benign error covariance scenario, the requirement (\ref{eq:refine-condition2}) is automatically satisfied, yielding an explicit CLAIR gain in Theorem \ref{th:refined-W-order}.

\begin{assumption}[Balanced benign error covariance]\label{ass:cov-comparable}
	Let $\bSigma_j:=\cov(\vec(\bXi_j))$ for each $j\in\mathcal C$ and $\tau$ be defined as in Assumption \ref{ass:E1}.
	Assume that there exist constants $0<\kappa_{-}\le \kappa_{+}<\infty$ such that
	\[
		\kappa_{-} \tau^2\bI_{pq} \preceq \bSigma_j \preceq \kappa_{+} \tau^2 \bI_{pq}, \qquad j\in\mathcal C.
	\]
\end{assumption}

\begin{theorem}[Noisy CLAIR gain under balanced benign error]\label{th:refined-W-order}
	Suppose assumptions of Theorem~\ref{th:set-recovery}-(ii) hold. Further assume Assumption~\ref{ass:cov-comparable} and Assumption ~\ref{ass:bounded-B} hold $B=O(\sqrt{qr})$.  Suppose $|\mathcal C|=\eta K$ for some fixed $\eta\in(0,1]$. When $r\ll p$ and $\{\bXi_j\}_{j\in\mathcal{C}}$ are evenly distributed over the space $\PV \cup \PVC$, we have  (\ref{eq:refine-condition2}) holds automatically. Moreover, the MSE of the CLAIR can be controlled by
\begin{equation}\label{eq:noisy-gain-1}
	\E\|\tW^{(k)}-\bW^{(k)}\|_{\F}^2
		\le \left(\frac{c_0}{\eta K}+\frac{c_1 r}{p-r}\right)   \E\|\hW^{(k)}-\bW^{(k)}\|_{\F}^2,
	\end{equation}
   where $c_0$ and $c_1$ are certain positive constants. 
\end{theorem}

Theorem \ref{th:refined-W-order} suggests that, even in the noisy case with estimation error, CLAIR successfully removes nearly all the error in the subspace $\PVC$ when $K$ is large. As for the irreducible component on $\PV$, this error is much smaller under the LoRA setting with $r \ll p$. Together, these results demonstrate the benefits of CLAIR.

	\section{Simulation}\label{sec:simulation}
		We illustrate the behavior of CLAIR in a multiple-response linear model. For client $k$, data are generated from
		\[
		\bY^{(k)}=\bW^{(k)}\bX^{(k)}+\bE^{(k)},
		\]
		where $\bX^{(k)}\in\mathbb{R}^{p\times n}$ has independent standard normal entries. The common backbone $\bW_0\in\mathbb{R}^{q\times p}$ and the common row factor $\bA\in\mathbb{R}^{r\times p}$ are sampled entrywise from $\mathrm{Unif}[-1,1]$, with $r=2$. For each benign client, $\bW^{(k)}=\bW_0+0.8\,\bB^{(k)}\bA$, where $\bB^{(k)}\in\mathbb{R}^{q\times r}$ is again sampled entrywise from $\mathrm{Unif}[-1,1]$.

		We allow $40\%$ of clients to be contaminated by replacing the low-rank client-specific adaptation by an unstructured perturbation, with entries drawn from $\mathrm{Unif}[-1,1]$ and scaled by  $c/\sqrt{q(p-r)}$, where the replicate-level signal $c$ is drawn uniformly from $\{3,4,5,6\}$. Thus $|\mathcal{C}|=0.6K$ and the expected orthogonal Frobenius contamination strength is comparable across dimensions. The noise covariance across the $q$ response coordinates follows an $\mathrm{AR}(1)$ structure with correlation parameter $0.25$ and scale $1$.

		We consider $(p,q)\in\{(10,10),(20,20),(50,50)\}$ and $K\in\{5,10,20\}$. The per-client sample size is set to $n=100,150,300$ for $p=10,20,50$, respectively, and each configuration is repeated over $100$ independent Monte Carlo replicates.
		Local estimators are computed by ordinary least squares. CLAIR uses pair weights $\omega_g=1/K$, majority threshold $\alpha=0.5$, and a largest-gap data-adaptive threshold $\tau_n$.
	The regularization parameters are set as $\lambda_L=c_1K^{-1/2}$ and $\lambda_S=c_2K^{-3/2}$, where the constants $c_1$ and $c_2$ are tuned on values of $(p,q,n)$.

		 We compare local OLS, CLAIR, FedAvg over all clients, and oracle FedAvg over the true benign clients, with oracle FedAvg leaving contaminated clients at their local OLS estimates. Table~\ref{tab:sim1-mse} reports the mean squared Frobenius coefficient error $\|\hW^{(k)}-\bW^{(k)}\|_{\F}^{2}$ for each method.

	\begin{table}[!ht]
		\centering
			\caption{Mean Frobenius-squared error averaged over clients and $100$ replicates, with $|\mathcal{C}^c|=0.4K$ contaminated clients.}
			\label{tab:sim1-mse}
			\begin{tabular}{ccccrrrr}
				\hline
				$(p,q)$ & $n$ & $K$ & $|\mathcal{C}^c|$ & Local & CLAIR & FedAvg (Oracle) & FedAvg \\ \hline
				&  & 5 & 2 & 1.121 & 1.109 & 6.102 & 15.262 \\
				$(10,10)$ & $100$ & 10 & 4 & 1.117 & 0.722 & 7.457 & 17.094 \\
				&  & 20 & 8 & 1.126 & 0.648 & 8.156 & 18.144 \\
				&  &  &  &  &  &  &  \\
				&  & 5 & 2 & 3.085 & 2.055 & 24.715 & 35.575 \\
				$(20,20)$ & $150$ & 10 & 4 & 3.093 & 1.735 & 29.977 & 38.623 \\
				&  & 20 & 8 & 3.097 & 1.581 & 32.617 & 41.660 \\
				&  &  &  &  &  &  &  \\
				&  & 5 & 2 & 10.052 & 6.070 & 148.871 & 179.855 \\
				$(50,50)$ & $300$ & 10 & 4 & 10.054 & 5.146 & 183.551 & 201.606 \\
				&  & 20 & 8 & 10.041 & 4.678 & 200.977 & 212.679 \\ \hline
			\end{tabular}
		\end{table}

		This experiment is deliberately challenging for naive model averaging. In all settings, FedAvg over all clients has much larger coefficient error than local OLS, as contaminated clients are averaged into the global estimator. Oracle FedAvg also performs poorly because the benign clients have client-specific low-rank perturbations, so averaging their coefficient matrices introduces substantial personalization bias. In contrast, CLAIR improves over local OLS in every reported setting.

	To further evaluate CLAIR, Table~\ref{tab:sim1-set} reports the collaborative-set recovery accuracy, $(\mathrm{TP}+\mathrm{TN})/K$, and the contaminated-client recall, $\mathrm{TN}/|\mathcal{C}^c|$, where $\mathrm{TP}$ denotes the number of benign clients correctly retained and $\mathrm{TN}$ denotes the number of contaminated clients correctly excluded. The results show that CLAIR accurately 
        identifies collaborative sets across a range of settings. In addition, Figure~\ref{fig:sim1-PV}
 presents the convergence of $\hPV$ as $K$ increases under different $(p,q,n)$ regimes, further illustrating the asymptotic behavior of CLAIR in estimating $\PA$.

	\begin{table}[!ht]
		\centering
			\caption{Collaborative-set recovery results with $|\mathcal{C}^c|=0.4K$ contaminated clients.}
			\label{tab:sim1-set}
				\begin{tabular}{cccrr}
					\hline
					$(p,q)$ & $n$ & $K$ & Accuracy & Contam. recall \\ \hline
					&  & 5 & 0.980 & 0.990 \\
					$(10,10)$ & $100$ & 10 & 1.000 & 1.000 \\
					&  & 20 & 1.000 & 1.000 \\
					&  &  &  &  \\
					&  & 5 & 1.000 & 1.000 \\
					$(20,20)$ & $150$ & 10 & 1.000 & 1.000 \\
					&  & 20 & 1.000 & 1.000 \\
					&  &  &  &  \\
					&  & 5 & 0.916 & 0.790 \\
					$(50,50)$ & $300$ & 10 & 0.920 & 0.800 \\
					&  & 20 & 0.900 & 0.750 \\ \hline
				\end{tabular}
		\end{table}

\begin{figure}[!ht]
		\centering
		\includegraphics[width=0.5\textwidth]{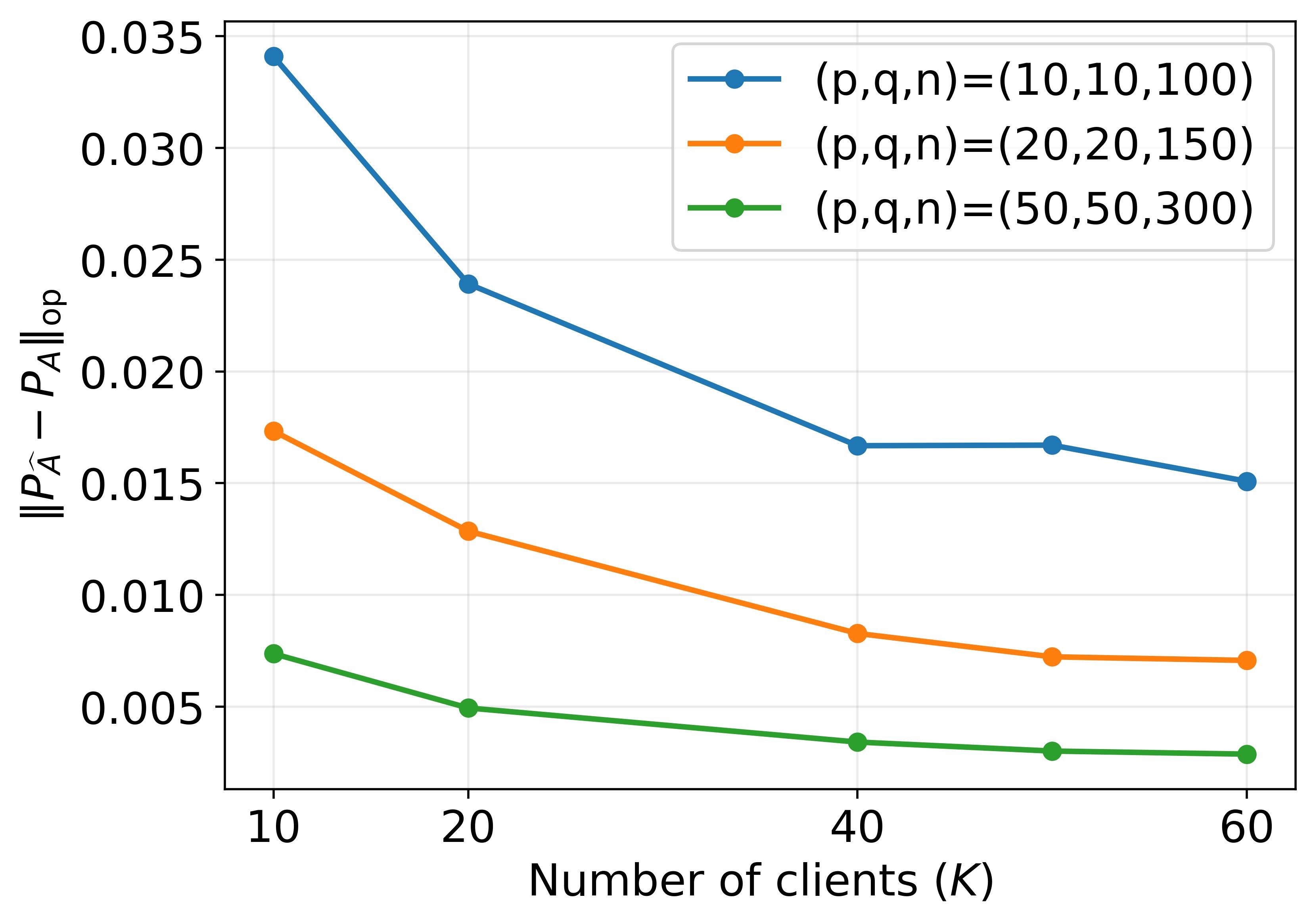}
		\caption{Estimation error of $\hPV$ compared to $K$ across $(p,q,n)$ regimes.}
		\label{fig:sim1-PV}
	\end{figure}

\section{Experiments on sequence copying tasks}\label{sec:copying}
We conduct a federated LoRA fine-tuning experiment with Transformer models, assigning each client a copying-based reasoning task following~\citep{song2025out}. 	This controlled design allows client heterogeneity, low-rank adaptation, and client-level contamination to be directly examined.
The experiment evaluates whether CLAIR can identify contaminated clients and improve benign-client prediction through selective collaborative refinement.

\subsection{Task and model}
In a copying task, the input sequence $\bx=(u_1,\ldots,u_T)$ contains a segment $\bu^\#$ of length $L<T/2$ that appears twice.
The model must use the earlier occurrence of this segment to predict the next tokens. For example, after observing ``[A], [l], [i], [c], [e]'' in $\bu^\#$, a model prompted with ``[A], [l]'' should generate ``[i], [c], [e]'' successively under next-token prediction.
Although simple to state, the task requires content-based retrieval and positional alignment, making it a controlled benchmark for a primitive form of algorithmic reasoning in LLMs.

In this experiment, all sequences have length $T=64$ and are generated over a discrete vocabulary
\[
    \mathcal A=\{a,\ldots,z,A,\ldots,Z,\&\},
\]
consisting of $52$ letter tokens and one padding token ``$\&$''.
Given a segment $\bu^\#=(u_1,\ldots,u_L)$ sampled from a client-specific distribution $P_x$ on $\mathcal A$, we construct a sequence
\[
    \bx=(*,\ \bu^\#,\ *,\ \bu^\#,\ *),
\]
where the two copies of $\bu^\#$ are placed randomly and the remaining positions $*$ are filled with random tokens.
Varying the segment length $L$ and sampling distribution $P_x$ across client-specific copying tasks naturally induces heterogeneity and can create contaminated clients.

In this experiment, each client $k$ is assigned a specific copying task. Following a standard parameter-efficient fine-tuning pipeline, we first pretrain an attention-based Transformer $f_0$ on a base copying task and then fine-tune it for each client using LoRA on local sequences $\{\bx_i^{(k)}\}_{i=1}^{n_k}$.
The Transformer contains two attention blocks, each with projection matrices $\{\bW_q,\bW_k,\bW_v,\bW_o\}$, and LoRA is applied to all eight projections.
A subsequent collaboration step, implemented either by a FedAvg-based procedure or by CLAIR, is applied separately to the client-specific adapted weights associated with each projection matrix.

The model is trained autoregressively by next-token prediction from prefixes.
The architecture and training protocol follow standard Transformer practice; implementation details are provided in the Appendix.

\subsection{Experimental settings}\label{sec:exp-settings}
We consider a federated network of nine benign clients and one contaminated client, and explore two experimental regimes: a homogeneous one where all benign clients share the same copying task, and a heterogeneous one where they follow the same copying rule but differ in token distribution $P_x$ and copy length $L$.
In both regimes, the contaminated client outputs a model adapted from a mismatched copying task compared to benign clients.

In the homogeneous regime, all benign clients share a common power-law sampling distribution \[P_x(X=a;t) = \frac{a^{-t}}{ \sum_{\ell=1}^{52}\ell^{-t}}\] for all $a\in [52]$, with exponent $t=1.1$, and uses a fixed copy length $L=16$.
In the heterogeneous regime, each benign client has a client-specific power-law exponent drawn from $[0.95,1.6]$ and a copy length drawn from $\{10,\ldots,26\}$ independently in each replicate.
In both regimes, each client trains LoRA adapter locally on $N_{\mathrm{train}}=2000$ generated sequences, using batch size $50$, one epoch, learning rate $0.001$, and LoRA rank $r=3$.

We compare CLAIR with several baselines: local LoRA, FedAvg over all clients, and oracle FedAvg over the benign client set. Importantly, oracle FedAvg relies on the true collaborative set and is therefore unavailable in practice; it is included only as an ideal contamination-free benchmark.
For CLAIR, we set $(\lambda_L,\lambda_S)=(0.5,0.4)$ in the homogeneous regime and $(0.5,0.2)$ in the heterogeneous regime, with collaborative-set detection thresholds $0.5$ and $0.01$, respectively, and voting threshold $\alpha=0.5$.
Each method is evaluated using the masked next-token error, which excludes the first copy and the first three positions of the second copy, as in \citep{song2025out}. Specifically, on the test sequences   $\mathcal{D}^{\mathrm{test}}:=\{\bx_i=(x_{i,1},\ldots,x_{i,T})\}_{i=1}^{N_{\mathrm{test}}}$, let $m_i$ be the starting position of the second copied segment and let $L_i$ be its length. Given the predictive distribution $\widehat{p}_t(\cdot\mid \bx_{i,<t})$ given by model $\widehat{f}$, define $\widehat{x}_{i,t}= \argmax_{a\in\mathcal A} \widehat{p}_t(a\mid \bx_{i,<t}).$
The masked next-token error is then
\[
	\operatorname{Err}_{\mathrm{mask}}= \frac{1}{\sum_{i=1}^{N_{\mathrm{test}}} |\mathcal M_i|} \sum_{i=1}^{N_{\mathrm{test}}} \sum_{t\in\mathcal M_i} \idop\{\widehat{x}_{i,t}\ne x_{i,t}\},
	\quad
	\mathcal M_i=\{m_i+3,\ldots,m_i+L_i-1\}.
\]
Equivalently, the reported masked next-token accuracy is
\[
	\operatorname{Acc}_{\mathrm{mask}} = 1-\operatorname{Err}_{\mathrm{mask}}.
\]

\subsection{Results}\label{sec:exp-results}
CLAIR identifies the contaminated client exactly in both regimes.
Layerwise recovery is exact for all eight attention matrices in the homogeneous regime, but less uniform in the heterogeneous regime, where benign client variation can itself induce layerwise contrasts.
We therefore report client-level CLAIR results using the estimated collaborative sets $\widehat{\mathcal C}$ rather than oracle contamination labels.

\begin{figure}[!ht]
    \centering
    \includegraphics[width=\textwidth]{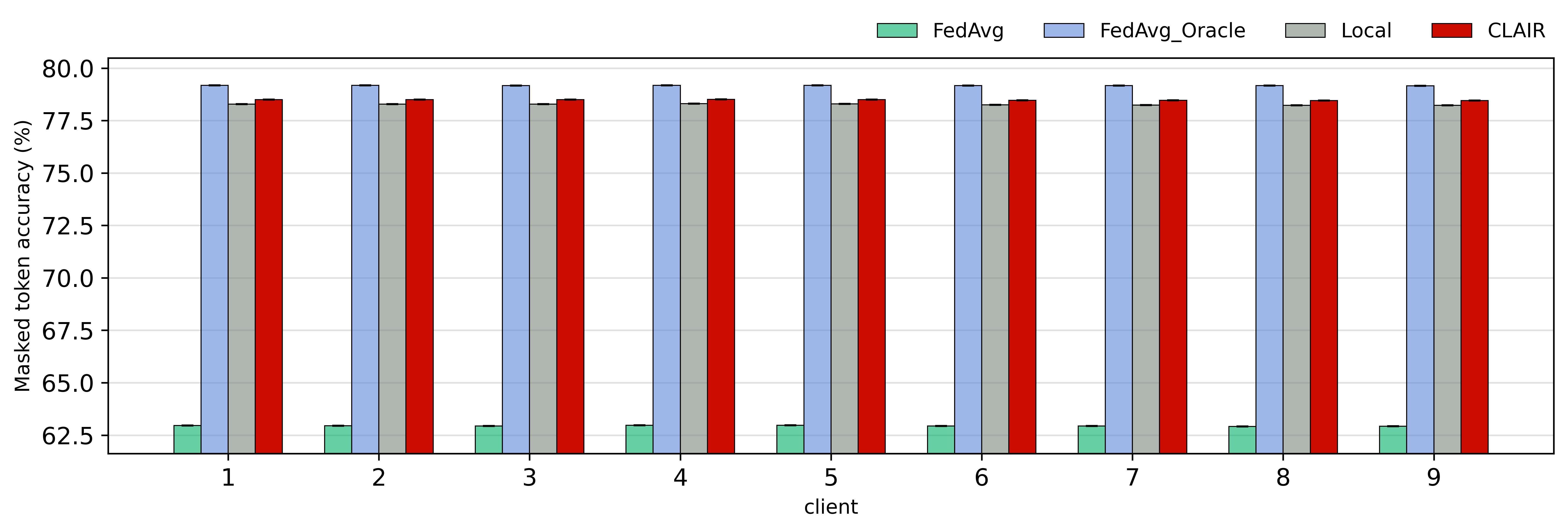}
    \includegraphics[width=\textwidth]{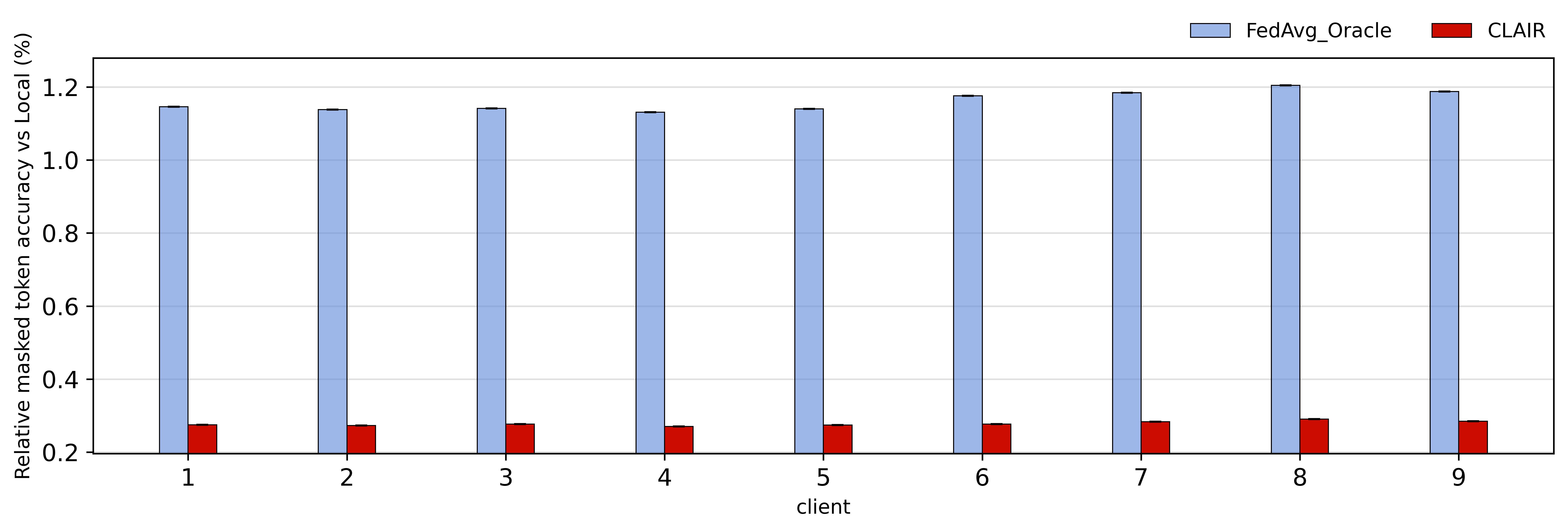}
    \caption{Homogeneous copying experiment evaluated on the common copying task.  The top panel reports client-level masked next-token accuracy averaged over $100$ replicates; the bottom panel reports relative accuracy change with respect to local LoRA fine-tuning.}
    \label{fig:sim2-homo}
\end{figure}

\begin{figure}[!ht]
    \centering
    \includegraphics[width=\textwidth]{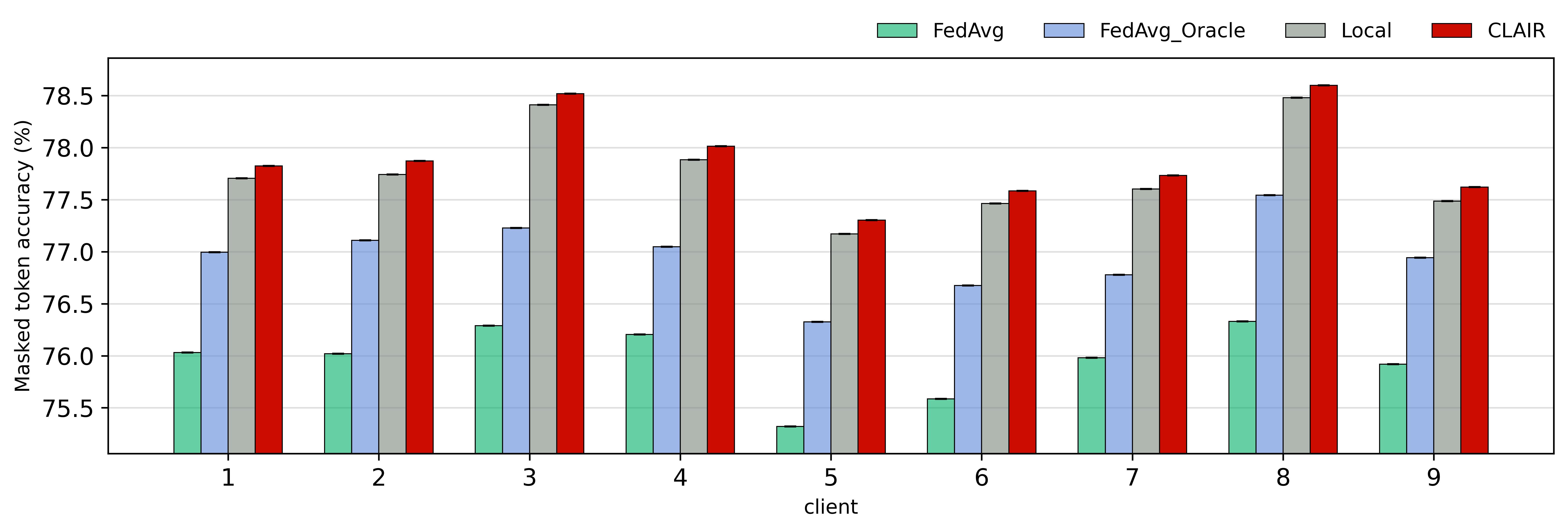}
    \includegraphics[width=\textwidth]{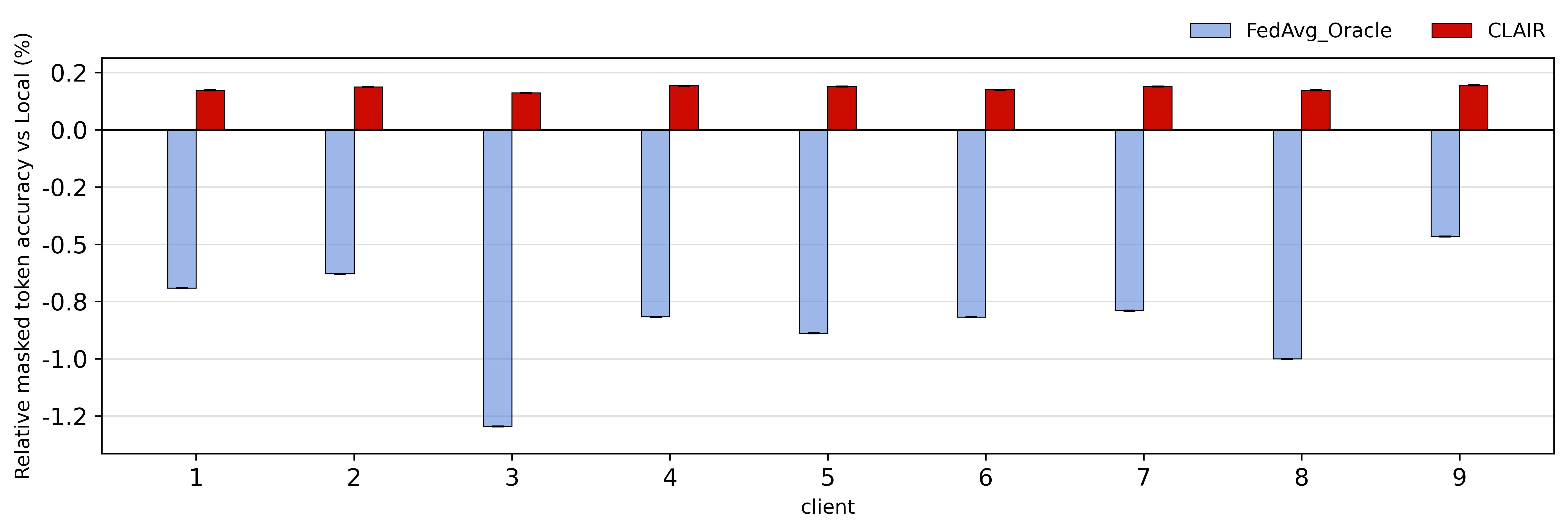}
    \caption{Heterogeneous copying experiment evaluated on client-specific tasks. The top panel reports client-level masked next-token accuracy averaged over $100$ replicates; the bottom panel reports relative accuracy change with respect to local LoRA fine-tuning.}
    \label{fig:sim2-hete}
\end{figure}

\begin{table}[!ht]
    \caption{Average masked next-token accuracy (\%). Results are averaged over $9$ benign clients and $100$ replicates.}
    \begin{tabular}{lrrrr}
        \hline
        & Local & FedAvg & FedAvg (Oracle) & CLAIR \\ \hline
        Homogeneous & 78.27 & 62.94 & \textbf{79.17} & 78.48 \\
        Heterogeneous & 77.77 & 75.96 & 76.96 & \textbf{77.90} \\ \hline
    \end{tabular}
    \label{tab:copying}
\end{table}

Table~\ref{tab:copying} summarizes the averaged results across benign clients and $100$ Monte Carlo replicates. In the homogeneous regime, oracle FedAvg achieves the highest accuracy, $79.17\%$, as expected when all benign clients share the same data distribution; CLAIR nevertheless yields a modest improvement over local LoRA. In the heterogeneous regime, CLAIR successfully avoids negative knowledge transfer and attains the highest average own-client accuracy. These results show that removing the contaminated client alone is insufficient to eliminate the personalization cost incurred by averaging heterogeneous benign clients.

To illustrate client-wise performance,  Figures~\ref{fig:sim2-homo} and~\ref{fig:sim2-hete} further report the masked accuracy and relative accuracy with respect to local LoRA for the nine benign clients.
Figure~\ref{fig:sim2-homo} evaluates clients under the common distribution $P_x$ in the homogeneous setting, whereas Figure~\ref{fig:sim2-hete} uses each client's own evaluation distribution in the heterogeneous setting.
These results highlight two complementary advantages of CLAIR: robustness against contaminated clients and client-specific collaborative refinement.

\section{Summary}\label{sec:summary}
We develop CLAIR as a contamination-aware framework for collaborative federated LoRA fine-tuning when the common backbone is unknown and only preliminary local estimators are available.
We explain why CLAIR can improve over local fine-tuning without directly averaging incompatible adapters.
The MSE analysis further provides insights that collaboration is beneficial only when the oracle variance reduction from averaging benign clients dominates the costs caused by row-space heterogeneity, subspace estimation, and noisy decomposition. Thus, the theory gives both a positive mechanism for federated LoRA gains and a diagnostic condition under which negative transfer may occur.

Future work should relax several structural restrictions.
One direction is to replace the single-majority collaborative set by multiple or overlapping collaborative groups, allowing layer-specific ranks and client-dependent adapter subspaces. This extension would connect naturally with mixture-of-experts federated LoRA, where clients may select different expert adapters rather than share one global subspace.
Another direction is to relax the current preliminary estimator assumption by deriving error conditions from the dynamics of deep learning fine-tuning, including stochastic optimization error, nonlinear representation drift, and module-wise interactions. These problems warrant further investigation.

\clearpage
\section*{Supplementary Material}
\addcontentsline{toc}{section}{Supplementary Material}
\setcounter{section}{0}
\setcounter{subsection}{0}
\setcounter{subsubsection}{0}
\setcounter{equation}{0}
\setcounter{figure}{0}
\setcounter{table}{0}
\setcounter{lemma}{0}
\setcounter{assumption}{0}
\setcounter{theorem}{0}
\setcounter{proposition}{0}
\setcounter{corollary}{0}
\setcounter{definition}{0}
\setcounter{remark}{0}
\renewcommand{\thesection}{S\arabic{section}}
\renewcommand{\thelemma}{S\arabic{lemma}}
\renewcommand{\theassumption}{S\arabic{assumption}}
\renewcommand{\thetheorem}{S\arabic{theorem}}
\renewcommand{\theproposition}{S\arabic{proposition}}
\renewcommand{\thecorollary}{S\arabic{corollary}}
\renewcommand{\thedefinition}{S\arabic{definition}}
\renewcommand{\theremark}{S\arabic{remark}}
\renewcommand{\thefigure}{S\arabic{figure}}
\renewcommand{\thetable}{S\arabic{table}}
\renewcommand{\theHsection}{supp.\arabic{section}}
\renewcommand{\theHsubsection}{supp.\arabic{section}.\arabic{subsection}}
\renewcommand{\theHsubsubsection}{supp.\arabic{section}.\arabic{subsection}.\arabic{subsubsection}}
\renewcommand{\theHequation}{supp.\arabic{section}.\arabic{equation}}
\renewcommand{\theHfigure}{supp.\arabic{figure}}
\renewcommand{\theHtable}{supp.\arabic{table}}
\renewcommand{\theHlemma}{supp.\arabic{lemma}}
\renewcommand{\theHassumption}{supp.\arabic{assumption}}
\renewcommand{\theHtheorem}{supp.\arabic{theorem}}
\renewcommand{\theHproposition}{supp.\arabic{proposition}}
\renewcommand{\theHcorollary}{supp.\arabic{corollary}}
\renewcommand{\theHdefinition}{supp.\arabic{definition}}
\renewcommand{\theHremark}{supp.\arabic{remark}}
\numberwithin{equation}{section}

\section{Proofs for exact recovery (Section 5.2)}
\subsection{Notation}
We first introduce tools for characterizing the set of low-rank and block-wise sparse matrices.
Let $\bL=\bU\Sigma \bV^\T$ be the singular value decomposition, where $\bU\in\mathbb{R}^{Gq\times r}$, $\bV\in \mathbb{R}^{p\times r}$, and $\bSigma\in \mathbb{R}^{r\times r}$ is diagonal.
The column and row spaces of $\bL$ are defined as
\[
\mathcal{U}=\{\bU \bX^\T \mid \bX\in \mathbb{R}^{p\times r}\},\quad  \mathcal{V}=\{\bX\bV^\T \mid \bX\in \mathbb{R}^{Gq\times r}\},
\]
respectively.
The tangent space at $\bL$ is
\[
\mathcal{T}(\bL)=\{\bU \bX_1^\T + \bX_2\bV^\T \mid \bX_1\in \mathbb{R}^{p\times r}, \bX_2\in \mathbb{R}^{Gq\times r}\}.
\]
Let $\PT$ denote the projection operator onto the tangent space $\mathcal{T}$, given by
\begin{equation}\label{supp:eq:PT}
	\PT(\bMM):=\bU\bU^\T \bMM + \bMM\bV\bV^\T - \bU\bU^\T\bMM\bV\bV^\T
\end{equation}
for any $\bMM\in \mathbb{R}^{Gq\times p}$, and letting $\PTC$ denote the projector onto the orthogonal complement $\mathcal{T}^\perp$.
The subdifferential of the nuclear norm at $\bL$ is
\[
\partial \|\bL\|_* = \{\bU\bV^\T + \bMM: \bU^\T\bMM=\bzero,\ \bMM\bV=\bzero,\ \|\bMM\|_2\le 1\}.
\]
Recall from Section~3.2 of the main paper that $\bX^{(g)}$ denotes the $q\times p$ block indexed by $g=(j,k)\in\mathcal{G}$, while $\Pg(\bX)$ is the corresponding zero-padded projection.
Define the blockwise sign operator by
\[
\begin{aligned}
    \Sign_{\mathcal I}(\bS)
    :=
    \Bigl\{
    \bMM\in\mathbb R^{Gq\times p}:\
    &\PI(\bMM)=\bMM,\  \bMM^{(g)}=\frac{\bS^{(g)}}{\|\bS^{(g)}\|_{\F}}
    \quad \text{for all } g\in  \operatorname{supp}_{\blk}(\bS),\\
    &\|\bMM^{(g)}\|_{\F}\le 1
    \quad \text{for all } g\in \mathcal I\setminus  \operatorname{supp}_{\blk}(\bS)
    \Bigr\},
\end{aligned}
\]
where $\operatorname{supp}_{\blk}(\bS):=\bigl\{g\in \mathcal{G}:\bS^{(g)} \neq\bzero\bigr\}$.
Define the dual block norm $\|\bZ\|_{\blk,\infty}=\underset{g\in \mathcal{G}}{\max}\|\bZ^{(g)}\|_{\F}$.
The subdifferential of the block $\ell_1$ norm is
\[
\partial\|\bS\|_{\blk,1}= \Bigl\{\bMM\in\mathbb R^{Gq\times p}: \PI(\bMM)\in \Sign_{\mathcal I}(\bS),\ \|\PIC(\bMM)\|_{\blk,\infty}\le 1 \Bigr\}.
\]

\subsection{Separation condition}
\begin{proof}[Proof of Lemma 1] 
	For any $\bMM\in\mathcal{U}\cap\mathcal{S}_{\mathcal{I}}$, we have $\PU\bMM=\bMM$ and $\PI\bMM=\bMM$.
	Hence,
	\[
	\bMM=\PI\PU\PI\bMM
	\]
	and  $\|\bMM\|_{\F}\le \rho\|\bMM\|_{\F}$.
	With $\rho<1$, this forces $\|\bMM\|_{\F}=0$ thus  $\bMM=0$.
\end{proof}

\begin{proof}[Proof of Lemma 2]
	Condition~(Block-IC) in Assumption~2 of the main paper  posits a uniform blockwise bound
	\[
		\|\Pg\PU\|_{\F}^2 \le \frac{\mu r}{G}
	\]
	for all $g\in\mathcal{G}$ with some $\mu\ge 1$.
	Then we have
	\[
		\rho =\|\PU\PI\PU\|_{\op} \le \sum_{g\in\mathcal{I}} \|\PU\Pg\PU\|_{\op} =\sum_{g\in\mathcal{I}}\|\Pg\PU\|_{\op}^2   \le s \cdot \frac{\mu r}{G},
	\]
	where $s=|\mathcal{I}|$.
	A sufficient condition for ensuring $\rho<1$ is therefore
	\[\mu r s/G < 1.\]
	Recall that the contaminated index set $\mathcal{I}$ contains all pairs $g=(j,k)$ such that at least one of the two client $j$ or $k$ is an outlier.
	Writing $|\mathcal{C}| = \eta K$ for the number of active clients, a simple counting argument gives
	\[
	s= G - \binom{\eta K}{2} = G\Bigl(1 - \eta^2 + \frac{\eta(1-\eta)}{K-1}\Bigr).
	\]
	Thus
	\[
	\frac{s}{G} = 1-\eta^2 + \frac{\eta(1-\eta)}{K-1} \le 1-\eta^2 +\frac{1}{4(K-1)},
	\]
	and
	\[
	\eta > \sqrt{1-\frac{1}{\mu r}+\frac{1}{4(K-1)}}
	\]
	implies $\rho<1$.
\end{proof}

\subsection{Supporting results}
Before proving Theorem~1 of the main paper,
we first examine several consequences of the oracle constraints in (23) 
The next lemma justifies the identification of the row space induced by an oracle optimizer, a fact used repeatedly in subsequent arguments.
\begin{lemma}\label{supp:lem:pPV}
	Let $(\pL,\pS)$ be an optimal solution of (23).
    Suppose $\delta=0$ and Assumption~1 of the main paper holds, then
	\[
	\rank(\pL)=r,\quad \row(\pL)=\row(\bA).
	\]
	Consequently, for a compact SVD $\pL=\pU\pSigma\pVT$, we have $\pPV=\PA$.
\end{lemma}
\begin{proof}[Proof of Lemma~\ref{supp:lem:pPV}]
	Since $\delta=0$, we have $\bD_0=\bL_0+\bS_0$. The oracle feasibility conditions imply
	\[
	\pL+\pS=\bD_0,\qquad \pL=\pL\PA,\qquad \PI(\pS)=\pS.
	\]
	Hence $\PIC(\pS)=\bzero$, and because $\PI(\bS_0)=\bS_0$,
	\[
	\PIC(\pL)=\PIC(\bD_0-\pS)=\PIC\bD_0=\PIC\bL_0=\bL_0.
	\]
	Therefore
	\[
	\rank(\pL)\ge \rank(\PIC\pL)=\rank(\bL_0)=r,
	\]
	where the last equality uses Condition (12). 
	On the other hand, $\pL=\pL\PA$ implies $\row(\pL)\subseteq\row(\bA)$, so
	\[
	\rank(\pL)\le \rank(\PA)=r.
	\]
	Thus $\rank(\pL)=r$. Since $\row(\pL)\subseteq\row(\bA)$ and both spaces have dimension $r$, we conclude that $\row(\pL)=\row(\bA)$. Consequently $\pPV=\PA$.
\end{proof}

We relate the feasible solutions of the oracle problem to the model-induced noiseless pair as follows.
In particular, the lemma implies that, when the shiftable subspace $\mathcal M\neq\{0\}$, the oracle feasible set is non-singleton, so the oracle optimizer need not be unique.

\begin{lemma}\label{supp:lem:oracle-feasible-class}
    Let $(\bL_0,\bS_0)$ satisfy (15). 
    A pair $(\pL,\pS)$ is feasible for the oracle problem (23)  if and only if there exists $\bMM\in\mathcal{M}$ (shiftable subspace) such that
    \begin{equation}\label{supp:eq:oracle-feasible-shift}
        \pL=\bL_0+\bMM,\qquad \pS=\bS_0-\bMM.
    \end{equation}
\end{lemma}
\begin{proof}[Proof of Lemma~\ref{supp:lem:oracle-feasible-class}]
	First, let $(\pL,\pS)$ be oracle feasible, so $\pL+\pS=\bD_0$, $\pL=\pL\PA$, and $\PI(\pS)=\pS$.
	Set
	\[
	\bMM:=\pL-\bL_0=(\bD_0-\pS)-(\bD_0-\bS_0)=\bS_0-\pS.
	\]
	Then
	\[
	\PI\bMM=\PI(\pL-\bL_0)=\PI(\bS_0-\pS)=\bS_0-\pS=\bMM
	\]
	since $\PI\bS_0=\bS_0$.
	Also,
	\[
	\bMM\PA=(\pL-\bL_0)\PA=\pL-\bL_0=\bMM.
	\]
	Hence $\bMM\in\mathcal{M}$, and $\pL=\bL_0+\bMM,\ \pS=\bS_0-\bMM$.

	Conversely, if $\bMM\in\mathcal{M}$ and $(\pL,\pS)$ are given by \eqref{supp:eq:oracle-feasible-shift}, then
	\[
	\pL+\pS=\bD_0,\qquad  \pL=\pL\PA
	\]
	since $\bL_0=\bL_0\PA$ and $\bMM=\bMM\PA$, and $\PI(\pS)=\pS$ since $\bS_0$ and $\bMM$ are supported on $\mathcal{I}$.
	Thus $(\pL,\pS)$ is oracle feasible.
\end{proof}

The next lemma derives a first-order identity for the oracle problem, linking the low-rank subgradient term to a projected block-norm subgradient term. It serves as a technical tool in dual-certificate construction.

\begin{lemma}\label{supp:lem:PIUVT}
    Let  $(\pL,\pS)$ be an optimal solution to (23), with compact SVD $\pL=\pU\pSigma\bV^{\prime\T}$.
    Then there exists a matrix $\bH\in \partial\|\pS\|_{\blk,1}$ such that
    \[
        \lambda_L\PI\pU\pVT = \lambda_S \PI(\bH\pPV).
    \]
\end{lemma}
\begin{proof}[Proof of Lemma~\ref{supp:lem:PIUVT}]
	Recall that $\mathcal{V}=\{\bX:\bX=\bX\PA\}$ and $\mathcal{S}:=\{\bX:\PI\bX=\bX\}$, and their normal spaces are given by $\mathcal{V}^\perp:=\{\bY:\bY\PA=0\}$ and $\mathcal{S}^\perp:=\{\bY:\PI\bY=0\}$, respectively.
	We can write Lagrangian for the oracle problem~(23) as follows:
	\[
	\mathcal{L}(\bL, \bS, \bQ) = \lambda_L\|\bL\|_* + \lambda_S \|\bS\|_{\blk,1} +  \idop_{\mathcal{V}}(\bL) + \idop_{\mathcal{S}}(\bS) - \langle \bQ, \bL + \bS - \bD_0 \rangle.
	\]
	Take the subdifferential with respect to $\bL$ and $\bS$ respectively, and use that fact that $(\pL,\pS)$ is optimal, we have
	\[
	\bQ \in \partial \lambda_L\|\pL\|_* + \partial \idop_{\mathcal{V}}(\pL)
	\]
	and
	\[
	\bQ \in \partial \lambda_S\|\pS\|_{\blk,1} + \partial \idop_{\mathcal{S}}(\pS).
	\]
	Note that
	\[
	\partial \idop_{\mathcal{V}}(\pL) = \mathcal{N}_{\mathcal{V}}(\pL) = \mathcal{V}^\perp,
	\]
	where $\mathcal{N}_{\mathcal{V}}(\pL):= \{\bY: \langle \bY, \bZ - \pL \rangle \le 0, \forall \bZ \in \mathcal{V}\}$ is the normal cone of $\mathcal{V}$ at $\bL$, and the second equality holds because $\mathcal{V}$ is a linear subspace.
	Similarly, we have
	\[
	\partial \idop_{\mathcal{S}}(\pS) = \mathcal{S}^\perp.
	\]
	Thus
	\[
	\bQ \in \partial \lambda_L\|\pL\|_* + \mathcal{V}^\perp,\qquad \bQ \in \partial \lambda_S\|\pS\|_{\blk,1} + \mathcal{S}^\perp.
	\]

	Therefore, there exists $\bG\in \partial\|\pL\|_* $ with normal complement $\bY_{\mathcal{V}} \in \mathcal{V}^\perp$ and $\bH\in\partial \|\pS\|_{\blk,1}$ with normal complement $\bY_{\mathcal{S}}\in \mathcal{S}^\perp$ such that
	\[
	\bQ=\lambda_L\bG + \bY_{\mathcal{V}}= \lambda_S\bH + \bY_{\mathcal{S}}.
	\]
	Using the nuclear norm subgradient structure, we may take
	\[
	\bG:=\pU\pVT + \bW,\quad \bW \in\mathcal{T}(\pL)^\perp,\ \|\bW\|_{\op}\le 1.
	\]
	Crucially, we have $\bW\pPV=0$.
	Also, oracle feasibility gives $\pL=\pL\PA$, hence $\row(\pL)\subseteq\row(\bA)$ and therefore $\pPV=\PA\pPV$. Since $\bY_{\mathcal{V}}\in \mathcal{V}^\perp$, we have $\bY_{\mathcal{V}}\PA=0$, so $\bY_{\mathcal{V}}\pPV=\bY_{\mathcal{V}}\PA\pPV=0$.
	Hence,
	\[
	\bQ\pPV = \lambda_L\pU\pVT = \lambda_S\bH\pPV + \bY_{\mathcal{S}}\pPV.
	\]
	Applying $\PI$, we further have
	\[
	\lambda_L\PI\pU\pVT = \lambda_S\PI\bH\pPV + \PI\bY_{\mathcal{S}}\pPV.
	\]
	Since $\bY_{\mathcal{S}}\in \mathcal{S}^\perp$ means $\PI\bY_{\mathcal{S}}=0$,
	we get
	\[
	\lambda_L\PI\pU\pVT = \lambda_S\PI\bH\pPV.
	\]
	This proves the result.
\end{proof}

Lemma~\ref{supp:lem:oracle-to-original} gives the conditions that a dual certificate must satisfy, employing the optimal solution pair of the oracle problem (23). 

\begin{lemma}\label{supp:lem:oracle-to-original}
    Let  $(\pL,\pS)$ be an optimal solution to (23), with compact SVD $\pL=\pU\pSigma\bV^{\prime\T}$. Suppose there exists $\bQ\in \mathbb{R}^{Gq\times p}$ and constants $a,b \in (0,1)$ such that
    \begin{equation}\label{supp:eq:KKT}
        \left\{
        \begin{aligned}
            &\pPT(\bQ) = \lambda_L\pU\pVT,\\
            &\|\pPTC(\bQ)\|_{\op} \le a\lambda_L, \\
            &\PI(\bQ)/\lambda_S \in  \Sign_{\mathcal{I}}(\pS), \\
            &\|\PIC(\bQ)\|_{\blk,\infty} \le b\lambda_S.
        \end{aligned}
        \right.
    \end{equation}
    Then $(\pL,\pS)$ is also an optimal solution of problem~(24).  
\end{lemma}
\begin{proof}[Proof of Lemma~\ref{supp:lem:oracle-to-original}]
    Since $a<1$, we have $\|\pPTC(\bQ)\|_{\op}<\lambda_L$ and $\pPT(\bQ)=\lambda_L\pU\pVT$, hence
    \[
    \bQ/\lambda_L \in \partial\|\pL\|_*.
    \]
    Also,  $\|\PIC(\bQ/\lambda_S)\|_{\blk,\infty} \le b<1$ and $\PI(\bQ/\lambda_S ) \in \Sign_{\mathcal{I}}(\pS)$ implies that
    \[
    \bQ/\lambda_S\in \partial\|\pS\|_{\blk,1}.
    \]
    Take any feasible $(\bL,\bS)$ with $\bL+\bS=\bD_0$.
    By the subgradient inequality for convex functions,
    \[
    \|\bL\|_* \ge \|\pL\|_* + \langle\frac{1}{\lambda_L}\bQ, \bL-\pL\rangle
    \]
    and
    \[
    \|\bS\|_{\blk,1} \ge \|\pS\|_{\blk,1} + \langle\frac{1}{\lambda_S}\bQ, \bS-\pS\rangle.
    \]
    Adding gives
    \[
    \lambda_L\|\bL\|_* + \lambda_S\|\bS\|_{\blk,1} \ge \lambda_L\|\pL\|_* + \lambda_S\|\pS\|_{\blk,1}  + \langle\bQ, (\bL+\bS)-(\pL+\pS)\rangle.
    \]
    Since both pairs satisfy $\bL+\bS=\bD_0$, the inner product term is zero. Hence $(\pL,\pS)$ is optimal for problem~(24).
\end{proof}

\begin{lemma}\label{supp:lem:subgradient-ineq}
	Let $(\pL,\pS)$ be an optimal solution of~(24),
    with compact SVD $\pL=\pU\pSigma\pVT$. Suppose there exists $\bQ\in \mathbb{R}^{Gq\times p}$ and constants $a,b \in (0,1)$ such that ~\eqref{supp:eq:KKT} holds at $(\pL,\pS)$.
	For any perturbation $\bDelta$, we have
	\[
	\begin{aligned}
		\|\pL+\bDelta\|_* \ge \|\pL\|_* + \langle \bQ/\lambda_L, \bDelta\rangle + (1-a)\|\pPTC\bDelta\|_*,\\
		\|\pS - \bDelta\|_{\blk,1} \ge \|\pS\|_{\blk,1} - \langle\bQ/\lambda_S,\bDelta\rangle + (1-b)\|\PIC\bDelta\|_{\blk,1}.
	\end{aligned}
	\]
	In particular, if $(\pL+\bDelta,\pS-\bDelta)$ is also optimal for~(24), then $\pPTC(\bDelta)=0$ and $\PIC(\bDelta)=0$.
\end{lemma}
\begin{proof}[Proof of Lemma~\ref{supp:lem:subgradient-ineq}]
	Write $\bQ/\lambda_L=\pU\pVT + \bZ$ with $\bZ:=\pPTC(\bQ/\lambda_L)$ and $\|\bZ\|_{\op}\le a$.
	Fix any perturbation $\bDelta$. By the duality between nuclear and operator norms, there exists $\bW\in\mathcal{T}(\pL)^\perp$ such that
	\[
	\|\bW\|_{\op}=1,\quad \langle\bW,\pPTC\bDelta\rangle = \|\pPTC\bDelta\|_*.
	\]
	Then $\pU\pVT + \bW \in \partial\|\pL\|_*$.
	By the subgradient inequality,
	\[
	\begin{aligned}
		\|\pL+\bDelta\|_* & \ge \|\pL\|_* + \langle\pU\pVT + \bW,\bDelta\rangle\\
		&\ge \|\pL\|_* + \langle\pU\pVT,\bDelta\rangle + \|\pPTC\bDelta\|_*
	\end{aligned}
	\]
	Note that
	\[
	\begin{aligned}
		\langle \bQ/\lambda_L,\bDelta \rangle &= \langle\pU\pVT, \bDelta\rangle + \langle\bZ,\pPTC\bDelta\rangle \\
		& \le  \langle\pU\pVT, \bDelta\rangle + a\cdot\|\pPTC\bDelta\|_*.
	\end{aligned}
	\]
	Plugging in gives
	\[\begin{aligned}
		\|\pL+\bDelta\|_* &\ge \|\pL\|_* + \langle\pU\pVT,\bDelta\rangle +\|\pPTC\bDelta\|_*\\
		&\ge \|\pL\|_* +\langle\bQ/\lambda_L,\bDelta\rangle + (1-a)\|\pPTC\bDelta\|_*.
	\end{aligned}
	\]

    For the sparse term, note that $\pS=\PI(\pS)$, so
    \[
        \|\pS-\bDelta\|_{\blk,1} =\|\PI(\pS-\bDelta)\|_{\blk,1}+\|\PIC(\bDelta)\|_{\blk,1}.
    \]
    Since $\PI(\bQ)/\lambda_S\in \Sign_{\mathcal{I}}(\pS)$, convexity on the support subspace gives
    \[
        \|\PI(\pS-\bDelta)\|_{\blk,1}\ge\|\pS\|_{\blk,1}
    -\langle\PI(\bQ)/\lambda_S, \PI(\bDelta)\rangle.
    \]
    Hence
    \[
        \|\pS-\bDelta\|_{\blk,1}\ge\|\pS\|_{\blk,1}-\langle\bQ/\lambda_S,\bDelta\rangle
        +\langle\PIC(\bQ)/\lambda_S,\PIC(\bDelta)\rangle
        +\|\PIC(\bDelta)\|_{\blk,1}.
    \]
    By block duality,
    \[
        |\langle \PIC(\bQ)/\lambda_S,\PIC(\bDelta)\rangle|\le
    \|\PIC(\bQ)/\lambda_S\|_{\blk,\infty}\cdot
    \|\PIC(\bDelta)\|_{\blk,1}
    \le b\|\PIC(\bDelta)\|_{\blk,1}.
    \]
    Therefore
    \[
        \|\pS-\bDelta\|_{\blk,1} \ge \|\pS\|_{\blk,1} - \langle\bQ/\lambda_S,\bDelta\rangle + (1-b)\|\PIC\bDelta\|_{\blk,1}.
    \]
	Finally, multiplying the regularization parameters and adding two inequalities leads to
	\[
	   \lambda_L\|\pL+\bDelta\|_* + \lambda_S\|\pS-\bDelta\|_{\blk,1} \ge \lambda_L\|\pL\|_* + \lambda_S\|\pS\|_{\blk,1} + \lambda_L(1-a)\|\pPTC\bDelta\|_* + \lambda_S(1-b)\|\PIC\bDelta\|_{\blk,1}.
	\]
	Thus, if $(\pL+\bDelta,\pS-\bDelta)$ is also optimal for~(24), then objective values are equal, so the nonnegative remainder must be zero:
	\[\pPTC(\bDelta)=0,\quad  \PIC(\bDelta)=0.\]
	Proof completed.
\end{proof}

The following lemma transfers the row-block incoherence from $\bU_0$ to $\pU$ on $\mathcal{I}^c$, which plays a role similar to Lemma 10 in \citep{xu2012robust}.
\begin{lemma}[Incoherence transfer]\label{supp:lem:transfer-incoherence}
	Let $(\pL,\pS)$ be an optimal solution to~(23), 
    with SVD $\pL=\pU\pSigma\bV^{\prime\T}$. Let $(\bL_0,\bS_0)$ be the model-induced pair, that is, $\PIC\bL_0=\bL_0$, $\PI\bS_0=\bS_0$ and $\rank(\bL_0)=r$, with SVD $\bL_0=\bU_0\bSigma_0\bV_0^{\T}$. Then $\|\Pg\pU\|_{\F} \le \|\Pg\bU_0\|_{\F}$ for every $g\in\mathcal{I}^c$.
	In particular, if (Block-IC) holds for $\bU_0$, then
	\[
	\max_{g\in \mathcal{I}^c}\|\Pg\pU\|_{\F}^2 \le \frac{\mu r}{G}.
	\]
\end{lemma}
\begin{proof}[Proof of Lemma~\ref{supp:lem:transfer-incoherence}]
	By Lemma~\ref{supp:lem:oracle-feasible-class}, there exists $\bMM\in\mathcal{M}$ such that
	\[
	\pL=\bL_0+\bMM,\qquad \pS=\bS_0-\bMM.
	\]
	Since $\bMM\in\mathcal{M}\subseteq\mathcal{S}_{\mathcal{I}}$, we have $\PIC\bMM=\bzero$. Hence
	\[
	\PIC\pL=\PIC(\bL_0+\bMM)=\PIC\bL_0=\bL_0.
	\]
	Moreover, $\pL=\pL\PA$ and Condition~(12) implies $\rank(\pL)\le \rank(\PA)=r$, while
	\[
	\rank(\pL)\ge \rank(\PIC\pL)=\rank(\bL_0)=r.
	\]
	Therefore $\rank(\pL)=r$.

	Now, write the compact SVDs
	\[
	\pL=\pU\pSigma\pVT,\qquad \bL_0=\bU_0\bSigma_0\bV_0^\T,
	\]
	so we have
	\[
	(\PIC\pU)\pSigma\pVT=\PIC\pL=\bL_0.
	\]
	Since $\rank(\pL)=r$, the matrix $\pSigma\pVT$ has full row rank $r$, so there exists a matrix $\bH$ satisfying $(\pSigma\pVT)\bH=\bI_r$. Consequently,
	\[
	\PIC\pU=(\PIC\pU)(\pSigma\pVT)\bH = (\PIC\pU\pSigma\pVT)\bH,
	\]
	which yields
	\[
	\operatorname{col}(\PIC\pU)=\operatorname{col}(\PIC\pU\pSigma\pVT)=\operatorname{col}(\bL_0)=\operatorname{col}(\bU_0).
	\]
	Thus,
	\[
	\PIC\pU = \bU_0 \bU_0^\T (\PIC\pU)
	\]
	because $\bU_0 \bU_0^\T$ is the orthogonal projector onto $\operatorname{col}(\bU_0)$.

	Now fix $g\in\mathcal{I}^c$. Since $\Pg\PIC=\Pg$,
	\[
	\begin{aligned}
		\Pg\pU &=\Pg(\PIC\pU) \\
		&=\Pg(\bU_0 \bU_0^\T \PIC\pU) \\
		&=\Pg(\bU_0) (\bU_0^\T \PIC\pU)
	\end{aligned}
	\]

	Since
	\[
	(\bU_0^\T \PIC\pU)^\T \bU_0^\T \PIC\pU = \bU^{\prime\T}\PIC\pU=\bI_r-\bU^{\prime\T}\PI\pU\preceq \bI_r,
	\]
	we have $\|\bU_0^\T \PIC\pU\|_{\op}\le 1$.
	Therefore,
	\[
	\|\Pg\pU\|_{\F}\le \|\Pg\bU_0\|_{\F}\|\bU_0^\T \PIC\pU\|_{\op}\le \|\Pg\bU_0\|_{\F}.
	\]
	Taking the maximum over $g\in\mathcal{I}^c$ and using (Block-IC) gives
	\[
	   \max_{g\in \mathcal{I}^c}\|\Pg\pU\|_{\F}^2 \le \frac{\mu r}{G}.
	\]
\end{proof}

\begin{lemma}\label{supp:lem:rho-prime}
    Let $(\pL,\pS)$ be an optimal solution to~(23), 
    with compact SVD $\pL=\pU\pSigma\bV^{\prime\T}$.
    Define $\rho^\prime:=\|\PI\pPU\PI\|_{\op}$.
    Then
    \[
        \rho^\prime\le \frac{\lambda_S}{\lambda_L}\sqrt{s}.
    \]
    In particular, if
    \[
        \frac{\lambda_S}{\lambda_L}\le a\frac{1-\rho^\prime}{\sqrt{s}}
    \]
    for some $a\in(0,1)$, then
    \[
        \rho^\prime\le \frac{a}{1+a}<1.
    \]
\end{lemma}
\begin{proof}[Proof of Lemma~\ref{supp:lem:rho-prime}]
	By Lemma~\ref{supp:lem:PIUVT}, there exists $\bH\in\partial\|\pS\|_{\blk,1}$ such that
	\[
	\lambda_L\PI\pU\pVT=\lambda_S\PI(\bH\pPV).
	\]
	Right-multiplying by $\pV$ and then left-multiplying by $(\pU)^\T$ gives
	\[
	\lambda_L(\pU)^\T\PI\pU=\lambda_S(\pU)^\T\PI\bH\pPV.
	\]
	Taking operator norms and using $\|\pU\|_{\op}=\|\pPV\|_{\op}=1$, we obtain
	\[
	\rho^\prime=\|(\pU)^\T\PI\pU\|_{\op}
	\le \frac{\lambda_S}{\lambda_L}\|(\pU)^\T\PI\bH\pPV\|_{\op}
	\le \frac{\lambda_S}{\lambda_L}\|\PI\bH\|_{\op}
	\le \frac{\lambda_S}{\lambda_L}\|\PI\bH\|_{\F}.
	\]
	Because $\bH\in\partial\|\pS\|_{\blk,1}$, each block of $\PI\bH$ has Frobenius norm at most $1$, so
	\[
	\|\PI\bH\|_{\F}^2\le \sum_{g\in\mathcal{I}}1=s.
	\]
	This proves the first claim.
	If the displayed upper bound on $\lambda_S/\lambda_L$ holds, then
	\[
	\rho^\prime\le \frac{\lambda_S}{\lambda_L}\sqrt{s}\le a(1-\rho^\prime),
	\]
	hence $(1+a)\rho^\prime\le a$, which yields the second claim.
\end{proof}

\begin{lemma}\label{supp:lem:neumann}
	Let $(\pL,\pS)$ be an optimal solution to (23),
    with compact SVD $\pL=\pU\pSigma\bV^{\prime\T}$.
	Let $\rho^\prime:=\|\PI\pPU\PI\|_{\op}$.
	Under $\delta=0$ and Condition~(12), 
    $\rho^\prime<1$, thus the linear operator $(\bI-\PI\pPU\PI)$ is invertible, and
	\[
	\|(\bI-\PI\pPU\PI)^{-1} \|_{\op} \le \frac{1}{1-\rho^\prime}.
	\]
\end{lemma}
\begin{proof}[Proof of Lemma~\ref{supp:lem:neumann}]
	Note that $\bD_0=\pL+\pS=\bL_0+\bS_0$.
	Since $\PI\pS=\pS$ implies $\PIC\pS=0$,  $\PIC\pL=\PIC(\bD_0-\pS)=\PIC\bD_0$.
	Also, $\PIC\bS_0=0$ leads to $\PIC\bD_0=\PIC\bL_0=\bL_0$.
	Thus we have
	\[
	\PIC\pL=\bL_0.
	\]
	Write $A:=\PIC\pU$, then
	\[
	\PIC\pL=A\pSigma\pVT=\bL_0,
	\]
	By assumption and Lemma~\ref{supp:lem:pPV}, $\rank(\pL)=\rank(\bL_0)=r$, hence $\rank(A)=r$, so
	\[
	A^{\T}A=(\pU)^\T\PIC\pU \succ 0.
	\]
	Also, since $\PIC$ is an orthogonal projector,
	\[
	\bzero\prec A^{\T}A\preceq \bI_r
	\]
	Therefore,
	\[
	(\pU)^\T\PI\pU= \bI_r -A^{\T}A
	\]
	has eigenvalues in $[0,1)$, which implies that
	\[
	\rho^\prime=\|\PI\pPU\PI\|_{\op}=\|(\pU)^\T\PI\pU\|_{\op}<1,
	\]
	where the second equation holds because
	\[
	\PI\pPU\PI=(\PI\pU)(\PI\pU)^\T,\quad (\pU)^\T\PI\pU=(\PI\pU)^\T(\PI\pU),
	\]
	and they share the same eigenvalues.

	Since $\rho^\prime<1$, then $\bI-\PI\pPU\PI$ is invertible, and
	\[
	\|(\bI-\PI\pPU\PI)^{-1} \|_{\op} \le \sum_{k\ge 0}\|\PI\pPU\PI\|_{\op}^{k}=\frac{1}{1-\rho^\prime}
	\]
	via the Neumann series $(\bI-\PI\pPU\PI)^{-1}=\sum_{k\ge 0}(\PI\pPU\PI)^{k}$.
\end{proof}

The next proposition ensures the existence of the dual certificate used in the proof of Theorem~1.
\begin{proposition}\label{supp:prop:Q-existence}
	Suppose the noiseless decomposition $\bD_0=\bL_0+\bS_0$ satisfies Assumptions~1--2 of the main paper.
    Let $(\pL,\pS)$ be any optimal solution of the oracle problem~(23), 
    with compact SVD $\pL=\pU\pSigma\pVT$. Let $a,b\in(0,1)$ satisfy $\sqrt{\mu rs/G}<ab/(1+a)^2$.
	If
	\[
	\frac{\lambda_S}{\lambda_L}\in\biggl(
	\frac{\sqrt{\mu r/G}}{b-(1+a)\sqrt{\mu rs/G}},\  \frac{a}{(1+a)\sqrt{s}}
	\biggr),
	\]
	then there exists a matrix $\bQ\in\mathbb{R}^{Gq\times p}$ satisfying \eqref{supp:eq:KKT}.
\end{proposition}
\begin{proof}[Proof of Proposition~\ref{supp:prop:Q-existence}]
	Define $\rho^\prime:=\|\PI\pPU\PI\|_{\op}$ and set
	\[
	x:=\sqrt{\frac{\mu rs}{G}},
	\qquad
	\lambda:=\frac{\lambda_S}{\lambda_L}.
	\]
	By Lemma~\ref{supp:lem:rho-prime} and the assumed upper bound on $\lambda$, we have
	\[
	\rho^\prime\le \lambda\sqrt{s}<\frac{a}{1+a}.
	\]
	Hence,
	\[
	1-\rho^\prime>\frac{1}{1+a}.
	\]

	Under the present noiseless setup, Lemma~\ref{supp:lem:neumann} implies that $(\bI - \PI\pPU\PI)$ is invertible on $\mathcal S_{\mathcal I}$, and further
	\[
	   \|(\bI-\PI\pPU\PI)^{-1}\|_{\op}\le \frac{1}{1-\rho^\prime}<1+a.
	\]

	Now we construct $\bQ$ in the form
	\[
	   \bQ:=\lambda_L\pU\pVT + \pPTC(\bZ),
	\]
	for some $\bZ\in\mathbb{R}^{Gq\times p}$ satisfying $\PI(\bZ)=\bZ$.
	By construction,
	\[
	   \pPT(\bQ)=\lambda_L\pU\pVT
	\]
	is satisfied.
	Moreover,
	\[
	\pPTC\bZ=\pPUC\bZ\pPVC=\pPUC\bZ-\pPUC\bZ\pPV.
	\]
	Hence, once $\bZ\pPV=0$, we have $\pPTC\bZ=\pPUC\bZ$.

	By Lemma~\ref{supp:lem:PIUVT}, there exists $\bH\in\partial\|\pS\|_{\blk,1}$ on $\mathcal{I}$ such that
	\[
	\lambda_L\PI\pU\pVT=\lambda_S\PI(\bH\pPV).
	\]
	Define
	\[
	\bR:=\lambda_S\PI(\bH)-\lambda_L\PI\pU\pVT=\lambda_S\PI(\bH\pPVC),
	\]
	then $\PI(\bR)=\bR$ and $\bR\pPV=\bzero$.
	We therefore set
	\begin{equation}\label{supp:eq:prop-PV-noiseless-Z}
		\bZ:=(\bI-\PI\pPU\PI)^{-1}\bR.
	\end{equation}

	Since the inverse operator acts only on the left, it does not change the nullspace of $\bR$. So $\bR\pPV=\bzero$ implies that $\bZ\pPV=\bzero$.
	Thus indeed $\pPTC\bZ=\pPUC\bZ$.
	Substituting \eqref{supp:eq:prop-PV-noiseless-Z} into
	\[
	\begin{aligned}
		\PI(\bQ)&=\lambda_L\PI\pU\pVT + \PI\pPUC\bZ\\
		&=\lambda_L\PI\pU\pVT + (\bI-\PI\pPU\PI)\bZ
	\end{aligned}
	\]
	yields
	\[
	\PI(\bQ)=\lambda_S\PI(\bH).
	\]
	Because $\bH\in\partial\|\pS\|_{\blk,1}$, we have $\PI(\bH) \in \Sign_{\mathcal{I}}(\pS)$, and hence the third condition in \eqref{supp:eq:KKT} hold.

	Next we verify the bound on $\pPTC(\bQ)$.
	Since $\pPTC(\bQ)=\pPTC(\bZ)$,
	\[
	\begin{aligned}
		\|\pPTC(\bQ)\|_{\op} &\le \|\pPTC(\bZ)\|_{\F}\\
		&\le \|\bZ\|_{\F} \\
		& \le \|(\bI-\PI\pPU\PI)^{-1}\|_{\op}\,\|\bR\|_{\F}.
	\end{aligned}
	\]
	Also, each block of $\PI(\bH)$ has Frobenius norm at most $1$, so
	\[
	\begin{aligned}
		\|\bR\|_{\F} &=\lambda_S\|\PI(\bH\pPVC)\|_{\F}\\
		&\le \lambda_S\|\PI(\bH)\|_{\F}\\
		&\le \lambda_S\sqrt{s},
	\end{aligned}
	\]
	and
	\begin{equation}\label{supp:eq:Z-bound}
		\|\bZ\|_{\F} \le (1+a)\lambda_S\sqrt{s}.
	\end{equation}
	Therefore
	\[
	\|\pPTC(\bQ)\|_{\op} < (1+a)\lambda_S\sqrt{s} < a\lambda_L,
	\]
	where the last inequality is exactly the assumed upper bound on $\lambda$. This proves the second condition in \eqref{supp:eq:KKT}.

	Finally, we bound $\|\PIC(\bQ)\|_{\blk,\infty}$.
	Let
	\[
	\gamma:=\max_{g\in\mathcal{I}^c}\|\Pg\pU\|_{\F}.
	\]
	For any $g\in\mathcal{I}^c$, since $\Pg\bZ=\bzero$,
	\[
	\begin{aligned}
		\Pg\bQ &=\lambda_L\Pg\pU\pVT+\Pg\pPUC\bZ\\
		&=\lambda_L\Pg\pU\pVT-\Pg\pPU\bZ.
	\end{aligned}
	\]
	Hence
	\[
	\begin{aligned}
		\|\Pg\bQ\|_{\F}
		& \le \lambda_L\|\Pg\pU\pVT\|_{\F}+\|\Pg\pPU\|_{\op}\|\bZ\|_{\F}\\
		&\le \gamma\bigl(\lambda_L+\|\bZ\|_{\F}\bigr).
	\end{aligned}
	\]
	Using the bound~\eqref{supp:eq:Z-bound}, we obtain
	\[
	\|\Pg\bQ\|_{\F} < \gamma\lambda_L+\gamma(1+a)\lambda_S\sqrt{s}.
	\]
	Thus it suffices to show
	\[
	\gamma\lambda_L+\gamma(1+a)\lambda_S\sqrt{s}<b\lambda_S,
	\]
	equivalently,
	\[
	\gamma\lambda_L < \lambda_S [b - \gamma(1+a)\sqrt{s}].
	\]
	By Lemma~\ref{supp:lem:transfer-incoherence}, $\gamma\le \sqrt{\mu r/G}=x/\sqrt{s}$.
	Since
	\[
	x<\frac{ab}{(1+a)^2}<\frac{b}{1+a},
	\]
	the term $b-(1+a)x$ is positive.
	Therefore
	\[
	\begin{aligned}
		\frac{\gamma}{b-(1+a)\gamma\sqrt{s}}
		& \le \frac{x/\sqrt{s}}{b-(1+a)x} \\
		&= \frac{\sqrt{\mu r/G}}{b-(1+a)\sqrt{\mu rs/G}}.
	\end{aligned}
	\]
	The assumed lower bound on $\lambda$ now yields
	\[
	\lambda>\frac{\gamma}{b-(1+a)\gamma\sqrt{s}},
	\]
	which is equivalent to
	\[
	\gamma\lambda_L+\gamma(1+a)\lambda_S\sqrt{s}<b\lambda_S.
	\]
	Consequently, $\|\Pg\bQ\|_{\F}<b\lambda_S$ for every $g\in\mathcal{I}^c$, and hence
	\[
	\|\PIC(\bQ)\|_{\blk,\infty}\le b\lambda_S.
	\]
	We have verified all four conditions in \eqref{supp:eq:KKT}, so the proof is complete.
\end{proof}

\subsection{Proof of Theorem~1}
Now we prove Theorem~1 of the main paper.
\begin{proof}[Proof of Theorem~1]
	Let $(\pL,\pS)$ be the optimal solution of the oracle problem~(23). 
	By Proposition~\ref{supp:prop:Q-existence}, there exists a matrix $\bQ \in \mathbb{R}^{Gq\times p}$ satisfying \eqref{supp:eq:KKT} at $(\pL,\pS)$. Lemma~\ref{supp:lem:oracle-to-original} then implies that $(\pL,\pS)$ is also optimal for the original problem~(24). 
	Define $\bDelta:=\hL-\pL$. Since both $(\hL,\hS)$ and $(\pL,\pS)$ are feasible, we have
	\[
	(\hL,\hS) = (\pL+\bDelta,\pS-\bDelta).
	\]
	By Lemma~\ref{supp:lem:subgradient-ineq},
	\[
	\pPTC(\bDelta)=0,\quad \PIC(\bDelta)=0.
	\]

	From $\pPTC(\bDelta)=0$, we know $\bDelta\in \mathcal{T}$, so $\bDelta$ can be written as
	$\bDelta = \pU \bX^\T + \bY\pVT$ for some $\bX$ and $\bY$.
	Now multiply by $\pPVC$:
	\[
	\bDelta \pPVC = \pU \bX^\T \pPVC \in \mathcal{U}^\prime,
	\]
	where $\mathcal{U}^\prime:=\{\pU\bX^\T\mid \bX\in\mathbb{R}^{p\times r}\}$.
	The condition $\PIC(\bDelta)=0$ implies that $\bDelta$ is supported only on blocks in $\mathcal{I}$, so $\bDelta\pPVC\in \mathcal{S}_{\mathcal{I}}$.
	Hence
	\[
	\bDelta\pPVC \in \mathcal{U}^\prime \cap \mathcal{S}_{\mathcal{I}}.
	\]

	Let $\rho^\prime:=\|\PI\pPU\PI\|_{\op}$. By Lemma~\ref{supp:lem:neumann}, $\rho^\prime<1$. Applying the same argument as in Lemma~1 of the main paper with $\PU$ replaced by $\pPU$, we obtain
	\[
	\mathcal{U}^\prime \cap \mathcal{S}_{\mathcal{I}}=\{0\}.
	\]
	Thus
	\[
	\bDelta \pPVC=0,
	\]
	which further leads to $\bDelta=\bDelta \pPV$.
	Therefore,
	\[
	\hL = \pL + \bDelta = (\pL + \bDelta)\pPV = \hL\pPV,
	\]
	so
	\begin{equation}\label{supp:eq:row1}
		\row(\hL)\subseteq \span(\pV)=\row(\pL)=\row(\bA)
	\end{equation}

	On the other hand, using the fact that $\PIC(\pS)=\bzero$ and $\PIC(\bDelta)=\bzero$, we have
	\[
	\PIC\hS=\bzero.
	\]
	By feasibility, we further have
	\[
	\PIC(\hL)=\PIC(\bD_0)=\bL_0.
	\]
	Since $\bL_0=\bL_0\PA$ and Condition~(12) gives $\rank(\bL_0)=r$, we have $\row(\bL_0)=\row(\bA)$.
	Hence
	\begin{equation}\label{supp:eq:row2}
		\row(\bA)=\row(\bL_0)=\row(\PIC\hL)\subseteq \row(\hL).
	\end{equation}

	Combine~\eqref{supp:eq:row1} and ~\eqref{supp:eq:row2}, we have
	\[
	\row(\hL)=\row(\bA).
	\]
	Therefore, for any compact SVD $\hL=\hU\hSigma\hV^\T$,
	\[
	P_{\hV}=P_{\bA}.
	\]
	This holds for every optimizer, thus $P_{\hV}$ is unique.
\end{proof}

\section{Proofs for stable recovery (Section 5.3)}
\subsection{Supporting results}
\begin{lemma}\label{supp:lem:oracle-signal-scale}
	Let the active client set have size $|\mathcal{C}|=\eta K$ for some $\eta\in(0,1]$.
	For $k\in\mathcal{C}$, let $\bB^{(k)}\in\mathbb{R}^{q\times r}$,  $\bA^{\T}\in\mathbb{O}_{p,r}$, and $\bL_0^{(j,k)}=(\bB^{(j)}-\bB^{(k)})\bA$.  Let $\pL$ be an optimizer of the oracle problem.
	Under Assumption~4 of the main paper,
    we have
	\[
	\sigma_r(\pL)=\Omega(K),\quad \rank(\bL_0)=r,\quad \|\bL_0\|_*\le r\sqrt{\kappa_1}\,\eta K=O(K).
	\]
\end{lemma}

\begin{proof}[Proof of Lemma~\ref{supp:lem:oracle-signal-scale}]
	Define the stacked matrix  $\bC_{\mathcal{C}}$ by stacking $\bB^{(j)}-\bB^{(k)}$ over $j<k$ with $j,k\in\mathcal{C}$, and  $\bL_{\mathcal{C}}:=\bC_{\mathcal{C}}\bA$.
	Then $\bL_0$ is obtained from $\bL_{\mathcal{C}}$ by padding with zero block rows for non-benign pairs, thus
	\[
	\bL_0^{\T}\bL_0=\bL_{\mathcal{C}}^{\T}\bL_{\mathcal{C}},
	\qquad
	\|\bL_0\|_{\op}=\|\bL_{\mathcal{C}}\|_{\op},
	\qquad
	\sigma_r(\bL_0)=\sigma_r(\bL_{\mathcal{C}}).
	\]

	Denote $m:=|\mathcal{C}|=\eta K$. We first write
	\begin{equation}\label{supp:eq:pairwise-gram-identity}
		\bC_{\mathcal{C}}^{\T}\bC_{\mathcal{C}}= m\sum_{k\in\mathcal{C}}(\bB^{(k)}-\bar{\bB})^{\T}(\bB^{(k)}-\bar{\bB})
		= m^2\cdot \Bigl(\frac{1}{m}\sum_{k\in\mathcal{C}}(\bB^{(k)}-\bar{\bB})^{\T}(\bB^{(k)}-\bar{\bB})\Bigr).
	\end{equation}

	By Assumption~4 of the main paper,
	\begin{equation}\label{supp:eq:CA-gram-bounds}
		\kappa_0 m^2\,\bI_r \ \preceq\ \bC_{\mathcal{C}}^{\T}\bC_{\mathcal{C}} \ \preceq\ \kappa_1 m^2\,\bI_r.
	\end{equation}
	Since $\bA^{\T}\in \mathbb{O}_{p,r}$, we have $\bA\bA^{\T}=\bI_r$ and hence
	\[
	\bL_{\mathcal{C}}\bL_{\mathcal{C}}^{\T}
	=(\bC_{\mathcal{C}}\bA)(\bC_{\mathcal{C}}\bA)^{\T}
	=\bC_{\mathcal{C}}\bC_{\mathcal{C}}^{\T}.
	\]
	Therefore, $\bL_{\mathcal{C}}$ and $\bC_{\mathcal{C}}$ have the same singular values, and in particular the nonzero eigenvalues of $\bL_{\mathcal{C}}^{\T}\bL_{\mathcal{C}}$ coincide with those of $\bC_{\mathcal{C}}^{\T}\bC_{\mathcal{C}}$.
	Taking eigenvalues in \eqref{supp:eq:CA-gram-bounds} yields
	\begin{equation}\label{supp:eq:LA-gram-eigs}
		\kappa_0\,m^2 \ \le\ \lambda_r(\bL_{\mathcal{C}}^{\T}\bL_{\mathcal{C}})\ \le\ \lambda_1(\bL_{\mathcal{C}}^{\T}\bL_{\mathcal{C}})\ \le\ \kappa_1 m^2.
	\end{equation}
	Also,
	\[
	\|\bL_0\|_{\F}^2=\|\bL_{\mathcal{C}}\|_{\F}^2=\tr(\bL_{\mathcal{C}}^{\T}\bL_{\mathcal{C}})\le r\kappa_1m^2.
	\]
	By construction of the noiseless signal pair (Section~3.2 of the main paper), $\bL_0$ is obtained from $\bL_{\mathcal{C}}$ by padding with zero block rows for outlier pairs, so $\bL_0^{\T}\bL_0=\bL_{\mathcal{C}}^{\T}\bL_{\mathcal{C}}$, hence $\|\bL_0\|_{\op}=\|\bL_{\mathcal{C}}\|_{\op}$ and $\sigma_r(\bL_0)=\sigma_r(\bL_{\mathcal{C}})$.
	Consequently, $\|\bL_{\mathcal{C}}\|_{\op}=\Theta(m)$ and $\sigma_r(\bL_{\mathcal{C}})^2=\Theta(m^2)$.

	Now note that by Lemma~\ref{supp:lem:oracle-feasible-class}, $\pL=\bL_0+\bMM$ for some $\bMM\in\mathcal{M}$.
	Since $\bL_0$ is supported on benign blocks while $\bMM$ is supported on $\mathcal{I}$, we have $\PI(\bL_0)=\bzero$ and $\PIC(\bMM)=\bzero$, hence $\bL_0^{\T}\bMM=\bzero$.
	Therefore
	\[
	(\pL)^\T \pL=\bL_0^{\T}\bL_0+\bMM^{\T}\bMM \succeq \bL_0^{\T}\bL_0.
	\]
	In particular,
	\[
	\sigma_r(\pL)^2\ge \sigma_r(\bL_0)^2=\lambda_r(\bL_0^{\T}\bL_0)\ge \kappa_0 m^2=\kappa_0 \eta^2 K^2
	\]
	by \eqref{supp:eq:LA-gram-eigs}. Moreover,
	\[
	\|\bL_0\|_*\le \sqrt{r}\,\|\bL_0\|_{\F}\le r\sqrt{\kappa_1}\,m=r\sqrt{\kappa_1}\eta K.
	\]
	Taking square roots gives $\sigma_r(\pL)\ge \sigma_r(\bL_0)\ge \sqrt{\kappa_0} \eta K$.
\end{proof}

\begin{lemma}\label{supp:lem:E0-bound}
    Suppose Assumptions~1 and~3 of the main paper hold, 
    and let $m:=|\mathcal{C}|=\eta K$.
    Then
    \begin{equation}\label{supp:eq:E0-Fbound}
        \|\bE_0\|_{\F}^2 \le 2m(m-1)B^2\delta^2 \le 2K^2B^2\delta^2,
    \end{equation}
    and therefore
    \begin{equation}\label{supp:eq:E0-proj-bound}
        \|\pPU\bE_0\pPVC\|_{\F}\le \sqrt{2}\,KB\delta,
        \qquad
        \|\PI\bE_0\pPVC\|_{\F}=0.
    \end{equation}
\end{lemma}

\begin{proof}[Proof of Lemma~\ref{supp:lem:E0-bound}]
	By the blockwise definition of $\bE_0$, if at least one of $j,k$ lies in $\mathcal{C}^c$, then $\bE_0^{(j,k)}=\bzero$.
	Hence $\bE_0$ is supported on $\mathcal{I}^c$, so $\PI(\bE_0)=\bzero$.

	If $j,k\in\mathcal{C}$, then by Assumptions~1 and~3 of the main paper, and the triangle inequality,
	\[
	\begin{aligned}
		\|\bE_0^{(j,k)}\|_{\F}
		&\le \|\bB^{(j)}(\bA^{(j)}-\bA)\|_{\F}+\|\bB^{(k)}(\bA^{(k)}-\bA)\|_{\F}\\
		&\le \|\bB^{(j)}\|_{\F}\|\bA^{(j)}-\bA\|_{\F}+\|\bB^{(k)}\|_{\F}\|\bA^{(k)}-\bA\|_{\F}\\
		&\le 2B\delta.
	\end{aligned}
	\]
	This also covers the case $(j,k)\notin\mathcal{C}\times\mathcal{C}$ since then $\bE_0^{(j,k)}=\bzero$.

	Because exactly $\binom{m}{2}$ active-active pairs may contribute,
	\[
	\|\bE_0\|_{\F}^2
	=
	\sum_{(j,k)\in\mathcal{G}}\|\bE_0^{(j,k)}\|_{\F}^2
	\le
	\binom{m}{2}(2B\delta)^2
	=
	2m(m-1)B^2\delta^2,
	\]
	which proves \eqref{supp:eq:E0-Fbound}.
	Since $\|\pPU\|_{\op}=\|\pPVC\|_{\op}=1$, we have
	\[
	\|\pPU\bE_0\pPVC\|_{\F}\le \|\bE_0\|_{\F}\le \sqrt{2}\,KB\delta,
	\]
	while
	\[
	\|\PI\bE_0\pPVC\|_{\F}\le \|\PI\bE_0\|_{\F}=0.
	\]
	This proves \eqref{supp:eq:E0-proj-bound}.
\end{proof}

\begin{lemma}\label{supp:lem:E1-identity}
    Suppose $\bE_1$ admits the blockwise representation $\bE_1^{(j,k)} = \bXi_j -\bXi_k$ for $(j,k)\in\mathcal{G}$, where $\mathcal{G}:=\{(j,k):1\le j<k\le K\}$, $G:=|\mathcal{G}|$, and $\bXi_k\in\mathbb R^{q\times p}$.
    Then,
    \[
    \|\bE_1\|_{\F}^2 =K\sum_{k=1}^K \|\bXi_k\|_{\F}^2 -\Bigl\|\sum_{k=1}^K \bXi_k\Bigr\|_{\F}^2.
    \]
    In particular,
    $\|\bE_1\|_{\F}^2 \le K\sum_{k=1}^K \|\bXi_k\|_{\F}^2.$
\end{lemma}

\begin{proof}[Proof of Lemma~\ref{supp:lem:E1-identity}]
	By the definition of the stacked pairwise error matrix,
	\[
	\|\bE_1\|_{\F}^2 = \sum_{(j,k)\in\mathcal G}\|\bE_1^{(j,k)}\|_{\F}^2 =\sum_{j<k}\|\bXi_j-\bXi_k\|_{\F}^2.
	\]
	Expanding the square gives
	\[
	\sum_{j<k}\|\bXi_j-\bXi_k\|_{\F}^2 = \sum_{j<k}\Bigl(\|\bXi_j\|_{\F}^2+\|\bXi_k\|_{\F}^2-2\langle \bXi_j,\bXi_k\rangle\Bigr).
	\]

	For the diagonal part, we have
	\[
	\begin{aligned}
		\sum_{j<k}\bigl(\|\bXi_j\|_{\F}^2+\|\bXi_k\|_{\F}^2\bigr)
		&= \sum_{j=1}^K\sum_{k=1}^K\bigl(\|\bXi_j\|_{\F}^2 + \|\bXi_k\|_{\F}^2\bigr)\idop\{j<k\}\\
		&= \sum_{j=1}^K\|\bXi_j\|_{\F}^2 \cdot\Bigl(\sum_{k=1}^K\idop\{j<k\}\Bigr) + \sum_{k=1}^K\|\bXi_k\|_{\F}^2\cdot\Bigl(\sum_{j=1}^K\idop\{j<k\}\Bigr)\\
		&=\sum_{j=1}^K (K-j)\|\bXi_j\|_{\F}^2 +\sum_{k=1}^K (k-1)\|\bXi_k\|_{\F}^2 \\
		&=(K-1)\sum_{k=1}^K\|\bXi_k\|_{\F}^2.
	\end{aligned}
	\]

	For the cross term, use
	\[
	\Bigl\|\sum_{k=1}^K \bXi_k\Bigr\|_{\F}^2=\sum_{k=1}^K \|\bXi_k\|_{\F}^2 + 2\sum_{j<k}\langle \bXi_j,\bXi_k\rangle.
	\]
	Equivalently,
	\[
	2\sum_{j<k}\langle \bXi_j,\bXi_k\rangle =    \Bigl\|\sum_{k=1}^K \bXi_k\Bigr\|_{\F}^2 -\sum_{k=1}^K \|\bXi_k\|_{\F}^2.
	\]

	Substituting these two identities into the expansion yields
	\[
	\sum_{j<k}\|\bXi_j-\bXi_k\|_{\F}^2= (K-1)\sum_{k=1}^K \|\bXi_k\|_{\F}^2 - \Bigl(\Bigl\|\sum_{k=1}^K \bXi_k\Bigr\|_{\F}^2 - \sum_{k=1}^K \|\bXi_k\|_{\F}^2\Bigr),
	\]
	that is,
	\[
	\|\bE_1\|_{\F}^2 = K\sum_{k=1}^K \|\bXi_k\|_{\F}^2
	- \Bigl\|\sum_{k=1}^K \bXi_k\Bigr\|_{\F}^2.
	\]
	The inequality follows by dropping the nonnegative term
	$\bigl\|\sum_{k=1}^K \bXi_k\bigr\|_{\F}^2$.
\end{proof}

\begin{lemma}\label{supp:lem:E1-Fnorm}
    Under the representation in Lemma~\ref{supp:lem:E1-identity} and Assumption 5 of the main paper, there exist absolute constants $c, C>0$ such that, for every $k\in[K]$,
    \[
    \E\|\bXi_k\|_{\F}^2\le c\,pq\,\tau_k^2,
    \]
    and, for any $t>0$,
    \[
    \|\bE_1\|_{\F}^2
    \le
    C\,pq\,\tau^2\bigl(K^2+K^{3/2}\sqrt t+Kt\bigr)
    \]
    with probability at least $1-e^{-t}$.
\end{lemma}
\begin{proof}[Proof of Lemma~\ref{supp:lem:E1-Fnorm}]
	By Lemma~\ref{supp:lem:E1-identity}, we have
	\[
	\|\bE_1\|_{\F}^2 \le K\sum_{k=1}^K \|\bXi_k\|_{\F}^2.
	\]
	Now define
	\[
	Z_k:=\|\bXi_k\|_{\F}^2 - \E\|\bXi_k\|_{\F}^2,
	\]
	then $Z_k$'s are independent zero mean random variables.
	It suffices to seek a tight upper bound fort
	\[
	\sum_{k=1}^K \|\bXi_k\|_{\F}^2 =\sum_{k=1}^K \E\|\bXi_k\|_{\F}^2 + \sum_{k=1}^K Z_k.
	\]

	We first bound the expectation term $\E\|\bXi_k\|_{\F}^2$.
	Let
	\[\bxi_k:=\vec(\bXi_k)\in\mathbb R^{pq}\]
	and define
	\[
	\bSigma_k:=\cov(\bxi_k)=\E(\bxi_k\bxi_k^\T),
	\]
	where the second equality uses $\E\bxi_k=\bzero$.
	By Assumption~5, for every $\bu\in\mathbb R^{pq}$ with $\|\bu\|_2=1$,
	\[
	\|\langle \bu,\bxi_k\rangle\|_{\psi_2}\le \tau_k.
	\]
	Using the standard sub-Gaussian moment bound, we obtain
	\[
	\bu^\T\bSigma_k\bu = \E\langle \bu,\bxi_k\rangle^2 \le c\tau_k^2
	\]
	for every unit $\bu$ and some constant $c>0$, which implies
	\[
	\bSigma_k\preceq c\tau_k^2\bI_{pq}.
	\]
	Therefore,
	\[
	\E\|\bXi_k\|_{\F}^2 = \E\|\bxi_k\|_2^2 = \tr(\bSigma_k) \le cpq\,\tau_k^2.
	\]

	To bound the stochastic deviation term, note that Assumption~5 with a standard $\varepsilon$-net argument for sub-Gaussian vectors  yields
	\[
	\bigl\|\|\bxi_k\|_2\bigr\|_{\psi_2}\le c_2 \sqrt{pq} \tau_k
	\]
	for constant $c_2>0$.
	Therefore, using the standard bound $\|X^2\|_{\psi_1}\lesssim \|X\|_{\psi_2}^2$ for scalar $X$,
	\[
	\bigl\|\|\bxi_k\|_2^2\bigr\|_{\psi_1}\le c_3 pq\tau_k^2
	\]
	for some $c_3>0$.
	Since $\|\bxi_k\|_2=\|\bXi_k\|_{\F}$, it follows that
	\[
	\begin{aligned}
		\|Z_k\|_{\psi_1} &= \bigl\|\|\bXi_k\|_F^2-\E\|\bXi_k\|_F^2\bigr\|_{\psi_1}\\
		&\le \bigl\|\|\bxi_k\|_2^2\bigr\|_{\psi_1} + \|\E\|\bXi_k\|_F^2\|_{\psi_1} \\
		&\le \bigl\|\|\bxi_k\|_2^2\bigr\|_{\psi_1} + c_4\bigl\|\|\bxi_k\|_2^2\bigr\|_{\psi_1} \\
		&\le c_5 pq\tau_k^2,
	\end{aligned}
	\]
	where $c_4,c_5>0$ are constants.
	Thus, $Z_k$ are sub-Exponential.
	Employing Bernstein inequality leads to
	\[
	\sum_{k=1}^K Z_k \le c_6 \Biggl(pq\sqrt{t\sum_{k=1}^K\tau_k^4} + pqt\max_{1\le k\le K}\tau_k^2\Biggr),
	\]
	with probability at least $1-e^{-t}$.

	Combining the expectation term and stochastic term yields
	\[
	\begin{aligned}
		\sum_{k=1}^K \|\bXi_k\|_{\F}^2
		&\le \sum_{k=1}^K c pq\tau_k^2 + c_6 \Biggl(pq\sqrt{t\sum_{k=1}^K\tau_k^4} + pqt\max_{1\le k\le K}\tau_k^2\Biggr) \\
		&\le Cpq\tau^2(K + \sqrt{tK} + t)
	\end{aligned}
	\]
	for some constant $C>0$, where the last inequality employs the uniform bound $\tau_k \le\tau$.
	Therefore,
	\[
	\|\bE_1\|_{\F}^2 \le K\sum_{k=1}^K \|\bXi_k\|_{\F}^2 \le Cpq\tau^2(K^2 + \sqrt{t}K^{3/2} + tK)
	\]
	with probability at least $1-e^{-t}$.
\end{proof}

\begin{lemma}\label{supp:lem:E1-proj-PI}
    Let $(\pL,\pS)$ be an optimal solution of (23),
    with compact SVD $\pL=\pU\pSigma\pVT$.
    Under the representation in Lemma~\ref{supp:lem:E1-identity} and Assumption~5 of the main paper, there exists an absolute constant $C>0$ such that, for every $t\ge 0$, with probability at least $1-2e^{-t}$,
    \begin{equation}\label{supp:eq:E1-proj-PI-square}
        \|\PI\bE_1\pPVC\|_{\F}^2
        \le C\,q(p-r)\tau^2\bigl(K^2+K^{3/2}\sqrt{t}+Kt\bigr).
    \end{equation}
    In particular, after enlarging $C$ if necessary,
    \begin{equation}\label{supp:eq:E1-proj-PI-root}
        \|\PI\bE_1\pPVC\|_{\F}
        \le C\tau\sqrt{q(p-r)\,K(K+t)}.
    \end{equation}
\end{lemma}

\begin{proof}[Proof of Lemma~\ref{supp:lem:E1-proj-PI}]
	If $p=r$, then $\pPVC=\bzero$, so the claim is trivial. Assume $p>r$.
	Let $\bV_\perp\in\mathbb R^{p\times (p-r)}$ have orthonormal columns spanning $\range(\pPVC)$, and define
	\[
	\widetilde{\bXi}_k:=\bXi_k\bV_\perp\in\mathbb R^{q\times (p-r)},\qquad
	\widetilde{\bE}_1^{(j,k)}:=\widetilde{\bXi}_j-\widetilde{\bXi}_k.
	\]
	Then
	\[
	\widetilde{\bE}_1^{(j,k)}=\bE_1^{(j,k)}\bV_\perp,\qquad 1\le j<k\le K.
	\]
	Hence
	\[
	\|\PI\bE_1\pPVC\|_F^2
	=
	\sum_{(j,k)\in\mathcal I}\|\widetilde{\bXi}_j-\widetilde{\bXi}_k\|_F^2
	\le
	\sum_{1\le j<k\le K}\|\widetilde{\bXi}_j-\widetilde{\bXi}_k\|_F^2
	=
	\|\widetilde{\bE}_1\|_F^2.
	\]
	For any $\bB\in\mathbb R^{q\times (p-r)}$ with $\|\bB\|_F=1$,
	\[
	\langle \bB,\widetilde{\bXi}_k\rangle
	=
	\langle \bB\bV_\perp^\top,\bXi_k\rangle,
	\qquad
	\|\bB\bV_\perp^\top\|_F=\|\bB\|_F=1.
	\]
	Therefore Assumption~5 holds for $\widetilde{\bXi}_k$ with the same parameter $\tau_k$.
	Applying Lemma~\ref{supp:lem:E1-Fnorm} in ambient dimension $q(p-r)$ yields
	\[
	\|\widetilde{\bE}_1\|_F^2
	\le
	C\,q(p-r)\tau^2\bigl(K^2+K^{3/2}\sqrt t+Kt\bigr)
	\]
	with probability at least $1-e^{-t}$.
	The first claim follows. The second follows by taking square roots and using
	\[
	K^2+K^{3/2}\sqrt t+Kt \le \tfrac32 K(K+t).
	\]
\end{proof}

\begin{lemma}\label{supp:lem:E1-linear}
    Under the representation in Lemma~\ref{supp:lem:E1-identity} and Assumption~5 of the main paper, there exists an absolute constant $C>0$ such that for any fixed $\bA\in\mathbb R^{Gq\times p}$,
    \[
    \bigl\|\langle \bA,\bE_1\rangle\bigr\|_{\psi_2}
    \le C\,\tau\,\sqrt{K}\,\|\bA\|_{\F}.
    \]
\end{lemma}
\begin{proof}[Proof of Lemma~\ref{supp:lem:E1-linear}]
	As in Lemma~\ref{supp:lem:E1-Fnorm}, we leverage the client-wise decomposition of $\bE_1$ through the independent mean-zero matrices $\{\bXi_k\}_{k=1}^K$.

	By linearity of the Frobenius inner product,
	\[
	\langle \bA,\bE_1\rangle = \sum_{(j,k)\in\mathcal G}\langle \bA^{(j,k)},\bXi_j-\bXi_k\rangle.
	\]
	Rewrite the sum by collecting the terms multiplying the same $\bXi_k$. Then, we obtain
	\[
	\langle \bA,\bE_1\rangle = \sum_{k=1}^K \langle \bZ_k(\bA),\bXi_k\rangle,
	\]
	where
	\[
	\bZ_k(\bA):=\sum_{k<j}\bA^{(k,j)}-\sum_{j<k}\bA^{(j,k)}\in\mathbb{R}^{q\times p}.
	\]
	For each fixed $k$, define
	\[
	\mathcal G_k:=\{g=(a,b)\in\mathcal G:\ a=k \text{ or } b=k\},
	\]
	the set of blocks incident to client $k$. Then $|\mathcal G_k|=K-1$, and
	$\bZ_k(\bA)$ is a signed sum of the matrices $\{\bA^{(g)}:g\in\mathcal G_k\}$.
	Hence, by Cauchy-Schwarz inequality,
	\[
	\|\bZ_k(\bA)\|_{\F}^2 \le (K-1)\sum_{g\in\mathcal G_k}\|\bA^{(g)}\|_{\F}^2.
	\]
	Note that each block $g=(a,b)\in\mathcal G$ belongs to exactly two sets, namely $\mathcal G_a$ and $\mathcal G_b$, so summing over $k$ yields
	\begin{equation}\label{supp:eq:E1-linear-Z-bound}
		\begin{aligned}
			\sum_{k=1}^K \|\bZ_k(\bA)\|_{\F}^2
			&\le (K-1)\sum_{k=1}^K\sum_{g\in\mathcal G_k}\|\bA^{(g)}\|_{\F}^2 \\
			&=2(K-1)\sum_{g\in\mathcal G}\|\bA^{(g)}\|_{\F}^2 \\
			&= 2(K-1)\|\bA\|_{\F}^2.
		\end{aligned}
	\end{equation}

	Now fix $k\in[K]$. If $\bZ_k(\bA)=\bzero$, then
	\[
	\langle \bZ_k(\bA),\bXi_k\rangle=0.
	\]
	Otherwise, define
	\[
	\bB_k:=\frac{\bZ_k(\bA)}{\|\bZ_k(\bA)\|_{\F}},
	\]
	such that $\|\bB_k\|_{\F}=1$.
	By Assumption~5,
	\[
	\begin{aligned}
		\bigl\|\langle \bZ_k(\bA),\bXi_k\rangle\bigr\|_{\psi_2}
		&=\|\bZ_k(\bA)\|_{\F}\cdot\bigl\|\langle \bB_k,\bXi_k\rangle\bigr\|_{\psi_2} \\
		&\le \tau_k\,\|\bZ_k(\bA)\|_{\F} \\
		&\le \tau\,\|\bZ_k(\bA)\|_{\F}.
	\end{aligned}
	\]

	Since the matrices $\{\bXi_k\}_{k=1}^K$ are independent and mean-zero, the random variables
	\[
	Y_k:=\langle \bZ_k(\bA),\bXi_k\rangle,\qquad k\in[K],
	\]
	are independent mean-zero sub-Gaussian random variables. A standard sub-Gaussian summation
	bound therefore yields
	\[
	\Bigl\|\sum_{k=1}^K Y_k\Bigr\|_{\psi_2}
	\le
	C_1\Bigl(\sum_{k=1}^K \|Y_k\|_{\psi_2}^2\Bigr)^{1/2}
	\le
	C_1\tau
	\Bigl(\sum_{k=1}^K \|\bZ_k(\bA)\|_{\F}^2\Bigr)^{1/2}.
	\]
	Combining this with \eqref{supp:eq:E1-linear-Z-bound}, we obtain
	\[
	\bigl\|\langle \bA,\bE_1\rangle\bigr\|_{\psi_2}
	\le
	C_2\tau\,\sqrt{2(K-1)}\,\|\bA\|_{\F}
	\le
	C\,\tau\,\sqrt{K}\,\|\bA\|_{\F}.
	\]
\end{proof}

\begin{lemma}\label{supp:lem:S0-blk1}
    Under decomposition~(15) with  $s:=|\mathcal{I}|$, we have
    \[
    \|\bS_0\|_{\blk,1} \le (K-1)\sum_{k\in\mathcal{C}^c}\|\bB_k\bA_k\|_{\F} + |\mathcal{C}^c|\sum_{j\in\mathcal{C}}\|\bB_j\bA_j\|_{\F}.
    \]

    In particular, suppose Assumption~1 and 3 of main paper hold, i.e,
    \[
    \max_{k\in[K]}\|\bB^{(k)}\|_{\F}\le B
    \quad\text{and}\quad
    \bA^{(k),\T}\in\mathbb{O}_{p,r}\ \text{for all }k\in[K],
    \]
    then
    \begin{equation}\label{supp:eq:S0-blk1-bound}
        \|\bS_0\|_{\blk,1}\le 2sB.
    \end{equation}
\end{lemma}

\begin{proof}[Proof of Lemma~\ref{supp:lem:S0-blk1}]
	By the block definition of $\bS_0$, for every $(j,k)\in\mathcal{I}$,
	\[
	\|\bS_0^{(j,k)}\|_{\F} = \|\bB_j\bA_j - \bB_k\bA_k\|_{\F} \le
	\|\bB_j\bA_j\|_{\F}+\|\bB_k\bA_k\|_{\F}.
	\]
	Summing over $(j,k)\in\mathcal{I}$ yields
	\[
	\|\bS_0\|_{\blk,1} \le \sum_{(j,k)\in\mathcal{I}}\bigl(\|\bB_j\bA_j\|_{\F}+\|\bB_k\bA_k\|_{\F}\bigr).
	\]
	Each contaminated client $k\in\mathcal{C}^c$ appears in exactly $K-1$ pairs in $\mathcal{I}$, whereas each benign client $j\in\mathcal{C}$ appears in exactly $|\mathcal{C}^c|$ pairs in $\mathcal{I}$.
	Collecting coefficients of the same $\|\bX_k\|_{\F}$ yields
	\[
	\|\bS_0\|_{\blk,1} \le (K-1)\sum_{k\in\mathcal{C}^c}\|\bB_k\bA_k\|_{\F} + |\mathcal{C}^c|\sum_{j\in\mathcal{C}}\|\bB_j\bA_j\|_{\F}.
	\]
	Given $\max_k\|\bB^{(k)}\|_{\F}\le B$ and $\bA^{(k),\T}\in\mathbb{O}_{p,r}$ for every $k$, the uniform bound \eqref{supp:eq:S0-blk1-bound} is immediate from
	\[
	\|\bS_0^{(j,k)}\|_{\F}\le 2B,\qquad (j,k)\in\mathcal{I},
	\]
	and the fact that $|\mathcal{I}|=s$.
\end{proof}

\begin{lemma}\label{supp:lem:subgaussian-vector-norm}
	Let $z\in\mathbb{R}^m$ be a mean-zero random vector such that
	\[
	\sup_{u\in\mathbb{S}^{m-1}}\|\langle u,z\rangle\|_{\psi_2}\le L
	\]
	for some constant $L>0$. Then there exists an absolute constant $C>0$ such that for every $t\ge 0$,
	\[
	\pr\bigl(\|z\|_2>C\,L(\sqrt{m}+\sqrt{t})\bigr)\le 2e^{-t}.
	\]
\end{lemma}
\begin{proof}[Proof of Lemma~\ref{supp:lem:subgaussian-vector-norm}]
	Let $\mathcal{N}$ be a $1/2$-net of $\mathbb{S}^{m-1}$ with $|\mathcal{N}|\le 5^m$. By the standard net argument \citep{vershynin2018high},
	\[
	\|z\|_2\le 2\max_{u\in\mathcal{N}}|\langle u,z\rangle|.
	\]
	Fix $u\in\mathcal{N}$. Since $\langle u,z\rangle$ is mean-zero and
	\[
	\|\langle u,z\rangle\|_{\psi_2}\le L,
	\]
	the standard tail bound for sub-Gaussian random variables implies
	\[
	\pr \bigl(|\langle u,z\rangle|>C_1L(\sqrt{m}+\sqrt{t})\bigr)
	\le 2\exp \bigl(-(1+\log 5)(m+t)\bigr)
	\]
	for a sufficiently large absolute constant $C_1>0$.
	A union bound over $\mathcal{N}$ then gives
	\[
	\pr \Bigl(\max_{u\in\mathcal{N}}|\langle u,z\rangle|>C_1L(\sqrt{m}+\sqrt{t})\Bigr)
	\le 2e^{-t}.
	\]
	Combining the last two displays yields
	\[
	\pr \Bigl(\|z\|_2>2C_1L(\sqrt{m}+\sqrt{t})\Bigr)\le 2e^{-t}.
	\]
	Renaming the constant completes the proof.
\end{proof}

	\begin{lemma}\label{supp:lem:leakage-ineq-pen}
		Let $(\hL,\hS)$ be any optimal solution of the penalized program~(7),
        and define
		$\Pomega:=\sum_{g\in\mathcal{G}}\omega_g\Pg$, $\hR:=\hD-\hL-\hS$.
		Let $(\pL,\pS)$ be any optimal solution of the noiseless oracle problem~(23),
        and define
		\[
		\bDelta_L:=\hL-\pL,\qquad \bDelta_S:=\hS-\pS,\qquad \bDelta:=\bDelta_L+\bDelta_S,\qquad \bE:=\hD-\pL-\pS.
		\]
		Assume there exists a matrix $\bQ\in\mathbb{R}^{Gq\times p}$ and constants $a,b\in(0,1)$ such that~\eqref{supp:eq:KKT} holds,
		then
		\begin{equation}\label{supp:eq:leakage-penalized}
			(1-a)\lambda_L\|\pPTC\bDelta_L\|_*+(1-b)\lambda_S\|\PIC\bDelta_S\|_{\blk,1}  \le \mathcal{R}_E,
		\end{equation}
		where
		\[
		\mathcal{R}_E := \frac{1}{2}\langle \bE,\Pomega\bE\rangle - \langle \bQ,\bE\rangle 	+ \frac{1}{2}\|\Pomega^{-1/2}\bQ\|_{\F}^2 = \frac{1}{2}\|\Pomega^{1/2}\bE-\Pomega^{-1/2}\bQ\|_{\F}^2.
		\]
	\end{lemma}

\begin{proof}[Proof of Lemma~\ref{supp:lem:leakage-ineq-pen}]
	By optimality of $(\hL,\hS)$ for (7) and feasibility of $(\pL,\pS)$,
	\begin{equation}\label{supp:eq:pen-basic-ineq}
		\frac{1}{2}\langle \hR,\Pomega\hR\rangle + \lambda_L\|\hL\|_*+\lambda_S\|\hS\|_{\blk,1}
		\le \frac{1}{2}\langle \bE,\Pomega\bE\rangle + \lambda_L\|\pL\|_*+\lambda_S\|\pS\|_{\blk,1}.
	\end{equation}
	Since \eqref{supp:eq:KKT} implies $\bQ/\lambda_L\in\partial\|\pL\|_*$ and $\bQ/\lambda_S\in\partial\|\pS\|_{\blk,1}$, the same argument as in Lemma~\ref{supp:lem:subgradient-ineq} gives
	\begin{equation}\label{supp:eq:pen-subg-ineq}
		\begin{aligned}
			\lambda_L\|\hL\|_*+\lambda_S\|\hS\|_{\blk,1}
			&\ge \lambda_L\|\pL\|_*+\lambda_S\|\pS\|_{\blk,1}
			+\langle \bQ,\bDelta\rangle \\
			&\quad +(1-a)\lambda_L\|\pPTC\bDelta_L\|_*
			+(1-b)\lambda_S\|\PIC\bDelta_S\|_{\blk,1}.
		\end{aligned}
	\end{equation}
	Combining \eqref{supp:eq:pen-basic-ineq} and \eqref{supp:eq:pen-subg-ineq} yields
	\[
	\begin{aligned}
		\frac{1}{2}\langle \hR,\Pomega\hR\rangle + \langle \bQ,\bDelta\rangle
		&+(1-a)\lambda_L\|\pPTC\bDelta_L\|_* \\
		&+(1-b)\lambda_S\|\PIC\bDelta_S\|_{\blk,1}
		\le \frac{1}{2}\langle \bE,\Pomega\bE\rangle.
	\end{aligned}
	\]
	Since $\bDelta=\bE-\hR$, we have
	\[
	\langle \bQ,\bDelta\rangle = \langle \bQ,\bE\rangle-\langle \bQ,\hR\rangle.
	\]
	Therefore,
	\[
	\begin{aligned}
		&(1-a)\lambda_L\|\pPTC\bDelta_L\|_*+(1-b)\lambda_S\|\PIC\bDelta_S\|_{\blk,1}\\
		&\le \frac{1}{2}\langle \bE,\Pomega\bE\rangle - \frac{1}{2}\langle \hR,\Pomega\hR\rangle
		-\langle \bQ,\bE\rangle + \langle \bQ,\hR\rangle.
	\end{aligned}
	\]
	Since the row-block projectors are mutually orthogonal and satisfy $\sum_{g\in\mathcal{G}}\Pg=\bI$ on the stacked block space,
	the operator $\Pomega=\sum_{g\in\mathcal{G}}\omega_g\Pg$ acts blockwise as multiplication by $\omega_g$.
	Because each $\omega_g>0$, $\Pomega$ is positive definite and $\Pomega^{-1/2}$ is well-defined. Hence
	\[
	\langle \bQ,\hR\rangle
	=
	\langle \Pomega^{-1/2}\bQ,\Pomega^{1/2}\hR\rangle
	\le \frac{1}{2}\|\Pomega^{-1/2}\bQ\|_{\F}^2+\frac{1}{2}\langle \hR,\Pomega\hR\rangle.
	\]
	Substituting this bound yields
	\[
	(1-a)\lambda_L\|\pPTC\bDelta_L\|_*+(1-b)\lambda_S\|\PIC\bDelta_S\|_{\blk,1}
	\le
	\frac{1}{2}\langle \bE,\Pomega\bE\rangle - \langle \bQ,\bE\rangle
	+ \frac{1}{2}\|\Pomega^{-1/2}\bQ\|_{\F}^2.
	\]
	Finally,
	\[
	\|\Pomega^{1/2}\bE-\Pomega^{-1/2}\bQ\|_{\F}^2
	=
	\langle \bE,\Pomega\bE\rangle - 2\langle \bQ,\bE\rangle
	+ \|\Pomega^{-1/2}\bQ\|_{\F}^2,
	\]
	which complements the proof.
\end{proof}

	\begin{lemma}\label{supp:lem:RE-bound}
		Let $(\pL,\pS)$ be any optimal solution of problem~(23). Let $\bQ\in\mathbb{R}^{Gq\times p}$ and $a,b\in(0,1)$ satisfy \eqref{supp:eq:KKT} at $(\pL,\pS)$.
		Under the representation in Lemma~\ref{supp:lem:E1-identity} and Assumption~5 of the main paper, there exists an absolute constant $C>0$ such that, for every $t>0$, with probability at least $1-2e^{-t}$,
		\begin{equation}\label{supp:eq:Rtilde-pen-bound}
			\mathcal{R}_E
			\le
			\mathcal{R}_{0}
			+ C\tau\sqrt{K}\,\|\Pomega\bE_0-\bQ\|_{\F}\sqrt{t}
			+ C\omega_{\max}pq\tau^2\bigl(K^2+K^{3/2}\sqrt{t}+Kt\bigr),
		\end{equation}
		where
		\[
		\mathcal{R}_{0} := \frac{1}{2}\|\Pomega^{1/2}\bE_0-\Pomega^{-1/2}\bQ\|_{\F}^2
		\]
		and $\omega_{\max}:=\max_{g\in\mathcal{G}}\omega_g$.
		In addition, if Assumption~1 and~3 hold, then, on the same event,
		\begin{equation}\label{supp:eq:Rtilde-pen-bound-struct}
			\mathcal{R}_E \le 4\omega_{\max}K^2B^2\delta^2 	+2\lambda_S^2\Biggl(\sum_{g\in\mathcal{I}}\omega_g^{-1} +b^2\sum_{g\in\mathcal{I}^c}\omega_g^{-1} \Biggr) + C\omega_{\max}pq\tau^2K(K+t).
		\end{equation}
	\end{lemma}

\begin{proof}[Proof of Lemma~\ref{supp:lem:RE-bound}]
	By the definition of $\mathcal{R}_E$ and the decomposition $\bE=\bE_0+\bE_1$,
	\begin{equation}\label{supp:eq:RE-decomp}
		\begin{aligned}
			\mathcal{R}_E
			&=
			\frac{1}{2}\langle \bE_0+\bE_1,\Pomega(\bE_0+\bE_1)\rangle
			-\langle \bQ,\bE_0+\bE_1\rangle
			+\frac{1}{2}\|\Pomega^{-1/2}\bQ\|_{\F}^2 \\
			&=
			\mathcal{R}_{0}
			+\langle \Pomega\bE_0-\bQ,\bE_1\rangle
			+\frac{1}{2}\langle \bE_1,\Pomega\bE_1\rangle.
		\end{aligned}
	\end{equation}

	By Lemma~\ref{supp:lem:E1-linear} applied with the fixed matrix $\Pomega\bE_0-\bQ$,
	\[
	\|\langle \Pomega\bE_0-\bQ,\bE_1\rangle\|_{\psi_2}
	\le C_1\tau\sqrt{K}\,\|\Pomega\bE_0-\bQ\|_{\F}
	\]
	for some absolute constant $C_1>0$.
	Hence the standard one-sided sub-Gaussian tail bound yields
	\begin{equation}\label{supp:eq:Rtilde-pen-linear}
		\langle \Pomega\bE_0-\bQ,\bE_1\rangle
		\le C_2\tau\sqrt{K}\,\|\Pomega\bE_0-\bQ\|_{\F}\sqrt{t}
	\end{equation}
	with probability at least $1-e^{-t}$, for a sufficiently large absolute constant $C_2>0$.

	Next,
	\[
	\begin{aligned}
		\langle \bE_1,\Pomega\bE_1\rangle 	&= 	\sum_{g\in\mathcal{G}}\omega_g\|\Pg(\bE_1)\|_{\F}^2 \\
		&\le 	\omega_{\max}\sum_{g\in\mathcal{G}}\|\Pg(\bE_1)\|_{\F}^2 \\
		& = 	\omega_{\max}\|\bE_1\|_{\F}^2.
	\end{aligned}
	\]
	By Lemma~\ref{supp:lem:E1-Fnorm}, there exists an absolute constant $C_3>0$ such that
	\begin{equation}\label{supp:eq:Rtilde-pen-quad}
		\frac{1}{2}\langle \bE_1,\Pomega\bE_1\rangle
		\le
		C_3\omega_{\max}pq\tau^2\bigl(K^2+K^{3/2}\sqrt{t}+Kt\bigr)
	\end{equation}
	with probability at least $1-e^{-t}$.

	On the intersection of the events \eqref{supp:eq:Rtilde-pen-linear} and \eqref{supp:eq:Rtilde-pen-quad}, which has probability at least $1-2e^{-t}$, the bound \eqref{supp:eq:Rtilde-pen-bound} follows from \eqref{supp:eq:RE-decomp} after absorbing constants.

	Moreover,
	\[
	\begin{aligned}
		\|\Pomega\bE_0-\bQ\|_{\F} 	&= 	\|\Pomega^{1/2}(\Pomega^{1/2}\bE_0-\Pomega^{-1/2}\bQ)\|_{\F} 	\\
		& \le \sqrt{\omega_{\max}} \cdot \|\Pomega^{1/2}\bE_0-\Pomega^{-1/2}\bQ\|_{\F} 	\\
		& = 	\sqrt{2\omega_{\max}\mathcal{R}_{0}}.
	\end{aligned}
	\]
	Therefore,
	\[
	\begin{aligned}
		C\tau\sqrt{K}\,\|\Pomega\bE_0-\bQ\|_{\F}\sqrt{t} &	\le C\tau\sqrt{2\omega_{\max}K\mathcal{R}_{0}t}\\
		&=C\tau\sqrt{2\omega_{\max}Kt} \cdot \sqrt{\mathcal{R}_{0}}\\
		& \le \frac{1}{2}\mathcal{R}_{0} + C\omega_{\max}\tau^2Kt.
	\end{aligned}
	\]
	Since $pq\ge 1$, the last term is absorbed by 	$C\omega_{\max}pq\tau^2\bigl(K^2+K^{3/2}\sqrt{t}+Kt\bigr)$ after enlarging $C$, and
	\begin{equation}\label{supp:eq:Rtilde-pen-bound-simple}
		\mathcal{R}_E 	\le 2\mathcal{R}_{0} + C\omega_{\max}pq\tau^2\bigl(K^2+K^{3/2}\sqrt{t}+Kt\bigr).
	\end{equation}

	In particular, since $K^{3/2}\sqrt{t}\le (K^2+Kt)/2$, we also have
	\begin{equation}\label{supp:eq:Rtilde-pen-bound-order}
		\mathcal{R}_E \le  2\mathcal{R}_{0}	+ C\omega_{\max}pq\tau^2K(K+t).
	\end{equation}

	Moreover, 	we have
	\[
	2\mathcal{R}_{0} = 	\|\Pomega^{1/2}\bE_0-\Pomega^{-1/2}\bQ\|_{\F}^2 \le 2\langle \bE_0,\Pomega\bE_0\rangle
	+2\|\Pomega^{-1/2}\bQ\|_{\F}^2.
	\]
	Substituting this into \eqref{supp:eq:Rtilde-pen-bound-order} yields
	\begin{equation}\label{supp:eq:Rtilde-pen-bound-explicit}
		\mathcal{R}_E \le 2\langle \bE_0,\Pomega\bE_0\rangle + 2\|\Pomega^{-1/2}\bQ\|_{\F}^2 + C\omega_{\max}pq\tau^2K(K+t).
	\end{equation}

	Under Assumption~1 and~3, Lemma~\ref{supp:lem:E0-bound} gives
	\[
	\langle \bE_0,\Pomega\bE_0\rangle 	\le	\omega_{\max}\|\bE_0\|_{\F}^2 	\le 2\omega_{\max}K^2B^2\delta^2,
	\]
	which implies
	\[
	\mathcal{R}_E \le 4\omega_{\max}K^2B^2\delta^2 +2\|\Pomega^{-1/2}\bQ\|_{\F}^2 + C\omega_{\max}pq\tau^2K(K+t).
	\]

	Moreover,  by \eqref{supp:eq:KKT}, we have
	\[
	\|\Pg\bQ\|_{\F} \le \lambda_S, \quad g\in\mathcal{I},
	\]
	and
	\[
	\|\Pg\bQ\|_{\F}\le b\lambda_S,\quad g\in\mathcal{I}^c, \quad b\in (0,1).
	\]
	Therefore,
	\[
	\begin{aligned}
		\|\Pomega^{-1/2}\bQ\|_{\F}^2
		&=
		\sum_{g\in\mathcal{G}}\omega_g^{-1}\|\Pg\bQ\|_{\F}^2 \\
		&\le
		\lambda_S^2\Biggl(
		\sum_{g\in\mathcal{I}}\omega_g^{-1}
		+b^2\sum_{g\in\mathcal{I}^c}\omega_g^{-1}
		\Biggr).
	\end{aligned}
	\]
	Substituting this gives
	\[
	\mathcal{R}_E \le 4\omega_{\max}K^2B^2\delta^2 	+2\lambda_S^2\Biggl(\sum_{g\in\mathcal{I}}\omega_g^{-1} +b^2\sum_{g\in\mathcal{I}^c}\omega_g^{-1} \Biggr) + C\omega_{\max}pq\tau^2K(K+t).
	\]
	This completes the proof.
\end{proof}

	\begin{lemma}\label{supp:lem:BL-bound}
		Let $(\hL,\hS)$ be any optimal solution of the penalized program~(7).
		Let $(\pL,\pS)$ be an optimal solution of the noiseless oracle problem (23).
		Define $\bE:=\hD-\pL-\pS$, denote $\omega_{\max}:=\max_{g\in\mathcal{G}}\omega_g$, and write
		\[
		\mathcal{B}_L := \|\pL\|_*+\frac{\lambda_S}{\lambda_L}\|\pS\|_{\blk,1} +\frac{1}{2\lambda_L}\langle \bE,\Pomega\bE\rangle.
		\]
		Then
		\[
		\|\hL\|_*\le \mathcal{B}_L,
		\qquad
		\|\hL\|_{\op}\le \mathcal{B}_L.
		\]
		If, in addition, Assumption~1 and~3 hold, and $\lambda_S/\lambda_L$ satisfies the upper endpoint of (25), 
        then
		\begin{equation}\label{supp:eq:oracle-op-envelope-interval-struct}
			\mathcal{B}_L \le  \|\bL_0\|_*  + 2a(1-\rho^\prime)\sqrt{s}\,B	+ \frac{\omega_{\max}}{\lambda_L}\Bigl(\|\bE_0\|_{\F}^2+\|\bE_1\|_{\F}^2\Bigr).
		\end{equation}
		Moreover, under the assumptions of the preceding display and Assumption~5, for every $t>0$, with probability at least $1-e^{-t}$,
		\begin{equation}\label{supp:eq:BL-bound}
			\mathcal{B}_L \le  \|\bL_0\|_*  + 2a(1-\rho^\prime)\sqrt{s}\,B	+ \frac{2\omega_{\max}}{\lambda_L}K^2B^2\delta^2 +\frac{C\omega_{\max}}{\lambda_L}pq\tau^2\bigl(K^2+K^{3/2}\sqrt t+Kt\bigr)
		\end{equation}
		for some absolute constant $C>0$.
		Consequently, if moreover $B=O(1)$, $\|\bL_0\|_*=O(K)$, $\delta=O(K^{-1/2})$, and
		\[
		\omega_{\max}\asymp K^{-1},\qquad \lambda_L\asymp K^{-1/2},\qquad \lambda_S\asymp K^{-3/2},
		\]
		with $r,q,p$ fixed, $s=O(K^2)$, and $K\tau^2=O(1)$, then
		\[
		\mathcal{B}_L =O_p(K).
		\]
	\end{lemma}

\begin{proof}[Proof of Lemma~\ref{supp:lem:BL-bound}]
	Define $\hR:=\hD-\hL-\hS$.
	By optimality of $(\hL,\hS)$ for (7) and feasibility of $(\pL,\pS)$,
	\[
	\frac{1}{2}\langle\hR ,\Pomega\hR\rangle +\lambda_L\|\hL\|_*+\lambda_S\|\hS\|_{\blk,1}	\le \frac{1}{2}\langle \bE,\Pomega\bE\rangle 	+\lambda_L\|\pL\|_*+\lambda_S\|\pS\|_{\blk,1}.
	\]
	Dropping the nonnegative terms on the left and dividing by $\lambda_L$ gives
	\[
	\|\hL\|_{\op} \le \|\hL\|_*  \le \|\pL\|_*+\frac{\lambda_S}{\lambda_L}\|\pS\|_{\blk,1} + \frac{1}{2\lambda_L}\langle \bE,\Pomega\bE\rangle:=\mathcal{B}_L.
	\]

	By oracle optimality at right-hand side $\bD_0$ and feasibility of $(\bL_0,\bS_0)$ for (23),
	\[
	\|\pL\|_*+\frac{\lambda_S}{\lambda_L}\|\pS\|_{\blk,1}
	\le
	\|\bL_0\|_*+\frac{\lambda_S}{\lambda_L}\|\bS_0\|_{\blk,1}.
	\]
	Note that
	\[
	\langle \bE,\Pomega\bE\rangle \le \|\Pomega\|_{\op}\|\bE\|_{\F}^2 = \omega_{\max}\|\bE\|_{\F}^2.
	\]
	Using $\bE=\bE_0+\bE_1$,  we have
	\[
	\begin{aligned}
		\mathcal{B}_L  &\le \|\bL_0\|_* + \frac{\lambda_S}{\lambda_L}\|\bS_0\|_{\blk,1}  + \frac{\omega_{\max}}{2\lambda_L}\|\bE\|_{\F}^2 \\
		&\le \|\bL_0\|_* + a\frac{1-\rho^\prime}{\sqrt{s}}\|\bS_0\|_{\blk,1}
		+ \frac{\omega_{\max}}{\lambda_L}\Bigl(\|\bE_0\|_{\F}^2+\|\bE_1\|_{\F}^2\Bigr) \\
		&\le  \|\bL_0\|_*  + 2a(1-\rho^\prime)\sqrt{s}\,B	+ \frac{\omega_{\max}}{\lambda_L}\Bigl(\|\bE_0\|_{\F}^2+\|\bE_1\|_{\F}^2\Bigr),
	\end{aligned}
	\]
	where the second line uses the upper endpoint of (25), the inequality $(a+b)^2\le 2a^2+2b^2$, and Lemma~\ref{supp:lem:S0-blk1}. This proves \eqref{supp:eq:oracle-op-envelope-interval-struct}.

	Note that under Assumption~1 and~3, Lemma~\ref{supp:lem:E0-bound} gives
	\[
	\|\bE_0\|_{\F}^2 \le 2K^2B^2\delta^2,
	\]
	Moreover, under Assumption~5, Lemma~\ref{supp:lem:E1-Fnorm} yields
	\[
	\|\bE_1\|_{\F}^2 \le	C\,pq\,\tau^2\bigl(K^2+K^{3/2}\sqrt t+Kt\bigr)
	\]
	with probability at least $1-e^{-t}$.
	Therefore \eqref{supp:eq:BL-bound} holds with probability at least $1-e^{-t}$.
	Under the above assumptions, the four terms on the RHS of ~\eqref{supp:eq:BL-bound} are $O(K)$, $O(K)$, $O(\sqrt{K})$, and $O(\sqrt{K})$, respectively, so
	\[
	\mathcal{B}_L=O_{p}(K).
	\]
\end{proof}

\begin{proposition}[Noisy recovery]\label{supp:prop:PV-noisy}
	Suppose the noisy decomposition $\hD=\bL_0+\bS_0+\bE$ with $\bE=\bE_0+\bE_1$ satisfies Assumption~1-4 of the main paper.
	Suppose in addition $\bE_1$ admits the blockwise representation $\bE_1^{(j,k)} = \bXi_j -\bXi_k$ for $(j,k)\in\mathcal{G}$ and satisfies Assumption~5.
	Let $(\hL,\hS)$ be any optimal solution of the penalized problem~(7), and let $\hPV$ be an orthogonal projector onto a dominant $r$-dimensional right singular subspace of $\hL$.
	Let $a,b\in(0,1)$ satisfy $\sqrt{\mu rs/G}<ab/(1+a)^2$ and
		\begin{equation}
			\frac{\lambda_S}{\lambda_L}\in\biggl(
			\frac{\sqrt{\mu r/G}}{b-(1+a)\sqrt{\mu rs/G}},
			\ \frac{a}{(1+a)\sqrt{s}}
			\biggr).
		\end{equation}
		Let $(\pL,\pS)$ be any optimal solution of the noiseless oracle problem~(23), and let $\bQ$ be the certificate guaranteed by Proposition~\ref{supp:prop:Q-existence} for $(\pL,\pS)$.
	Write the compact SVD $\pL=\pU\pSigma\pVT$ with $\sigma_r(\pL)>0$, and let $\PV:=\pPV$.
	Define
	\[
	\mathcal{B}_L := \|\pL\|_*+\frac{\lambda_S}{\lambda_L}\|\pS\|_{\blk,1} +\frac{1}{2\lambda_L}\langle \bE,\Pomega\bE\rangle.
	\]
	and
	\[
	\mathcal{R}_E := \frac{1}{2}\langle \bE,\Pomega\bE\rangle - \langle \bQ,\bE\rangle 	+ \frac{1}{2}\|\Pomega^{-1/2}\bQ\|_{\F}^2.
	\]
	Then
	\begin{equation}\label{supp:eq:th2-penalized-bound}
		\begin{aligned}
			\|\hPV-\PV\|_{\op}
			&\le \frac{2\mathcal{B}_L}{\sigma_r(\pL)^2}\cdot \frac{1}{1-\rho^\prime}\Biggl[
			\Bigl(\|\pPU\bE\pPVC\|_{\F}+\sqrt{\rho^\prime}\,\|\PI\bE\pPVC\|_{\F}\Bigr)\\
			&\quad + \Bigl(\sqrt{r}\,\frac{\lambda_L}{\omega_{\min}}+\sqrt{\rho^\prime}\sqrt{s}\,\frac{\lambda_S}{\omega_{\min}}\Bigr)
			+ \sqrt{\frac{\mu r}{G}}\,\frac{\mathcal{R}_E}{(1-b)\lambda_S}
			+ \Bigl(1+\sqrt{\rho^\prime}-\rho^\prime\Bigr)\frac{\mathcal{R}_E}{(1-a)\lambda_L}
			\Biggr],
		\end{aligned}
	\end{equation}
	where $\rho^\prime:=\|\PI\pPU\PI\|_{\op}$, $\Pomega:=\sum_{g\in\mathcal{G}}\omega_g\Pg$, $
	\omega_{\min}:=\min_{g\in\mathcal{G}}\omega_g>0$, and $
	\omega_{\max}:=\max_{g\in\mathcal{G}}\omega_g.$
\end{proposition}
\begin{proof}[Proof of Proposition~\ref{supp:prop:PV-noisy}]
	Let $\bDelta_L:=\hL-\pL$, $\bDelta_S:=\hS-\pS$, and $\bDelta:=\bDelta_L+\bDelta_S$.
	Let $\hR:=\hD-\hL-\hS$, so that $\bDelta=\bE-\hR$.
	Define the Gram matrices
	\[
	\widehat{\bMM}:=\frac{1}{G}\hL^{\T}\hL,\qquad \bMM':=\frac{1}{G}(\pL)^{\T}\pL.
	\]
	Let $\bV_\perp\in\mathbb{R}^{p\times (p-r)}$ be any orthonormal complement of $\pV$.
	Since $\rank(\pL)=r$, $\bMM'$ has exactly $r$ positive eigenvalues and $\lambda_r(\bMM')=\sigma_r(\pL)^2/G$.
	By definition, $\hPV$ is also the orthogonal projector onto a leading $r$-dimensional eigenspace of $\widehat{\bMM}$.
	By the Davis-Kahan $\sin\Theta$ theorem
	\begin{equation}\label{supp:eq:dk-sym}
		\|\hPV-\PV\|_{\op}
		\le \frac{2\|(\widehat{\bMM}-\bMM')\pPVC\|_{\op}}{\lambda_r(\bMM')}.
	\end{equation}
	Because $\bMM'\pPVC=\bzero$, we have
	\[
	\|(\widehat{\bMM}-\bMM')\pPVC\|_{\op}
	=\|\widehat{\bMM}\pPVC\|_{\op}
	\le \frac{1}{G}\|\hL\|_{\op}\,\|\hL\pPVC\|_{\F}.
	\]
	Substituting this into \eqref{supp:eq:dk-sym} yields
	\[
	\|\hPV-\PV\|_{\op}
	\le \frac{2\|\hL\|_{\op}}{\sigma_r(\pL)^2}\,\|\hL\pPVC\|_{\F}.
	\]
	By Lemma~\ref{supp:lem:BL-bound}, we know
	\[
	\|\hL\|_{\op} \le \|\hL\|_*\le \mathcal{B}_L.
	\]
	Therefore,
		\[
		\|\hPV-\PV\|_{\op}
		\le \frac{2\mathcal{B}_L}{\sigma_r(\pL)^2}\,\|\hL\pPVC\|_{\F}.
		\]

	Since $\pL\pPVC=\bzero$, we have $\hL\pPVC=(\hL-\pL)\pPVC=\bDelta_L\pPVC$.
	Thus it remains to control $\|\bDelta_L\pPVC\|_{\F}$.
	Since
	\[
	\|\bDelta_L\pPVC\|_{\F} \le \|\pPU\bDelta_L\pPVC\|_{\F}+\|\pPTC\bDelta_L\|_{\F},
	\]
	we bound each term on RHS respectively.

	First, note that
	\[
	\bDelta\pPVC = \pPU\bDelta_L\pPVC+\pPUC\bDelta_L\pPVC+\PI\bDelta_S\pPVC+\PIC\bDelta_S\pPVC.
	\]
	Applying $\pPU$ to this decomposition gives
	\[
	\pPU\bDelta\pPVC
	=\pPU\bDelta_L\pPVC+\pPU\PI\bDelta_S\pPVC+\pPU\PIC\bDelta_S\pPVC.
	\]
	Because $\pPU$ and $\PI$ are orthogonal projectors,
	\[
	\|\pPU\PI\|_{\op}^2=\|(\pPU\PI)^\T(\pPU\PI)\|_{\op}=\|\PI\pPU\PI\|_{\op}=\rho^\prime.
	\]
	By Lemma \ref{supp:lem:rho-prime} and the upper bound of $\lambda_S/\lambda_L$, $\rho^\prime\le a/(1+a)<1$.
	Therefore,
	\begin{equation}\label{supp:eq:pu-pre}
		\|\pPU\bDelta_L\pPVC\|_{\F} \le  \|\pPU\bDelta\pPVC\|_{\F}+\sqrt{\rho^\prime}\|\PI\bDelta_S\pPVC\|_{\F}+\|\pPU\PIC\bDelta_S\pPVC\|_{\F}.
	\end{equation}
	Similarly, applying $\PI$ yields
	\[
	\PI\bDelta\pPVC=\PI\pPU\bDelta_L\pPVC+\PI\pPUC\bDelta_L\pPVC+\PI\bDelta_S\pPVC,
	\]
	and
	\[
	\|\PI\pPU\|_{\op}^2=\|(\PI\pPU)^\T(\PI\pPU)\|_{\op}=\|\pPU\PI\pPU\|_{\op}=\rho^\prime,
	\]
	so that
	\begin{equation}\label{supp:eq:pi-pre}
		\|\PI\bDelta_S\pPVC\|_{\F}
		\le \|\PI\bDelta\pPVC\|_{\F}+\sqrt{\rho^\prime}\|\pPU\bDelta_L\pPVC\|_{\F}+\|\PI\pPUC\bDelta_L\pPVC\|_{\F}.
	\end{equation}
	Substituting \eqref{supp:eq:pi-pre} into \eqref{supp:eq:pu-pre}, rearranging, and using
	\[
	\|\PI\pPUC\bDelta_L\pPVC\|_{\F}\le \|\pPUC\bDelta_L\pPVC\|_{\F}=\|\pPTC\bDelta_L\|_{\F},
	\]
	we obtain
	\begin{equation}\label{supp:eq:pu-delta-pvc-bound-pen}
		(1-\rho^\prime)\|\pPU\bDelta_L\pPVC\|_{\F}
		\le \|\pPU\bDelta\pPVC\|_{\F} + \sqrt{\rho^\prime}\|\PI\bDelta\pPVC\|_{\F}
		+ \sqrt{\rho^\prime}\|\pPTC\bDelta_L\|_{\F}
		+ \|\pPU\PIC\bDelta_S\pPVC\|_{\F}.
	\end{equation}

	Now, recall that $\bDelta=\bE-\hR$.
	Since $\|\pPU\|_{\op}=\|\PI\|_{\op}=\|\pPVC\|_{\op}=1$,
	we have
	\[
	\|\pPU\bDelta\pPVC\|_{\F} \le \|\pPU\bE\pPVC\|_{\F}+\|\pPU\hR\pPVC\|_{\F},
	\]
	and
	\[
	\|\PI\bDelta\pPVC\|_{\F} \le \|\PI\bE\pPVC\|_{\F}+\|\PI\hR\pPVC\|_{\F}.
	\]

	By Lemma~\ref{supp:lem:oracle-signal-scale}, Assumption~4 implies Condition~(12).
	Hence Proposition~\ref{supp:prop:Q-existence} applies to $(\pL,\pS)$, and the certificate $\bQ$ in the statement satisfies~\eqref{supp:eq:KKT}.

	Recall that the smooth part of objective~(7) is
	\[
	\ell(\bL,\bS)=\frac{1}{2}\langle \hD-\bL-\bS,\Pomega(\hD-\bL-\bS)\rangle,
	\]
	whose gradient with respect to either $\bL$ or $\bS$ equals $-\Pomega\hR$.
	Therefore, first-order optimality of (7) gives
	\[
	\bzero\in -\Pomega\hR+\lambda_L\partial\|\hL\|_*,
	\qquad
	\bzero\in -\Pomega\hR+\lambda_S\partial\|\hS\|_{\blk,1}.
	\]
	Equivalently, there exist subgradients $\bG_L\in\partial\|\hL\|_*$ and $\bG_S\in\partial\|\hS\|_{\blk,1}$ such that
	\[
	\Pomega\hR=\lambda_L\bG_L=\lambda_S\bG_S.
	\]
	Since $\|\bG_L\|_{\op}\le 1$, we have $\|\Pomega\hR\|_{\op}\le \lambda_L$. Since $\sigma_{\min}(\Pomega)=\omega_{\min}$, it follows that
	\[
	\|\hR\|_{\op}\le \frac{\lambda_L}{\omega_{\min}}.
	\]
	Also, each block of $\bG_S$ has Frobenius norm at most $1$, so
	\[
	\|\Pg(\Pomega\hR)\|_{\F}\le \lambda_S,\qquad g\in\mathcal{G}.
	\]
	Hence
	\[
	\|\PI\hR\|_{\F}\le \sqrt{s}\,\frac{\lambda_S}{\omega_{\min}}.
	\]
	Using $\|\pPU\hR\pPVC\|_{\F}\le \sqrt{r}\,\|\hR\|_{\op}$ and $\|\PI\hR\pPVC\|_{\F}\le \|\PI\hR\|_{\F}$, we obtain
	\[
	\|\pPU\hR\pPVC\|_{\F}\le \sqrt{r}\,\frac{\lambda_L}{\omega_{\min}},
	\qquad
	\|\PI\hR\pPVC\|_{\F}\le \sqrt{s}\,\frac{\lambda_S}{\omega_{\min}}.
	\]
	Moreover,
	\[
	\|\pPTC\bDelta_L\|_{\F}\le \|\pPTC\bDelta_L\|_*,
	\]
	By Lemma~\ref{supp:lem:transfer-incoherence}, for every $g\in\mathcal{I}^c$,
	\[
	\|\Pg\pU\|_{\F}^2\le \frac{\mu r}{G}.
	\]
	Therefore,
	\[
	\begin{aligned}
		\|\pPU\PIC\bDelta_S\pPVC\|_{\F}
		&= \|(\pU)^\T\PIC\bDelta_S\pPVC\|_{\F} \\
		&= \Bigl\|\sum_{g\in\mathcal{I}^c}(\pU)^\T\Pg\bDelta_S\pPVC \Bigr\|_{\F} \\
		&\le \sum_{g\in\mathcal{I}^c}\|(\pU)^\T\Pg\bDelta_S\pPVC\|_{\F} \\
		&\le \sum_{g\in\mathcal{I}^c}\|\Pg\pU\|_{\F} \cdot \|\Pg\bDelta_S\pPVC\|_{\F} \\
		&\le \sqrt{\frac{\mu r}{G}}\sum_{g\in\mathcal{I}^c}\|\Pg\bDelta_S\pPVC\|_{\F} \\
		&= \sqrt{\frac{\mu r}{G}}\,\|\PIC\bDelta_S\pPVC\|_{\blk,1} \\
		&\le \sqrt{\frac{\mu r}{G}}\,\|\PIC\bDelta_S\|_{\blk,1}.
	\end{aligned}
	\]

	Lemma~\ref{supp:lem:leakage-ineq-pen} yields
		\[
		\|\pPTC\bDelta_L\|_*\le \frac{\mathcal{R}_E}{(1-a)\lambda_L},
		\qquad
		\|\PIC\bDelta_S\|_{\blk,1}\le \frac{\mathcal{R}_E}{(1-b)\lambda_S}.
		\]

	Substituting these bounds into \eqref{supp:eq:pu-delta-pvc-bound-pen} gives
	\[
	\begin{aligned}
		(1-\rho^\prime)\|\pPU\bDelta_L\pPVC\|_{\F}
		&\le \|\pPU\bE\pPVC\|_{\F}+\sqrt{\rho^\prime}\,\|\PI\bE\pPVC\|_{\F}
		+\sqrt{r}\,\frac{\lambda_L}{\omega_{\min}}
		+\sqrt{\rho^\prime}\sqrt{s}\,\frac{\lambda_S}{\omega_{\min}}\\
		&\quad + \frac{\sqrt{\rho^\prime}\,\mathcal{R}_E}{(1-a)\lambda_L}
		+\sqrt{\frac{\mu r}{G}}\,\frac{\mathcal{R}_E}{(1-b)\lambda_S}.
	\end{aligned}
	\]
	Therefore,
	\[
	\begin{aligned}
		\|\bDelta_L\pPVC\|_{\F}
		&\le \frac{\|\pPU\bE\pPVC\|_{\F}+\sqrt{\rho^\prime}\,\|\PI\bE\pPVC\|_{\F}}{1-\rho^\prime}
		+\frac{\sqrt{r}\,\lambda_L/\omega_{\min}+\sqrt{\rho^\prime}\sqrt{s}\,\lambda_S/\omega_{\min}}{1-\rho^\prime}\\
		&\quad + \sqrt{\frac{\mu r}{G}}\,\frac{\mathcal{R}_E}{(1-\rho^\prime)(1-b)\lambda_S}
		+\Bigl(1+\frac{\sqrt{\rho^\prime}}{1-\rho^\prime}\Bigr)\frac{\mathcal{R}_E}{(1-a)\lambda_L}.
	\end{aligned}
	\]
	Combining this with the Davis-Kahan step proves \eqref{supp:eq:th2-penalized-bound}.

	It remains to control the stochastic terms from $\bE_1$.

	Since the matrices $\{\bXi_k\}_{k=1}^K$ are mean-zero in Lemma~\ref{supp:lem:E1-identity}, the random matrix $\bE_1$ is mean-zero.
	Let
	\[
	\bX:=(\pU)^{\T}\bE_1\bV_\perp\in\mathbb{R}^{r\times (p-r)},
	\qquad
	x:=\mathrm{vec}(\bX)\in\mathbb{R}^{d},
	\qquad
	d:=r(p-r).
	\]
	For any $u\in\mathbb{S}^{d-1}$, let $\mathbf{M}(u)\in\mathbb{R}^{r\times (p-r)}$ satisfy $\mathrm{vec}(\mathbf{M}(u))=u$, and define
	\[
	\bA(u):=\pU\mathbf{M}(u)\bV_\perp^{\T}\in\mathbb{R}^{Gq\times p}.
	\]
	Because $\pU$ and $\bV_\perp$ have orthonormal columns,
	\[
	\|\bA(u)\|_{\F}=\|\mathbf{M}(u)\|_{\F}=\|u\|_2=1.
	\]
	Moreover,
	\[
	\langle u,x\rangle
	=\langle \mathbf{M}(u),\bX\rangle
	=\langle \bA(u),\bE_1\rangle.
	\]
	Therefore, Lemma~\ref{supp:lem:E1-linear} yields
	\[
	\sup_{u\in\mathbb{S}^{d-1}}\|\langle u,x\rangle\|_{\psi_2}
	\le C_0\tau\sqrt{K}
	\]
	for an absolute constant $C_0>0$.
	Since $x$ is mean-zero, Lemma~\ref{supp:lem:subgaussian-vector-norm} implies
	\[
	\pr\!\left(\|x\|_2>C_1\tau\sqrt{K}\bigl(\sqrt{d}+\sqrt{t}\bigr)\right)\le 2e^{-t}
	\]
	for some absolute constant $C_1>0$.
	Using $\|\pPU\bE_1\pPVC\|_{\F}=\|\bX\|_{\F}=\|x\|_2$, and absorbing constants, we obtain
	\begin{equation}\label{supp:eq:th2-pen-subg-PUE}
		\|\pPU\bE_1\pPVC\|_{\F}
		\le C\tau\sqrt{K}\Bigl(\sqrt{r(p-r)}+\sqrt{t}\Bigr),
	\end{equation}
	with probability at least $1-2e^{-t}$, for any $t\ge 0$ with some constant $C>0$.

	Also, Lemma~\ref{supp:lem:E1-proj-PI} yields
	\begin{equation}\label{supp:eq:th2-pen-subg-PIE}
		\|\PI\bE_1\pPVC\|_{\F}
		\le C\tau\sqrt{q(p-r)\,K(K+t)},
	\end{equation}
	with probability at least $1-2e^{-t}$, for any $t\ge 0$ with some constant $C>0$.

	Moreover, by Lemma~\ref{supp:lem:RE-bound}, there exists an absolute constant $C>0$ such that, for any $t\ge 0$, with probability at least $1-2e^{-t}$,
	\begin{equation}\label{supp:eq:th2-RE}
		\mathcal{R}_E \le 4\omega_{\max}K^2B^2\delta^2 	+2\lambda_S^2\Biggl(\sum_{g\in\mathcal{I}}\omega_g^{-1} +b^2\sum_{g\in\mathcal{I}^c}\omega_g^{-1} \Biggr) + C\omega_{\max}pq\tau^2K(K+t).
	\end{equation}

	Let $\mathcal{F}_{t}$ be the event on which \eqref{supp:eq:th2-pen-subg-PUE}, \eqref{supp:eq:th2-pen-subg-PIE}, and \eqref{supp:eq:th2-RE} all hold. By a union bound,
	\[
	\pr(\mathcal{F}_{t})\ge 1-6e^{-t}.
	\]
	On $\mathcal{F}_{t}$, using $\bE=\bE_0+\bE_1$ and the triangle inequality,
	\[
	\|\pPU\bE\pPVC\|_{\F}
	\le \|\pPU\bE_0\pPVC\|_{\F}+C\tau\sqrt{K}\Bigl(\sqrt{r(p-r)}+\sqrt{t}\Bigr),
	\]
	and
	\[
	\|\PI\bE\pPVC\|_{\F}
	\le \|\PI\bE_0\pPVC\|_{\F}+C\tau\sqrt{q(p-r)\,K(K+t)}.
	\]

	Under  Assumption~1 and~3, then Lemma~\ref{supp:lem:E0-bound} implies
	\[
	\|\pPU\bE_0\pPVC\|_{\F}\le \sqrt{2}\,KB\delta, \qquad \|\PI\bE_0\pPVC\|_{\F}=0,
	\]
	which sharpens these bounds to
	\begin{equation}\label{supp:eq:th2-pen-subg-PUE-total}
		\|\pPU\bE\pPVC\|_{\F}
		\le \sqrt{2}\,KB\delta+C\tau\sqrt{K}\Bigl(\sqrt{r(p-r)}+\sqrt{t}\Bigr),
	\end{equation}
	and
	\begin{equation}\label{supp:eq:th2-pen-subg-PIE-total}
		\|\PI\bE\pPVC\|_{\F}
		\le C\tau\sqrt{q(p-r)\,K(K+t)}.
	\end{equation}
\end{proof}

\subsection{Proof of Theorem 2}
\begin{proof}[Proof of Theorem~2]
	Fix $t\ge 0$, and let $\mathcal{F}_{t}^\prime$ be the event on which \eqref{supp:eq:BL-bound}, \eqref{supp:eq:th2-RE}, \eqref{supp:eq:th2-pen-subg-PUE-total}, and~\eqref{supp:eq:th2-pen-subg-PIE-total} all hold.
	By the proof of Proposition~\ref{supp:prop:PV-noisy}, Lemma~\ref{supp:lem:BL-bound}, and a union bound,
	\[
	\pr(\mathcal{F}_{t}^\prime)\ge 1-7e^{-t}.
	\]
	By Lemma~\ref{supp:lem:rho-prime} and the upper endpoint of (25),
	\[
	\rho^\prime\le \frac{a}{1+a},
	\qquad
	\frac{1}{1-\rho^\prime}\le 1+a=O(1).
	\]
	By Lemma~\ref{supp:lem:E0-bound},
	\[
	\|\pPU\bE_0\pPVC\|_{\F}
	+\sqrt{\rho^\prime}\,\|\PI\bE_0\pPVC\|_{\F}
	\le \sqrt{2}\,KB\delta
	=
	O(\sqrt{K}).
	\]
	On $\mathcal{F}_{t}^\prime$, \eqref{supp:eq:th2-pen-subg-PUE-total} and \eqref{supp:eq:th2-pen-subg-PIE-total} give
	\[
	\begin{aligned}
		\|\pPU\bE\pPVC\|_{\F} +\sqrt{\rho^\prime}\,\|\PI\bE\pPVC\|_{\F}
		&\le \sqrt{2}\,KB\delta +C\tau\sqrt{K}\Bigl(\sqrt{r(p-r)}+\sqrt{t}\Bigr)
		+C\tau\sqrt{\rho^\prime\,q(p-r)\,K(K+t)} \\
		&= 	O(\sqrt{K}),
	\end{aligned}
	\]
	for fixed $t$, because $KB\delta=O(\sqrt{K})$, $\tau\sqrt{K}=O(1)$, and $\tau\sqrt{K(K+t)}=O(\sqrt{K})$.
	Moreover,
	\[
	\sqrt{r}\,\frac{\lambda_L}{\omega_{\min}}+\sqrt{\rho^\prime}\sqrt{s}\,\frac{\lambda_S}{\omega_{\min}} =O(\sqrt{K}),
	\]
	since $s\le G=\binom{K}{2}=O(K^2)$, while \eqref{supp:eq:th2-RE} and
	\[
	\sum_{g\in\mathcal{I}}\omega_g^{-1}
	+b^2\sum_{g\in\mathcal{I}^c}\omega_g^{-1}
	\le \sum_{g\in\mathcal{G}}\omega_g^{-1}
	\]
	yield
	\[
	\mathcal{R}_E=O(1)
	\]
	for each fixed $t$ on $\mathcal{F}_{t}^\prime$, because $\omega_{\max}K^2B^2\delta^2=O(1)$, $\lambda_S^2\sum_{g\in\mathcal{G}}\omega_g^{-1}=O(1)$, and $\omega_{\max}pq\tau^2K(K+t)=O(1)$.
	Therefore,
	\[
	\sqrt{\frac{\mu r}{G}}\,\frac{\mathcal{R}_E}{\lambda_S}
	=
	O(\sqrt{K}),
	\qquad
	\frac{\mathcal{R}_E}{\lambda_L}
	=
	O(\sqrt{K}),
	\]
	so the entire bracket in \eqref{supp:eq:th2-penalized-bound} is $O(\sqrt{K})$ on $\mathcal{F}_{t}^\prime$.
	By Lemma~\ref{supp:lem:oracle-signal-scale}, we have $\|\bL_0\|_*=O(K)$ and $\sigma_r(\pL)=\Omega(K)$.
	On the same event, \eqref{supp:eq:BL-bound} gives $\mathcal{B}_L=O(K)$ under the displayed assumptions, because the four terms on the right-hand side of \eqref{supp:eq:BL-bound} are respectively $O(K)$, $O(K)$, $O(\sqrt{K})$, and $O(\sqrt{K})$.
	Combining these bounds with Lemma~\ref{supp:lem:oracle-signal-scale} in \eqref{supp:eq:th2-penalized-bound}, there exists a constant $C_t>0$ such that
	\[
	\|\hPV-\PV\|_{\op}\le C_tK^{-1/2}
	\]
	hold on $\mathcal{F}_{t}^\prime$ for all sufficiently large $K$. Therefore,
	\[
	\pr\Bigl(\|\hPV-\PV\|_{\op}\le C_tK^{-1/2}\Bigr)\ge 1-7e^{-t}
	\]
	for all sufficiently large $K$.

	A slightly sharper result is useful for the next corollary. Under the corollary assumptions, \eqref{supp:eq:BL-bound} yields
	\[
	\mathcal{B}_L\le C\Bigl(K+\sqrt{K}+\sqrt{t}+\frac{t}{\sqrt{K}}\Bigr)
	\]
	on $\mathcal{F}_{t}^\prime$ for some constant $C>0$ independent of $K$ and $t$, while \eqref{supp:eq:th2-pen-subg-PUE-total} and \eqref{supp:eq:th2-pen-subg-PIE-total} give
	\[
	\|\pPU\bE\pPVC\|_{\F}+\sqrt{\rho^\prime}\,\|\PI\bE\pPVC\|_{\F}
	\le C\sqrt{K}\Bigl(1+\sqrt{\frac{t}{K}}\Bigr).
	\]
	Moreover, \eqref{supp:eq:th2-RE} implies $\mathcal{R}_E\le C(1+t/K)$ on $\mathcal{F}_{t}^\prime$, because $\omega_{\max}\asymp K^{-1}$, $\lambda_S\asymp K^{-3/2}$, $\sum_{g\in\mathcal G}\omega_g^{-1}=O(K^3)$, $\delta=O(K^{-1/2})$, and $K\tau^2=O(1)$. Since also $\sigma_r(\pL)^2\ge cK^2$ by Lemma~\ref{supp:lem:oracle-signal-scale}, substituting these bounds into \eqref{supp:eq:th2-penalized-bound} shows that, after enlarging $C$ if necessary,
	\[
	\|\hPV-\PV\|_{\op}\le CK^{-1/2}\Bigl(1+\sqrt{\frac{t}{K}}+\frac{t}{K}\Bigr)
	\]
	holds on $\mathcal{F}_{t}^\prime$ for all sufficiently large $K$.
\end{proof}

The following corollary provides the moment control needed for the uniform blockwise error analysis and collaborative-set detection.
\begin{corollary}\label{supp:cor:PV-noisy-L4}
	Under the assumptions of Theorem~2, there exists a constant $C_V>0$ such that, for all sufficiently large $K$,
	\[
	\Bigl(\E\|\hPV-\PV\|_{\op}^4\Bigr)^{1/4}\le C_VK^{-1/2}.
	\]
\end{corollary}
\begin{proof}[Proof of Corollary~\ref{supp:cor:PV-noisy-L4}]
	By the explicit bound proved at the end of Theorem~2 of the main paper, there exists a constant $C>0$ such that for every $t\ge 0$,
	\[
	\pr \Bigl(\|\hPV-\PV\|_{\op}\le CK^{-1/2}\Bigl(1+\sqrt{\frac{t}{K}}+\frac{t}{K}\Bigr)\Bigr)\ge 1-7e^{-t}
	\]
	for all sufficiently large $K$.
	Set $t_K:=4\log K$. Since $t_K/K\to 0$, there exists $C'>0$ such that
	\[
	\pr \Bigl(\|\hPV-\PV\|_{\op}\le C'K^{-1/2}\Bigr)\ge 1-7K^{-4}
	\]
	for all sufficiently large $K$.
	Now let $X:=\|\hPV-\PV\|_{\op}$. Since $\hPV$ and $\PV$ are orthogonal projectors, we have $0\le X\le 1$.
	\[
	\mathbb E X^4 \le (C'K^{-1/2})^4 + \Pr(X>C'K^{-1/2}) \le(C'^4+7)K^{-2},
	\]
	hence
	\[
	(\mathbb E\|\hPV-\PV\|_{\op}^4)^{1/4}\le C_V K^{-1/2}.
	\]
\end{proof}

\section{Proofs for collaborative-set recovery (Section 5.4)}
\begin{proof}[Proof of Lemma~3 of the main paper]
	For every $g\in\mathcal{G}$,
	\[
	\Pg(\hS_{\bA^\perp}-\bS_{\bA^\perp}) =\Pg(\bD_0(\hPVC-\PVC))+\Pg(\bE\hPVC).
	\]
	Hence
	\[
	\|\Pg(\hS_{\bA^\perp}-\bS_{\bA^\perp})\|_{\F} \le \|\Pg(\bD_0)\|_{\F}\,\|\hPV-\PV\|_{\op}+\|\Pg(\bE)\|_{\F}.
	\]
	Under Assumption~3,
	\[
	\max_{g\in\mathcal{G}}\|\Pg(\bD_0)\|_{\F} \le 2B.
	\]
	Also, by Lemma~\ref{supp:lem:E0-bound},
	\[
	\max_{g\in\mathcal{G}}\|\Pg(\bE_0)\|_{\F}\le 2B\delta.
	\]
	Therefore
	\begin{equation}\label{supp:eq:support-eps-split}
		\varepsilon_{\sup} \le 2B \|\hPV-\PV\|_{\op} + 2B\delta +\max_{g\in\mathcal{G}}\|\Pg(\bE_1)\|_{\F}.
	\end{equation}

	To bound the last term, fix $g=(j,k)\in\mathcal{G}$ and let $\bE_1^{(g)}:=\bXi_j-\bXi_k$. Since $\Pg(\bE_1)$ has exactly one nonzero block equal to $\bE_1^{(g)}$,
	\[
	\|\Pg(\bE_1)\|_{\F}=\|\bE_1^{(g)}\|_{\F}.
	\]
	Apply Lemma~\ref{supp:lem:E1-Fnorm} with $K=2$ to the two-client difference $\bE_1^{(g)}=\bXi_j-\bXi_k$. Since $\tau_j,\tau_k\le \tau$, there exists an absolute constant $C_0>0$ such that, for every $u\ge 0$,
	\[
	\pr\bigl(\|\bE_1^{(g)}\|_{\F}^2> C_0pq\tau^2\bigl(4+2^{3/2}\sqrt{u}+2u\bigr)\bigr)\le e^{-u}.
	\]
	Since $\sqrt{u}\le 1+u$ for all $u\ge 0$, after enlarging the constant we obtain
	\[
	\pr\bigl(\|\Pg(\bE_1)\|_{\F}^2> Cpq\tau^2(1+u)\bigr)\le e^{-u}.
	\]
	Choosing $u=t+2\log K$ and using $G\le K^2$, a union bound yields
	\[
	\pr\Bigl(\max_{g\in\mathcal{G}}\|\Pg(\bE_1)\|_{\F}>C\tau\sqrt{pq\bigl(1+t+2\log K\bigr)}\Bigr)\le e^{-t}.
	\]
	Since $p,q$ are fixed, this implies
	\[
	\max_{g\in\mathcal{G}}\|\Pg(\bE_1)\|_{\F}=O_{\pr}(\tau\sqrt{\log K}).
	\]
	Combining this with \eqref{supp:eq:support-eps-split} and Theorem~2, we obtain
	\[
	\varepsilon_{\sup} = O_{\pr}\Bigl(\tau\sqrt{\log K}+\delta+\|\hPV-\PV\|_{\op}\Bigr) =O_{\pr}\Bigl(\sqrt{\frac{\log K}{K}}\Bigr),
	\]
	where the last step uses $\delta=O(K^{-1/2})$, $K\tau^2=O(1)$, and
	$\|\hPV-\PV\|_{\op}=O_{\pr}(K^{-1/2})$.
\end{proof}

\begin{proof}[Proof of Theorem~3 of the main paper]
	We first prove that, if
	\begin{equation}\label{supp:eq:tau-interval}
		\varepsilon_{\sup}\le \tau_n<\beta_{\min}-\varepsilon_{\sup}
	\end{equation}
	and
	\[
		\frac{|\mathcal{C}^c|-1}{K-1}<\alpha\le \frac{|\mathcal{C}|-1}{K-1},
	\]
	then $\widehat{\mathcal{C}}_\alpha=\mathcal{C}.$

	Fix $k\in\mathcal{C}$. For every $j\in\mathcal{C}\setminus\{k\}$,
	\[
	\|\hS_{\bA^\perp}^{(j,k)}\|_{\F} = \|\hS_{\bA^\perp}^{(j,k)}-\bS_{\bA^\perp}^{(j,k)}\|_{\F} \le \varepsilon_{\sup} \le\tau_n,
	\]
	where the equality uses $\bS_{\bA^\perp}^{(j,k)}=\bzero$ for benign pairs, and the last inequality uses \eqref{supp:eq:tau-interval}.
	Hence at least $|\mathcal{C}|-1$ terms in the empirical average defining $\widehat{\mathcal{C}}_\alpha$ are equal to one, so
	\[
	\frac{1}{K-1}\sum_{j\neq k}\idop\{\|\hS_{\bA^\perp}^{(j,k)}\|_{\F}\le \tau_n\} \ge \frac{|\mathcal{C}|-1}{K-1} \ge \alpha.
	\]
	Therefore $k\in \widehat{\mathcal{C}}_\alpha$.

	Now fix $k\in\mathcal{C}^c$. For every $j\in\mathcal{C}$,
	\[
	\|\hS_{\bA^\perp}^{(j,k)}\|_{\F} \ge \|\bS_{\bA^\perp}^{(j,k)}\|_{\F}-\varepsilon_{\sup} \ge \beta_{\min}-\varepsilon_{\sup} > \tau_n.
	\]
	Thus none of the $|\mathcal{C}|$ mixed pairs incident to $k$ is counted as small. At most the $|\mathcal{C}^c|-1$ contaminated pairs can contribute ones within summation, so
	\[
	\frac{1}{K-1}\sum_{j\neq k}\idop\{\|\hS_{\bA^\perp}^{(j,k)}\|_{\F}\le \tau_n\} \le \frac{|\mathcal{C}^c|-1}{K-1} <\alpha.
	\]
	Therefore $k\notin \widehat{\mathcal{C}}_\alpha$.
	This proves $\widehat{\mathcal{C}}_\alpha=\mathcal{C}$.
	Then, it suffices to show that \eqref{supp:eq:tau-interval} holds, exactly or with high probability.

	For part (i), the oracle setting implies $\hD=\bD_0$, and Theorem~1 gives $\hPV=\PA$. Hence, for every pair $(j,k)$,
	\[
		\hS_{\bA^\perp}^{(j,k)} = (\bD_0(\bI-\PA))^{(j,k)} = \bS_{\bA^\perp}^{(j,k)}.
	\]
	Thus $\varepsilon_{\sup}=0$. If $0\le \tau_n<\beta_{\min}$, then \eqref{supp:eq:tau-interval} holds, thus $\widehat{\mathcal C}_{\alpha}=\mathcal C$.

	For part (ii), we prove that \eqref{supp:eq:tau-interval} holds with high probability.
	By Lemma~3, we know
	\[
		\varepsilon_{\sup} \le c\sqrt{\frac{\log K}{K}}
	\]
	for some constant $c>0$ with high probability.
	By the gap condition
	\[
		c\sqrt{\frac{\log K}{K}}<\tau_n<\beta_{\min}-c\sqrt{\frac{\log K}{K}},
	\]
	the interval
	\[
	\Bigl(c\sqrt{\tfrac{\log K}{K}},\ \beta_{\min} - c\sqrt{\tfrac{\log K}{K}}\Bigr)
	\]
	is nonempty for sufficiently large $K$, which yields
	\[
	\pr \bigl(\varepsilon_{\sup}\le\tau_n<\beta_{\min}-\varepsilon_{\sup}\bigr)\to 1.
	\]
	By the first part of the proof, we then have
	\[
	\pr(\widehat{\mathcal{C}}_\alpha=\mathcal{C})\to 1.
	\]
\end{proof}

\section{Proofs for error analysis (Section 6)}
\subsection{Supporting results}

For $k\in\mathcal C$,  define
\[
    \widebar{\bW}:=\frac{1}{|\mathcal{C}|}\sum_{j\in\mathcal{C}}\hW^{(j)},
    \qquad \widebar{\bE}_{0}:=\frac{1}{|\mathcal{C}|}\sum_{j\in\mathcal{C}}\bE_{0}^{(k,j)},
\]
where $\bE_{0}^{(k,j)} =-\bE_0^{(j,k)}$ whenever $k>j$, and set
\[
    \widebar{\bXi}:=\frac{1}{|\mathcal{C}|}\sum_{j\in\mathcal{C}}\bXi_j.
\]

The following proposition gives an exact decomposition of the refinement error.
\begin{proposition}\label{supp:prop:W-tilde-decomp}
    Assume $|\mathcal{C}|\ge 2$.
    On the event $\{\widehat{\mathcal{C}}_\alpha=\mathcal{C}\}$, for every $k\in\mathcal{C}$, we have the identity
    \begin{equation}\label{supp:eq:refined-W-gain-decomp}
        \tW^{(k)} = \bW^{(k)} + \bXi_k\PV +\widebar{\bXi}\PVC - \widebar{\bE}_{0}\PVC +\bigl(\hW^{(k)}-\widebar{\bW}\bigr)(\hPV-\PV).
    \end{equation}
    Consequently,
    \begin{equation}\label{supp:eq:W-tilde-bound}
        \|\tW^{(k)}-\bW^{(k)}\|_{\F} \le \|\bXi_k\PV\|_{\F} + \|\widebar{\bXi}\PVC\|_{\F} +\|\bR_k\|_{\F},
    \end{equation}
    where $\bR_k:= - \widebar{\bE}_{0}\PVC +\bigl(\hW^{(k)}-\widebar{\bW}\bigr)(\hPV-\PV)$.
\end{proposition}

\begin{proof}[Proof of Proposition~\ref{supp:prop:W-tilde-decomp}]
    In what follows, we work on the event $\{\widehat{\mathcal{C}}_\alpha=\mathcal{C}\}$.
    Fix $k, j\in\mathcal{C}$.
    By  definition of $\hL_{\bA}^{(k,j)}$ and the construction of $\hD$,
    \[
        \hL_{\bA}^{(k,j)}=(\hW^{(k)}-\hW^{(j)})\hPV.
    \]
    Hence
    \[
        \hW^{(j)} - \hL_{\bA}^{(j,k)} = \hW^{(j)}+\bigl(\hW^{(k)}-\hW^{(j)}\bigr)\hPV = 		\hW^{(k)}\hPV+\hW^{(j)}\hPVC.
    \]
    Averaging over $j\in\mathcal{C}$ yields
    \begin{equation}\label{supp:eq:refined-W-proj-decomp}
        \tW^{(k)}=\hW^{(k)}\hPV+\widebar{\bW}\hPVC.
    \end{equation}

    Note that
    \[
    \begin{aligned}
        \tW^{(k)}-\bW^{(k)} &=\hW^{(k)}\hPV+\widebar{\bW}\hPVC - \bW^{(k)} \PV - \bW^{(k)} \PVC \\
        &= (\hW^{(k)} -\bW^{(k)})\PV + \hW^{(k)}(\hPV - \PV) + (\widebar{\bW} - \bW^{(k)})\PVC +\widebar{\bW}(\hPVC-\PVC) \\
        &=(\hW^{(k)}-\bW^{(k)})\PV +\bigl(\widebar{\bW}-\bW^{(k)}\bigr)\PVC +\bigl(\hW^{(k)}-\widebar{\bW}\bigr)(\hPV-\PV).
    \end{aligned}
    \]

    Since $\hW^{(k)}-\bW^{(k)}=\bXi_k$, the first term of the last line is $\bXi_k\PV$. For the second term,
    \[
    \widebar{\bW}-\bW^{(k)}
    =
    \widebar{\bXi}
    +\frac{1}{|\mathcal{C}|}\sum_{j\in\mathcal{C}}\bigl(\bW^{(j)}-\bW^{(k)}\bigr).
    \]
    Now, by the construction of $\bD$ and $\bE_0$,
    \[
    \bW^{(k)}-\bW^{(j)}=\bL_{\bA}^{(k,j)}+\bE_{0}^{(k,j)},
    \]
    where $\bL_{\bA}^{(k,j)}\PVC=\bzero$. Hence
    \[
    (\bW^{(j)}-\bW^{(k)})\PVC=-\bE_{0}^{(k,j)}\PVC.
    \]
    Therefore,
    \[
    \bigl(\widebar{\bW}-\bW^{(k)}\bigr)\PVC = 	\widebar{\bXi}\PVC-\widebar{\bE}_{0}\PVC.
    \]
    Substituting this proves \eqref{supp:eq:refined-W-gain-decomp}.
\end{proof}

\subsection{Proof of Theorem 4 of the main paper}
\begin{proof}[Proof of Theorem~4]
    For the local estimator,
    \[
    \hW^{(k)}-\bW^{(k)}=\bXi_k=\bXi_k\PV+\bXi_k\PVC,
    \]
    so
    \[
    \|\hW^{(k)}-\bW^{(k)}\|_{\F}^2 	= \|\bXi_k\PV\|_{\F}^2+\|\bXi_k\PVC\|_{\F}^2.
    \]
    Taking expectations gives
    \begin{equation}\label{supp:eq:local-W-mse}
        \E\|\hW^{(k)}-\bW^{(k)}\|_{\F}^2 =\EV^{(k)}+\EVC^{(k)}.
    \end{equation}

    Under $\delta=0$ and $\hPV=\PV$, Proposition~\ref{supp:prop:W-tilde-decomp} gives
    \[
    \tW^{(k)}-\bW^{(k)}=\bXi_k\PV+\widebar{\bXi}\PVC.
    \]
    Since $\PV\PVC=\bzero$, the two summands have orthogonal right singular subspaces, so
    \[
    \|\tW^{(k)}-\bW^{(k)}\|_{\F}^2 	= \|\bXi_k\PV\|_{\F}^2+\|\widebar{\bXi}\PVC\|_{\F}^2.
    \]
    Next, since
    \[
        \widebar{\bXi}\PVC=\frac{1}{|\mathcal{C}|}\sum_{j\in\mathcal{C}}\bXi_j\PVC,
    \]
    by independence and $\E(\bXi_j\PVC)=\bzero$, we have
    \[
        \E\|\widebar{\bXi}\PVC\|_{\F}^2 = \frac{1}{|\mathcal{C}|^2}\sum_{j\in\mathcal{C}}\EVC^{(j)},
    \]
    which leads to
    \begin{equation}\label{supp:eq:oracle-refined-W-mse-general}
        \E\|\tW^{(k)}-\bW^{(k)}\|_{\F}^2= 	\EV^{(k)} 	+ 	\frac{1}{|\mathcal{C}|^2}\sum_{j\in\mathcal{C}}\EVC^{(j)}
    \end{equation}
    Recall that
    \begin{equation*}
        \mathcal{H}^{(k)}:=\EVC^{(k)}-\frac{1}{|\mathcal{C}|^2}\sum_{j\in\mathcal{C}}\EVC^{(j)},\quad \phi_k:=\frac{ \mathcal{H}^{(k)}}{ \EVC^{(k)}},\quad 	\theta_k:=	\frac{\EVC^{(k)}}{\mathcal{E}^{(k)}}.
    \end{equation*}
    Thus,
    \begin{equation*}
        \begin{aligned}
            \E\|\tW^{(k)}-\bW^{(k)}\|_{\F}^2 &= \EV^{(k)}+ (-\mathcal{H}^{(k)} + \EVC^{(k)})\\
            & =\mathcal{E}^{(k)}-\mathcal{H}^{(k)}\\
            & = (1 - \phi_k\theta_k)\mathcal{E}^{(k)},
        \end{aligned}
    \end{equation*}
    which proves
    \[
    \E\|\tW^{(k)}-\bW^{(k)}\|_{\F}^2 =(1-\theta_k\phi_k) \	\E\|\hW^{(k)}-\bW^{(k)}\|_{\F}^2.
    \]

    When $\{\bXi_j\}_{j\in\mathcal{C}}$ are identically distributed,  the second term in \eqref{supp:eq:oracle-refined-W-mse-general} satisfies
    \[
        \frac{1}{|\mathcal{C}|^2}\sum_{j\in\mathcal{C}}\EVC^{(j)} = \frac{1}{|\mathcal{C}|}\EVC^{(k)},
    \]
    which implies
    \[
    \mathcal{H}^{(k)} = \frac{|\mathcal{C}|-1}{|\mathcal{C}|}\,\E\|\bXi_k\PVC\|_{\F}^2.
    \]
    Thus, $\phi_k=\frac{|\mathcal{C}|-1}{|\mathcal{C}|}$.
    Since $\{\bXi_j\}_{j\in\mathcal{C}}$ are also evenly distributed over the space
    $\operatorname{col}(\PV)\oplus \operatorname{col}(\PVC)$, we have $\theta_k= \frac{p-r}{p}$.
    substituting $\phi_k$ and $\theta_k$ proves the result.
\end{proof}

\subsection{Proof of Theorem 5 and 6 of the main paper}
\begin{proof}[Proof of Theorem~5 of the main paper]
    Define
    \[
        \bT_k :=\bXi_k\PV+\widebar{\bXi}\PVC
    \]
    and
    \[
        \bR_k :=-\widebar{\bE}_{0}\PVC+\bigl(\hW^{(k)}-\widebar{\bW}\bigr)(\hPV-\PV).
    \]
    By Proposition~\ref{supp:prop:W-tilde-decomp}, on the event $\{\widehat{\mathcal{C}}_\alpha=\mathcal{C}\}$,
    \[
        \tW^{(k)}-\bW^{(k)}=\bT_k+\bR_k.
    \]

    Since $\PV\PVC=\bzero$,
    \[
    \|\bT_k\|_{\F}^2=\|\bXi_k\PV\|_{\F}^2+\|\widebar{\bXi}\PVC\|_{\F}^2.
    \]
    By Assumption~5, the matrices $\{\bXi_j\}_{j\in\mathcal{C}}$ are independent and mean-zero. Hence
    \[
        \E\|\bT_k\|_{\F}^2 = \EV^{(k)}+\frac{1}{|\mathcal{C}|^2}
    \sum_{j\in\mathcal{C}}\EVC^{(j)} =\mathcal{E}^{(k)}-\mathcal{H}^{(k)}.
    \]

    Using the notations
    \[
    \theta_k=\frac{\EVC^{(k)}}{\mathcal{E}^{(k)}},\quad
    \phi_k=\frac{\mathcal{H}^{(k)}}{\EVC^{(k)}},\quad
    \varphi_k=\frac{\mathcal{E}_{\mathrm{cost}}^{(k)}}{\mathcal{E}^{(k)}},
    \]
    we have
    \[
    \mathcal{H}^{(k)}=\phi_k\theta_k\mathcal{E}^{(k)},\quad
    \E\|\bT_k\|_{\F}^2=(1-\phi_k\theta_k)\mathcal{E}^{(k)},\quad
    \E\|\bR_k\|_{\F}^2=\varphi_k\mathcal{E}^{(k)}.
    \]

    Expanding the square and applying Cauchy--Schwarz gives
    \[
    \begin{aligned}
        \E\|\tW^{(k)}-\bW^{(k)}\|_{\F}^2
        &=\E\|\bT_k+\bR_k\|_{\F}^2\\
        &\le
        \left\{\bigl(\E\|\bT_k\|_{\F}^2\bigr)^{1/2}
        +\bigl(\E\|\bR_k\|_{\F}^2\bigr)^{1/2}\right\}^2\\
        &=
        \left(\sqrt{1-\phi_k\theta_k}+\sqrt{\varphi_k}\right)^2
        \mathcal{E}^{(k)}.
    \end{aligned}
    \]

    Since $0\le \theta_k\le 1$ and $\zeta_k\in(0,\phi_k)$,
    we have $1-\zeta_k\theta_k>1-\phi_k\theta_k\ge 0$. Condition~(34) implies that
    \[
    \sqrt{\varphi_k} < \sqrt{1-\zeta_k\theta_k}-\sqrt{1-\phi_k\theta_k}.
    \]
    Substituting this bound into the previous display yields
    \[
    \E\|\tW^{(k)}-\bW^{(k)}\|_{\F}^2 < 		(1-\zeta_k\theta_k)\mathcal{E}^{(k)} = (1-\zeta_k\theta_k)\E\|\hW^{(k)}-\bW^{(k)}\|_{\F}^2.
    \]
    This completes the proof.
\end{proof}

\begin{lemma}[]\label{supp:lem:cost-bound}
    Assume $|\mathcal{C}|\ge 2$, Assumptions~3 and~5. Fix $k\in\mathcal{C}$. Suppose the assumptions of Corollary~\ref{supp:cor:PV-noisy-L4} hold with constant $C_V>0$.
    Then
    \[
        \mathcal{E}_{\mathrm{cost}}^{(k)} \le 8B^2\delta^2+2C_{\Delta}^2C_V^2(B+\tau\sqrt{pq})^2K^{-1}
    \]
    for some absolute constant $C_{\Delta}>0$ and all sufficiently large $K$.
\end{lemma}
\begin{proof}[Proof of Lemma~\ref{supp:lem:cost-bound}]
    Denote
    \[
        \bT_k :=\bXi_k\PV+\widebar{\bXi}\PVC,
    \]
    and
    \[
        \bR_k:=-\widebar{\bE}_{0}\PVC+\bigl(\hW^{(k)}-\widebar{\bW}\bigr)(\hPV-\PV).
    \]
    Define
    \[
        \Delta:=\Bigl(\E\|\hW^{(k)}-\widebar{\bW}\|_{\F}^4\Bigr)^{1/4}
    \]
    and
    \[
        \varepsilon_V:=\Bigl(\E\|\hPV-\PV\|_{\op}^4\Bigr)^{1/4}.
    \]

    Using $(a+b)^2\le 2a^2+2b^2$, together with $\|\widebar{\bE}_{0}\PVC\|_{\F}\le 2B\delta$, we obtain
    \[
        \|\bR_k\|_{\F}^2 \le 8B^2\delta^2 +2\|\hW^{(k)}-\widebar{\bW}\|_{\F}^2\,\|\hPV-\PV\|_{\op}^2.
    \]
    Taking expectations and applying Cauchy--Schwarz give
    \[
    \begin{aligned}
        \mathcal{E}_{\mathrm{cost}}^{(k)}	&\le 8B^2\delta^2 +2\Bigl(\E\|\hW^{(k)}-\widebar{\bW}\|_{\F}^4\Bigr)^{1/2}
        \Bigl(\E\|\hPV-\PV\|_{\op}^4\Bigr)^{1/2} \\
        &=8B^2\delta^2+2\Delta^2\varepsilon_V^2.
    \end{aligned}
    \]

    It remains to bound $\Delta_{-k}$ and $\varepsilon_V$.
    Write
    \[
        \hW^{(k)}-\widebar{\bW} =		\Bigl(\bW^{(k)}-\frac{1}{|\mathcal{C}|}\sum_{j\in\mathcal{C}}\bW^{(j)}\Bigr) 	+ \bigl(\bXi_k-\widebar{\bXi}\bigr).
    \]
    By Minkowski's inequality,
    \[
        \Delta \le
    \Bigl\|\bW^{(k)}-\frac{1}{|\mathcal{C}|}\sum_{j\in\mathcal{C}}\bW^{(j)}\Bigr\|_{\F}+\Bigl(\E\|\bXi_k-\widebar{\bXi}\|_{\F}^4\Bigr)^{1/4}.
    \]
    Since $\bW^{(j)}=\bW_0+\bB^{(j)}\bA^{(j)}$, $\|\bA^{(j)}\|_{\op}=1$, and $\|\bB^{(j)}\|_{\F}\le B$ for every benign client $j$,
    \[
        \Bigl\|\bW^{(k)}-\frac{1}{|\mathcal{C}|}\sum_{j\in\mathcal{C}}\bW^{(j)}\Bigr\|_{\F}
    \le
    \frac{1}{|\mathcal{C}|}\sum_{j\in\mathcal{C}}\|\bW^{(k)}-\bW^{(j)}\|_{\F}
    \le 2B.
    \]
    Moreover, by Assumption~5, the vector $\vec(\bXi_j)\in\mathbb{R}^{pq}$ is mean-zero and satisfies
    \[
    \sup_{u\in\mathbb{S}^{pq-1}}\|\langle u,\vec(\bXi_j)\rangle\|_{\psi_2}\le \tau,\qquad j\in[K].
    \]
    Applying Lemma~\ref{supp:lem:subgaussian-vector-norm} with $m=pq$ yields an absolute constant $C_0>0$ such that
    \[
        \pr \bigl(\|\bXi_j\|_{\F}>C_0\tau(\sqrt{pq}+\sqrt{t})\bigr)\le 2e^{-t}, \quad t\ge 0.
    \]
    By the tail-integral identity for nonnegative random variables, after enlarging $C_0$ if necessary, this implies
    \[
        \Bigl(\E\|\bXi_j\|_{\F}^4\Bigr)^{1/4}\le C_0\tau\sqrt{pq},\quad j\in[K].
    \]
    Using Minkowski again, together with
    \[
        \Bigl(\E\|\widebar{\bXi}\|_{\F}^4\Bigr)^{1/4} \le \frac{1}{|\mathcal{C}|}\sum_{j\in\mathcal{C}} \Bigl(\E\|\bXi_j\|_{\F}^4\Bigr)^{1/4} \le C_0\tau\sqrt{pq},
    \]
    we obtain
    \[
        \Bigl(\E\|\bXi_k-\widebar{\bXi}\|_{\F}^4\Bigr)^{1/4}\le 2C_0\tau\sqrt{pq}.
    \]
    Therefore,
    \begin{equation}\label{supp:eq:Delta-k-bound}
        \Delta\le 2B+2C_0\tau\sqrt{pq}\le C_{\Delta}(B+\tau\sqrt{pq})
    \end{equation}
    for $C_{\Delta}:=2\max\{1,C_0\}$.
    By Corollary~\ref{supp:cor:PV-noisy-L4},
    \[
        \varepsilon_V\le C_VK^{-1/2}
    \]
    for all sufficiently large $K$.
    Consequently,
    \begin{equation}\label{supp:eq:zk-bound}
        \mathcal{E}_{\mathrm{cost}}^{(k)}\le 8B^2\delta^2+2C_{\Delta}^2C_V^2(B+\tau\sqrt{pq})^2K^{-1}.
    \end{equation}
    This completes the proof.
\end{proof}

\begin{proof}[Proof of Theorem~6 of the main paper]
    By the proof of Theorem~5, we know
    \begin{equation}\label{supp:eq:gain-tmp}
        \begin{aligned}
            \E\|\tW^{(k)}-\bW^{(k)}\|_{\F}^2 &\le (\sqrt{1-\phi_k\theta_k}+\sqrt{\varphi_k})^2 \mathcal E^{(k)}\\
            & \le 2(1-\phi_k\theta_k + \varphi_k)\mathcal E^{(k)}
        \end{aligned}
    \end{equation}
    where the second inequality uses  $(a+b)^2\le 2a^2+2b^2$.
    So it suffices to bound $1-\phi_k\theta_k $ and $\varphi_k$ separately.

    Let $\bxi_j:=\vec(\bXi_j)$.
    Since $\PV$ and $\PVC$ are orthogonal projectors, vectorization gives
    \[
    \|\bXi_j\PV\|_{\F}^2 = \bxi_j^\top(\PV\otimes I_q)\bxi_j,
    \qquad
    \|\bXi_j\PVC\|_{\F}^2 = \bxi_j^\top(\PVC\otimes I_q)\bxi_j.
    \]
    By Assumption~5, $\E\bxi_j=\bzero$. Hence
    \[
    \EV^{(j)} = \tr\{(\PV\otimes I_q)\bSigma_j\},
    \qquad
    \EVC^{(j)} =\tr\{(\PVC\otimes I_q)\bSigma_j\}.
    \]
    The projectors $\PV\otimes I_q$ and $\PVC\otimes I_q$ have ranks $qr$ and $q(p-r)$, respectively. Assumption~6 of the main paper therefore implies, for every $j\in\mathcal C$,
    \[
    \kappa_{-}qr\tau^2 \le \EV^{(j)} \le \kappa_{+}qr\tau^2,
    \]
    and
    \[
    \kappa_{-}q(p-r)\tau^2 \le \EVC^{(j)} \le \kappa_{+}q(p-r)\tau^2.
    \]
    In particular, $\EVC^{(k)}>0$ and
    \[
    \kappa_- qp\tau^2
    \le
    \mathcal E^{(k)}
    =
    \EV^{(k)}+\EVC^{(k)}
    \le
    \kappa_+ qp\tau^2.
    \]

    Using the notation
    \[
    \theta_k=\frac{\EVC^{(k)}}{\mathcal E^{(k)}},
    \quad
    \phi_k=\frac{\mathcal H^{(k)}}{\EVC^{(k)}},
    \quad
    \varphi_k=\frac{\mathcal E_{\mathrm{cost}}^{(k)}}{\mathcal E^{(k)}},
    \]
    we have
    \[
    1-\theta_k
    =
    \frac{\EV^{(k)}}{\mathcal E^{(k)}}
    \le
    \frac{\EV^{(k)}}{\EVC^{(k)}}
    \le
    \frac{\kappa_+}{\kappa_-}\frac{r}{p-r},
    \]
    and
    \[
        1-\phi_k = \frac{1}{\EVC^{(k)}(\eta K)^2}\sum_{j\in\mathcal C}\EVC^{(j)} \le \frac{\kappa_+}{\kappa_-}\frac{1}{\eta K}.
    \]
    Thus, for all sufficiently large $K$,
    \begin{equation}\label{supp:eq:oracle-ratio-balanced}
        1-\phi_k\theta_k
        =
        (1-\theta_k)+\theta_k(1-\phi_k)
        \le
        C_1\left(\frac{r}{p-r}+\frac{1}{\eta K}\right)
    \end{equation}
    for a constant $C_1>0$ depending only on $\kappa_+/\kappa_-$.

    We next control the non-oracle cost $\varphi_k$. Lemma~\ref{supp:lem:cost-bound}, together with $\delta=O(K^{-1/2})$ and $B=O(\sqrt{qr})$, gives
    \[
    \mathcal E_{\mathrm{cost}}^{(k)}
    \le
    C_2\left\{\frac{qr}{K}
    +\frac{qr+\tau^2pq}{K}\right\}
    \]
    for a constant $C_2>0$.
    Dividing by the lower bound $\mathcal E^{(k)}\ge \kappa_-qp\tau^2$ yields
    \[
    \varphi_k
    \le
    C_3\left\{\frac{r}{pK\tau^2}+\frac{1}{K}\right\}.
    \]
    Using $K\tau^2\asymp 1$ and $\eta\le1$, this becomes
    \begin{equation}\label{supp:eq:cost-ratio-balanced}
        \varphi_k
        \le
        C_4\left(\frac{r}{p-r}+\frac{1}{\eta K}\right),
    \end{equation}
    after enlarging $C_4$.

    Combining \eqref{supp:eq:oracle-ratio-balanced} and \eqref{supp:eq:cost-ratio-balanced}, and \eqref{supp:eq:gain-tmp}, we have
    \[
    \E\|\tW^{(k)}-\bW^{(k)}\|_{\F}^2 \le \Bigl(\frac{c_0}{\eta K}+\frac{c_1r}{p-r}\Bigr)
    \E\|\hW^{(k)}-\bW^{(k)}\|_{\F}^2
    \]
    for constants $c_0,c_1>0$.

    It remains to justify that Condition~(34) in the main text, or equivalently,
    \begin{equation}\label{supp:eq:refine-condition3}
        \sqrt{1-\phi_k\theta_k} + \sqrt{\varphi_k} <\sqrt{1-\zeta_k\theta_k},
    \end{equation}
    is automatically satisfiable.
    By \eqref{supp:eq:oracle-ratio-balanced} and \eqref{supp:eq:cost-ratio-balanced}, we have $\sqrt{1-\phi_k\theta_k}+\sqrt{\varphi_k}=o(1)$ when $K\to\infty$ and $r/(p-r)\to0$.
    For all sufficiently large $K$,  set $\zeta_k:=(1-C_k)/\theta_k$ for any $(\sqrt{1-\phi_k\theta_k} + \sqrt{\varphi_k})^2 < C_k < 1$, then Condition~\eqref{supp:eq:refine-condition3} immediately holds.
    Since $C_k>1-\phi_k\theta_k$, we have $0<\zeta_k<\phi_k$.
    This completes the proof.
\end{proof}

\section{Additional details for the sequence-copying experiment}\label{supp:sec:supp-copying}
This section provides additional details on the sequence-copying experiment in Section~8 of the main paper.

\textbf{Copying tasks}. All sequences have length $T=64$ and are defined over a vocabulary of size $53$, consisting of $52$ letter tokens and one padding token. Given a sampled segment $u=(u_1,\ldots,u_L)$, the clean-copying task places two adjacent copies of $u$ in the sequence and pads the remaining positions. The fuzzy-copying task instead embeds the two copies in sampled background tokens; the three background lengths before, between, and after the copies are determined by sampling two distinct cut points among the available non-copy positions.

The model is trained autoregressively by next-token prediction from prefixes. Following \citep{song2025out}, evaluation uses masked next-token accuracy: positions in the first copy and the first three positions of the second copy are excluded, so the reported accuracy focuses on retrieval after a short burn-in period. In the homogeneous regime, evaluation is conducted on the global noisy-long distribution; in the heterogeneous regime, each benign client is evaluated on its own distribution. All results are averaged over $100$ Monte Carlo replicates.

\textbf{Backbone pretraining and local LoRA fine-tuning}. The common backbone is a Transformer with two self-attention blocks, hidden dimension $64$, one attention head per block, rotary positional encoding, residual connections, and layer normalization. It is pretrained for $5500$ Adam steps on a balanced mixture of fuzzy-copying tasks with $L\in\{5,\ldots,15\}$, power-law exponent $1.1$, batch size $64$, and learning rate $0.001$. The pretrained backbone is then frozen throughout the federated fine-tuning experiments.

In the local fine-tuning stage, each client inserts rank-$3$ LoRA adapters into the eight attention projections $\{\bW_q,\bW_k,\bW_v,\bW_o\}$ from the two Transformer blocks and into the final output layer. The LoRA $\bA$ factors are initialized from the same seed across clients, and only LoRA parameters are updated locally. Each client trains on $2000$ generated sequences for one epoch, using batch size $50$, learning rate $0.001$, LoRA scale $16$, and dropout $0.005$.

\textbf{Federated settings}. In the homogeneous regime, all benign clients share the same fuzzy-copying distribution, with power-law exponent $1.1$ and fixed segment length $L=16$. In the heterogeneous regime, each benign client's power-law exponent is redrawn from $[0.95,1.6]$, and its fixed segment length is redrawn from $\{10,\ldots,26\}$ in each replicate. In both regimes, the contaminated client contributes a model fine-tuned on a copying task mismatched to those of the benign clients.

In the refinement stage, aggregation is restricted to the eight attention matrices from the two Transformer blocks. The output layer is excluded from aggregation, although it is locally fine-tuned. CLAIR uses regularization parameters $(\lambda_L,\lambda_S)=(0.5,0.4)$ in the homogeneous regime and $(0.5,0.2)$ in the heterogeneous regime, with active-pair thresholds $0.5$ and $0.01$, respectively.

\clearpage
\bibliographystyle{plainnat}
\bibliography{paper}

@article{sheen2024implicit,
  title={Implicit regularization of gradient flow on one-layer softmax attention},
  author={Sheen, Heejune and Chen, Siyu and Wang, Tianhao and Zhou, Harrison H},
  journal={arXiv preprint arXiv:2403.08699},
  year={2024}
}

@article{klopp2019structured,
  title={Structured matrix estimation and completion},
  author={Klopp, Olga and Lu, Yu and Tsybakov, Alexandre B and Zhou, Harrison H},
  journal={Bernoulli},
  volume={25},
  number={4B},
  pages={3883--3911},
  year={2019},
  publisher={JSTOR}
}

@article{doss2023optimal,
  title={Optimal estimation of high-dimensional Gaussian location mixtures},
  author={Doss, Natalie and Wu, Yihong and Yang, Pengkun and Zhou, Harrison H},
  journal={The Annals of Statistics},
  volume={51},
  number={1},
  pages={62--95},
  year={2023},
  publisher={Institute of Mathematical Statistics}
}

@article{gu2025robust,
	title={Robust angle-based transfer learning in high dimensions},
	author={Gu, Tian and Han, Yi and Duan, Rui},
	journal={Journal of the Royal Statistical Society Series B: Statistical Methodology},
	volume={87},
	number={3},
	pages={723--745},
	year={2025},
	publisher={Oxford University Press UK}
}

@inproceedings{niu2024collaborative,
  title={Collaborative learning with shared linear representations: Statistical rates and optimal algorithms},
  author={Niu, Xiaochun and Su, Lili and Xu, Jiaming and Yang, Pengkun},
  booktitle={International Workshop on Federated Foundation Models in Conjunction with NeurIPS 2024},
  year={2024}
}

@inproceedings{zhang2024towards,
	title={Towards building the federatedgpt: Federated instruction tuning},
	author={Zhang, Jianyi and Vahidian, Saeed and Kuo, Martin and Li, Chunyuan and Zhang, Ruiyi and Yu, Tong and Wang, Guoyin and Chen, Yiran},
	booktitle={ICASSP 2024-2024 IEEE international conference on acoustics, speech and signal processing (ICASSP)},
	pages={6915--6919},
	year={2024},
	organization={IEEE}
}

@article{tian2024hydralora,
	title={Hydralora: An asymmetric lora architecture for efficient fine-tuning},
	author={Tian, Chunlin and Shi, Zhan and Guo, Zhijiang and Li, Li and Xu, Chengzhong},
	journal={Advances in Neural Information Processing Systems},
	volume={37},
	pages={9565--9584},
	year={2024}
}

@inproceedings{sun2024improving,
	title={Improving LoRA in Privacy-preserving Federated Learning},
	author={Sun, Youbang and Li, Zitao and Li, Yaliang and Ding, Bolin},
	booktitle={The Twelfth International Conference on Learning Representations}
}

@article{smith2017federated,
	title={Federated multi-task learning},
	author={Smith, Virginia and Chiang, Chao-Kai and Sanjabi, Maziar and Talwalkar, Ameet S},
	journal={Advances in neural information processing systems},
	volume={30},
	year={2017}
}

@article{fallah2020personalized,
	title={Personalized federated learning with theoretical guarantees: A model-agnostic meta-learning approach},
	author={Fallah, Alireza and Mokhtari, Aryan and Ozdaglar, Asuman},
	journal={Advances in neural information processing systems},
	volume={33},
	pages={3557--3568},
	year={2020}
}

@article{wang2024textit,
	title={Trans-LoRA: towards data-free Transferable Parameter Efficient Finetuning},
	author={Wang, Runqian and Ghosh, Soumya and Cox, David and Antognini, Diego and Oliva, Aude and Feris, Rogerio and Karlinsky, Leonid},
	journal={Advances in Neural Information Processing Systems},
	volume={37},
	pages={61217--61237},
	year={2024}
}

@inproceedings{tzeng2017adversarial,
	title={Adversarial discriminative domain adaptation},
	author={Tzeng, Eric and Hoffman, Judy and Saenko, Kate and Darrell, Trevor},
	booktitle={Proceedings of the IEEE conference on computer vision and pattern recognition},
	pages={7167--7176},
	year={2017}
}

@article{bousmalis2016domain,
	title={Domain separation networks},
	author={Bousmalis, Konstantinos and Trigeorgis, George and Silberman, Nathan and Krishnan, Dilip and Erhan, Dumitru},
	journal={Advances in neural information processing systems},
	volume={29},
	year={2016}
}

@article{tian2025learning,
	title={Learning from similar linear representations: Adaptivity, minimaxity, and robustness},
	author={Tian, Ye and Gu, Yuqi and Feng, Yang},
	journal={Journal of Machine Learning Research},
	volume={26},
	number={187},
	pages={1--125},
	year={2025}
}

@article{chua2021fine,
	title={How fine-tuning allows for effective meta-learning},
	author={Chua, Kurtland and Lei, Qi and Lee, Jason D},
	journal={Advances in Neural Information Processing Systems},
	volume={34},
	pages={8871--8884},
	year={2021}
}

@article{duan2023adaptive,
	title={Adaptive and robust multi-task learning},
	author={Duan, Yaqi and Wang, Kaizheng},
	journal={The Annals of Statistics},
	volume={51},
	number={5},
	pages={2015--2039},
	year={2023},
	publisher={Institute of Mathematical Statistics}
}

@inproceedings{liu2025lora,
	title={LoRA subtraction for drift-resistant space in exemplar-free continual learning},
	author={Liu, Xuan and Chang, Xiaobin},
	booktitle={Proceedings of the IEEE/CVF Conference on Computer Vision and Pattern Recognition},
	pages={15308--15318},
	year={2025}
}

@article{candes2012exact,
	title={Exact matrix completion via convex optimization},
	author={Candes, Emmanuel and Recht, Benjamin},
	journal={Communications of the ACM},
	volume={55},
	number={6},
	pages={111--119},
	year={2012},
	publisher={ACM New York, NY, USA}
}

@article{guo2025robust,
	title={Robust inference for federated meta-learning},
	author={Guo, Zijian and Li, Xiudi and Han, Larry and Cai, Tianxi},
	journal={Journal of the American Statistical Association},
	volume={120},
	number={551},
	pages={1695--1710},
	year={2025},
	publisher={Taylor \& Francis}
}

@article{kumar2023impact,
	title={The impact of adversarial attacks on federated learning: A survey},
	author={Kumar, Kummari Naveen and Mohan, Chalavadi Krishna and Cenkeramaddi, Linga Reddy},
	journal={IEEE Transactions on Pattern Analysis and Machine Intelligence},
	volume={46},
	number={5},
	pages={2672--2691},
	year={2023},
	publisher={IEEE}
}

@article{chen2023minimax,
	title={Minimax estimation for personalized federated learning: an alternative between FedAvg and local training?},
	author={Chen, Shuxiao and Zheng, Qinqing and Long, Qi and Su, Weijie J},
	journal={Journal of Machine Learning Research},
	volume={24},
	number={262},
	pages={1--59},
	year={2023}
}

@article{grattafiori2024llama,
	title={The llama 3 herd of models},
	author={Grattafiori, Aaron and Dubey, Abhimanyu and Jauhri, Abhinav and Pandey, Abhinav and Kadian, Abhishek and Al-Dahle, Ahmad and Letman, Aiesha and Mathur, Akhil and Schelten, Alan and Vaughan, Alex and others},
	journal={arXiv preprint arXiv:2407.21783},
	year={2024}
}

@article{comanici2025gemini,
	title={Gemini 2.5: Pushing the frontier with advanced reasoning, multimodality, long context, and next generation agentic capabilities},
	author={Comanici, Gheorghe and Bieber, Eric and Schaekermann, Mike and Pasupat, Ice and Sachdeva, Noveen and Dhillon, Inderjit and Blistein, Marcel and Ram, Ori and Zhang, Dan and Rosen, Evan and others},
	journal={arXiv preprint arXiv:2507.06261},
	year={2025}
}

@article{ji2026overview,
	title={An overview of large language models for statisticians},
	author={Ji, Wenlong and Yuan, Weizhe and Getzen, Emily and Cho, Kyunghyun and Jordan, Michael I and Mei, Song and Weston, Jason and Su, Weijie J and Xu, Jing and Zhang, Linjun},
	journal={The American Statistician},
	number={just-accepted},
	pages={1--106},
	year={2026},
	publisher={Taylor \& Francis}
}

@inproceedings{wang2025adaptive,
	title={Adaptive LoRA Experts Allocation and Selection for Federated Fine-Tuning},
	author={Wang, Lei and Bian, Jieming and Zhang, Letian and Xu, Jie},
	booktitle={The Thirty-ninth Annual Conference on Neural Information Processing Systems}
}

@inproceedings{houlsby2019parameter,
	title={Parameter-efficient transfer learning for NLP},
	author={Houlsby, Neil and Giurgiu, Andrei and Jastrzebski, Stanislaw and Morrone, Bruna and De Laroussilhe, Quentin and Gesmundo, Andrea and Attariyan, Mona and Gelly, Sylvain},
	booktitle={International conference on machine learning},
	pages={2790--2799},
	year={2019},
	organization={PMLR}
}

@article{xu2026parameter,
	title={Parameter-efficient fine-tuning methods for pretrained language models: A critical review and assessment},
	author={Xu, Lingling and Xie, Haoran and Qin, S Joe and Tao, Xiaohui and Wang, Fu Lee},
	journal={IEEE Transactions on Pattern Analysis and Machine Intelligence},
	year={2026},
	publisher={IEEE}
}

@article{vaswani2017attention,
	title={Attention is all you need},
	author={Vaswani, Ashish and Shazeer, Noam and Parmar, Niki and Uszkoreit, Jakob and Jones, Llion and Gomez, Aidan N and Kaiser, {\L}ukasz and Polosukhin, Illia},
	journal={Advances in neural information processing systems},
	volume={30},
	year={2017}
}

@inproceedings{guo2024selective,
	title={Selective aggregation for low-rank adaptation in federated learning},
	author={Guo, Pengxin and Zeng, Shuang and Wang, Yanran and Fan, Huijie and Wang, Feifei and Qu, Liangqiong},
	booktitle={13th International Conference on Learning Representations Iclr 2025},
	year={2025}
}

@article{wang2024flora,
	title={Flora: Federated fine-tuning large language models with heterogeneous low-rank adaptations},
	author={Wang, Ziyao and Shen, Zheyu and He, Yexiao and Sun, Guoheng and Wang, Hongyi and Lyu, Lingjuan and Li, Ang},
	journal={Advances in Neural Information Processing Systems},
	volume={37},
	pages={22513--22533},
	year={2024}
}

@inproceedings{bian2025lora,
	title={LoRA-FAIR: Federated LoRA fine-tuning with aggregation and initialization refinement},
	author={Bian, Jieming and Wang, Lei and Zhang, Letian and Xu, Jie},
	booktitle={Proceedings of the IEEE/CVF International Conference on Computer Vision},
	pages={3737--3746},
	year={2025}
}

@article{chandrasekaran2011rank,
	title={Rank-sparsity incoherence for matrix decomposition},
	author={Chandrasekaran, Venkat and Sanghavi, Sujay and Parrilo, Pablo A and Willsky, Alan S},
	journal={SIAM Journal on Optimization},
	volume={21},
	number={2},
	pages={572--596},
	year={2011},
	publisher={SIAM}
}

@article{hsu2011robust,
	title={Robust matrix decomposition with sparse corruptions},
	author={Hsu, Daniel and Kakade, Sham M and Zhang, Tong},
	journal={IEEE Transactions on Information Theory},
	volume={57},
	number={11},
	pages={7221--7234},
	year={2011},
	publisher={IEEE}
}

@article{xu2012robust,
	title={Robust PCA via Outlier Pursuit},
	author={Xu, Huan and CARAMANIS, Constantine and SANGHAVI, Sujay},
	journal={IEEE transactions on information theory},
	volume={58},
	number={5},
	pages={3047--3064},
	year={2012}
}

@book{beck2017first,
  title={First-order methods in optimization},
  author={Beck, Amir},
  year={2017},
  publisher={SIAM}
}

@article{song2025out,
  title={Out-of-distribution generalization via composition: a lens through induction heads in transformers},
  author={Song, Jiajun and Xu, Zhuoyan and Zhong, Yiqiao},
  journal={Proceedings of the National Academy of Sciences},
  volume={122},
  number={6},
  pages={e2417182122},
  year={2025},
  publisher={National Academy of Sciences}
}

@article{zhou2013tensor,
	title={Tensor regression with applications in neuroimaging data analysis},
	author={Zhou, Hua and Li, Lexin and Zhu, Hongtu},
	journal={Journal of the American Statistical Association},
	volume={108},
	number={502},
	pages={540--552},
	year={2013},
	publisher={Taylor \& Francis}
}

@article{candes2011robust,
	title={Robust principal component analysis?},
	author={Cand{\`e}s, Emmanuel J and Li, Xiaodong and Ma, Yi and Wright, John},
	journal={Journal of the ACM (JACM)},
	volume={58},
	number={3},
	pages={11},
	year={2011},
	publisher={ACM}
}

@article{wu2023brief,
  title={A brief overview of ChatGPT: The history, status quo and potential future development},
  author={Wu, Tianyu and He, Shizhu and Liu, Jingping and Sun, Siqi and Liu, Kang and Han, Qing-Long and Tang, Yang},
  journal={IEEE/CAA Journal of Automatica Sinica},
  volume={10},
  number={5},
  pages={1122--1136},
  year={2023},
  publisher={IEEE}
}

@inproceedings{mcmahan2017communication,
  title={Communication-efficient learning of deep networks from decentralized data},
  author={McMahan, Brendan and Moore, Eider and Ramage, Daniel and Hampson, Seth and y Arcas, Blaise Aguera},
  booktitle={Artificial intelligence and statistics},
  pages={1273--1282},
  year={2017},
  organization={PMLR}
}

@article{yuan2007dimension,
  title={Dimension reduction and coefficient estimation in multivariate linear regression},
  author={Yuan, Ming and Ekici, Ali and Lu, Zhaosong and Monteiro, Renato},
  journal={Journal of the Royal Statistical Society Series B: Statistical Methodology},
  volume={69},
  number={3},
  pages={329--346},
  year={2007},
  publisher={Oxford University Press}
}

@article{hu2022lora,
  title={Lora: Low-rank adaptation of large language models.},
  author={Hu, Edward J and Shen, Yelong and Wallis, Phillip and Allen-Zhu, Zeyuan and Li, Yuanzhi and Wang, Shean and Wang, Lu and Chen, Weizhu and others},
  journal={ICLR},
  volume={1},
  number={2},
  pages={3},
  year={2022}
}

@article{feng2022projected,
  title={Projected robust PCA with application to smooth image recovery},
  author={Feng, Long and Wang, Junhui},
  journal={Journal of Machine Learning Research},
  volume={23},
  number={249},
  pages={1--41},
  year={2022}
}

@book{vershynin2018high,
  title={High-dimensional probability: An introduction with applications in data science},
  author={Vershynin, Roman},
  volume={47},
  year={2018},
  publisher={Cambridge university press}
}

\end{document}